\documentclass[11pt]{article}    

\usepackage[margin=0.98in]{geometry} 

\usepackage[utf8]{inputenc} %
\usepackage{url}            %
\usepackage{kpfonts}
\usepackage{times,comment}
{\title{\LARGE Regret Minimization with Adaptive Opponents \\in Repeated Games  
}  
}

\author{
  Mingyang Liu\thanks{Alphabetical Order}$^1$~~~~~~~~~~Asuman Ozdaglar$^1$~~~~~~~~~~Tiancheng Yu$^2$~~~~~~~~~~Kaiqing Zhang$^3$ \\\\
$^1$ Massachusetts Institute of Technology\\ 
$^2$ OpenAI\\ 
$^3$ University of Maryland, College Park
\\
\texttt{\{liumy19,asuman\}@mit.edu\qquad yutc14@gmail.com \qquad kaiqing@umd.edu}
}

\usepackage{times}
\usepackage{breakcites}
\usepackage{inconsolata}
\usepackage{enumitem}
\usepackage{notation}
\usepackage{arxiv}
\usepackage{algorithm}
\usepackage{algpseudocode}

\crefname{theorem}{theorem}{theorems}
\crefname{lemma}{lemma}{lemmas}
\crefname{proposition}{proposition}{propositions}
\crefname{remark}{remark}{remarks}
\crefname{corollary}{corollary}{corollaries}
\crefname{condition}{condition}{conditions}
\crefname{equation}{}{}
\Crefname{equation}{}{}
\Crefname{condition}{Condition}{Conditions}
\crefname{enumi}{}{}
\Crefname{enumi}{}{}

\begin{document}

\maketitle

\begin{abstract}
In this paper, we study regret minimization in repeated games with \emph{adaptive} opponents who can respond based on histories of play. The standard metric of \emph{external regret} in online learning is known to fail to capture such adaptivity. To account for players' counterfactual reasoning, we introduce {\tt Repeated Policy Regret (RP-Regret)}, a game-theoretic metric that measures the difference between the \emph{realized} and the \emph{best-in-hindsight} accumulated utility when all players can \emph{respond} to the history of play. Compared to existing regret notions in this setting, ours is native to repeated game playing, enabling stronger comparators and opponents with fewer constraints, while maintaining the possibility of finding better equilibria when all players minimize it. We first identify necessary conditions for obtaining {\tt RP-Regret} sublinear in time, on the variation of the player's comparator strategies in the regret definition and on the memories of both the comparator and opponents' strategies. We then study additional conditions and provable algorithms to minimize {\tt RP-Regret}, which is by definition \emph{non-convex} in the strategy space. To address this challenge, we propose three algorithms: (i) one based on an optimization oracle, as assumed in some prior work in online non-convex learning; (ii) one that minimizes a convex and \emph{linearized} surrogate of {\tt RP-Regret} at each iteration; (iii) one that directly minimizes {\tt RP-Regret} when opponents change strategies slowly. Furthermore, when all players can run algorithms to minimize the {\tt RP-Regret} (or its linearized variant), certain subgame perfect equilibria of the repeated game can be learned. We also provide experiments showing that minimizing our regret notions can lead to more cooperative solutions with higher utility in games such as Stag-Hunt.\footnote{Accepted for presentation at the Conference on Learning Theory (COLT) 2026.}
\end{abstract}

\section{Introduction}
\label{sec:introduction}

{Solving for equilibria
has long been one of the core problems in algorithmic game theory. However, an equilibrium may not correspond to a solution that yields a high utility\footnote{For the classical examples of (Iterated) Prisoner's Dilemma and Stag-Hunt, we follow the convention of using the term \emph{utility}; later when defining regret notions, we will follow the convention from online learning and use the term \emph{loss} throughout.} for all the players in the game. For example,}
in the well-known game of Prisoner's Dilemma (PD),
the only Nash equilibrium (NE)
is \emph{defect-defect},
which yields a relatively low utility
for both players. {Interestingly, however,} when it comes to Iterated Prisoner's Dilemma (IPD) \citep{axelrod1981evolution-prisoner}, {a \emph{repeated} version of PD,} there exists an NE with a much higher utility
for both players ({cf.} Example \ref{example:IPD}).
Therefore, solving for equilibria in repeated games may have advantages over that in one-shot matrix games, in terms of possibly achieving higher utilities.

More importantly, repeated games can also be employed to compute equilibria for one-shot games. In particular, one of the most efficient and scalable approaches for equilibrium computation is through \emph{no-regret learning/regret-minimization},
which has been both theoretically \citep{cesa2006prediction,roughgarden2010algorithmic-game-theory} and empirically \citep{brown2019superhuman-pluribus,brown2018superhuman-libratus} supported. It is known that when all players in the game run a no-regret algorithm, their average
strategy over time will converge to certain equilibria determined by the regret they
minimize. At the heart of these approaches is a proper definition of \emph{regret}, which by default usually means \emph{external regret} \citep{hannan1957approximation,hart2000simple}, and measures the difference between the loss incurred by an \emph{online decision-maker} (or a \emph{player}
in repeated games) and certain comparator decisions that they would have made, knowing the time-varying environments in hindsight.

However, in this classical regret definition, the loss sequences are implicitly assumed to be
only functions of \emph{timestep}   \citep{de2003combine-flexible}. This is general enough and perfectly sensible in the \emph{online learning} setting, but not in a \emph{game-theoretic} one, where, by definition, the opponent players should be \emph{responsive}  to the player's behavior during learning \citep{schlag2012impossibility}. Hence, the classical regret cannot capture the \emph{adaptivity} of the opponents, who can make decisions based on the \emph{history} of the play, and the fact that a player's action may affect the opponents' decisions later. Such an inability may lead to sub-optimal equilibrium solutions (\emph{e.g.,} always ``defect'' for IPD, as in Example \ref{example:IPD}). In fact, it has been shown that when players are responsive, achieving no-(external-)regret is impossible  \citep{schlag2012impossibility}.

{To account for the responsiveness and adaptivity of the opponents/adversaries, several new notions of {regret} have been developed for both \emph{online learning} and \emph{repeated game} settings. However, they are either not directly applicable to the repeated-game setting, or computationally intractable. For example, in the online learning setting, \cite{merhav2002sequential-memory-regret-init,arora2012online} introduced the notion of \emph{Policy Regret} to model environments with memory, in which the loss at round $t$ may depend on a history of past actions rather than only the current action. Their guarantees typically rely on an $m$-bounded-memory assumption, meaning that the loss at time $t$ depends only on the most recent $m$ actions. This assumption can be too restrictive in repeated games: a deviation in an early round (\emph{e.g.}, at $t=1$) may alter the opponents' observations later, and thereby change their subsequent play well beyond $m$ rounds, even when each player uses a finite-memory strategy. In the repeated-game setting, \cite{zinkevich2005response-regret} proposed the notion of \emph{Response Regret}, which considered the \emph{mixed} strategy space that depends on the whole history of the play, making the minimization of Response Regret \emph{convex}. However, this is at the cost of exponentially large space in the history length to store the strategies, and thus computationally intractable. Subsequently, \cite{arora2018policy} extended the Policy Regret notion in  \cite{arora2012online} to the game setting. However, the comparator in the regret considered there is restricted to \emph{constant} actions and thus neither adaptive nor dynamic. Moreover, to achieve sublinear policy regret, \cite{arora2018policy} assumed that the opponents' strategies are insensitive to recent histories, which further restricted the adaptability and thus the power of opponents. For a more detailed related work discussion, we refer to Appendix \ref{sec:related-work}.}%

{Motivated by these prior studies, our paper focuses on the repeated-game setting, and aims at developing a regret notion that is natural and native to such a setting, while retaining as much generality as possible, \emph{e.g.,} allowing time-varying/dynamic comparator strategies and few constraints on all players' strategy spaces, while maintaining computational tractability and the possibility of finding better equilibrium solutions by minimizing such a notion. We summarize our contributions as follows.}

\paragraph{Our Contributions.}
Our main contributions are four-fold:  (i) For repeated games, we advocate  {a new and natural} notion of   {\tt Repeated Policy Regret (RP-Regret)}, {as a performance metric that {measures the difference between the \emph{realized} and the \emph{best-in-hindsight} accumulated utility %
when all players can respond to the history of play.}
It allows the opponents to be adaptive, and the comparator strategies to be time-varying, \emph{i.e.}, dynamic;}
(ii) We prove a series of necessary conditions for obtaining a sublinear {\tt RP-Regret}, on the variation of the player's comparator strategies in the regret definition, as well as the memories of both the comparator and the opponents' strategies (cf. Table \ref{table:hardness-result});
(iii) In light of the \emph{non-convexity} in minimizing {\tt RP-Regret}, we
propose three algorithms: (1) one that is based on a certain optimization oracle, as assumed in some prior work in  online non-convex learning; (2) one that minimizes a \emph{linearized} surrogate of {\tt RP-Regret} at each timestep, {\tt Local Repeated Policy Regret (LRP-Regret)}, that is convex; (3) one that directly minimizes {\tt RP-Regret} when the opponents change strategies slowly;
(iv) We establish the relationship between the minimization of our regret notions and the computation of certain equilibria of the repeated game. Additionally, we also provide experimental results to demonstrate the advantage of our new regret notions in finding better cooperative solutions with
higher utility for the players in games like the
Stag-Hunt.

\paragraph{Challenges \& Our Techniques.} Our notion of {\tt  RP-Regret} is general and particularly devised for the repeated-game setting, but at the cost of having a \emph{non-convex} objective to minimize due to the memory of players' strategies. To address the non-convexity issue, we develop several approaches for  {\tt RP-Regret} minimization: (i) we resort to certain non-convex optimization oracles as in \cite{suggala2020online-ONCO-oracle}; (ii) we linearize the expected loss at each timestep to \emph{convexify} it; (iii) we lift the variable dimension by reformulating the repeated game as a \emph{Markov game}, yielding a convex objective in the \emph{occupancy measure space} for the regret-minimizer.
In particular, for (iii) we have developed new techniques to address the following challenges:  Firstly,
the occupancy measure by nature incorporates the strategies of both the regret-minimizer and the opponents in the game, which does not align with our setting where the opponents are \emph{not} controlled by the regret-minimizer. Therefore, we need to optimize the occupancy measure while keeping the opponents' strategies extracted from the occupancy measure close to the actual ones. Due to the online nature of game-playing, we \emph{cannot} know the opponents' strategies at timestep $t$ before we propose the occupancy measure at timestep $t$. Naively projecting to the occupancy measure space corresponding to the opponents' strategies at the previous timestep $t-1$ will cause the error to accumulate and eventually blow up. As a result, we carefully design constraints for the occupancy measure, which can provably keep the first-order difference between the extracted opponents' strategies and the actual ones bounded, and then resort to the framework of \emph{online learning with time-varying constraints} to solve it;  Secondly, we convert the violation of constraints during online learning to the {\tt RP-Regret} upper bound; Thirdly, the original constraint on the ``forgetfulness'' of the behavioral-form strategies is \emph{non-convex} with respect to the occupancy measure, addressing which requires us to redesign the constraint for the behavioral-form strategies.

\paragraph{Motivating Examples.} Now we present
 two examples to illustrate the advantages of our {\tt RP-Regret}
 and {\tt LRP-Regret}
 over the classical notion of external regret in playing repeated games. Due to space constraints, {Example 2} can be found in \S\ref{appendix:motivating_example}. %

\begin{example}[Existence of A Better Strategy that is  No-{\tt RP-Regret}]
\label{example:IPD}
    In Iterated Prisoner's Dilemma \citep{axelrod1981evolution-prisoner} (cf. utility matrix in  \Cref{fig:stag-hunt-utility}), when both players are minimizing the \emph{external regret}  (\emph{i.e.,} obtaining regret sublinear in time),  {the only strategy that the time-average strategies converge to} is \emph{defect-defect} (the only CCE of PD, since defect is a strictly dominant strategy for both players), with a utility of 0.2 for each player.
    Although the well-known \emph{tit-for-tat} strategy (starting with Cooperate and mimicking the opponent's action in the previous round) is an NE of the IPD with infinite rounds and enjoys a higher  time-average utility of $0.6$, it, however, suffers \emph{linear}  external regret (cf. Appendix \ref{appendix:proof-example-IPD}).
    In contrast, when both players always play tit-for-tat, they will enjoy sublinear regret in terms of our new regret notion, {\tt RP-Regret} (to be formally defined in \S\ref{sec:problem-setup}), with the higher time-average utility (of $0.6$) (cf. Appendix \ref{appendix:proof-example-IPD}). In other words, our {\tt RP-Regret} can better capture this effective and cooperation-promoting strategy of tit-for-tat in this case, than the classical external regret. %
\end{example}

\section{Preliminaries}\label{sec:prelim}

\paragraph{Notation.} {We use $\RR^n$, $\RR^n_{\geq 0}$, and $\RR^n_{>0}$ to denote the space of $n$-dimensional real vectors, real vectors with non-negative elements, and those with positive elements, respectively.} For any integer $n$, let $[n]\coloneqq\cbr{1,2,\dots, n}$.
For any vector $\bx\in\RR^n$, we use $x_i$ to denote the \textit{i}-th element in $\bx=(x_1,\cdots,x_n)^\top$  and $\nbr{\bx}_p$ to denote the \textit{p}-norm. By default, we use $\nbr{\bx}$ to denote the $\ell_2$-norm of $\bx$. We define  $\Delta_m \coloneqq \{\bx\in[0,1]^m:\sum_{i=1}^m x_i=1\}$ as  the $(m-1)$-dimensional probability simplex. For a set $S$, we will use $S^n$ to denote the Cartesian product of $S$ for $n$ times.
For any convex {and differentiable} function $\psi^S\colon S\to \RR$, we can define the associated Bregman divergence as $D_{\psi^S}(\bx,\by)=\psi^S(\bx)-\psi^S(\by)-\inner{\nabla \psi^S(\by)}{\bx-\by}$. When $D_{\psi^S}(\bx,\by)\geq \frac{k}{2}\nbr{\bx-\by}^2$, we will call $\psi^S$ $k$-strongly convex. For any vector $\bx$ and positive integers $i<j$, we will use $\bx_{i:j}$ to denote the slicing $(x_i,x_{i+1},...,x_{j-1},x_j)$ of $\bx$. For any {finite} set $S$, we will use $|S|$ to denote its cardinality. We use $\emptyset$ to denote an empty set, and use  $\NN$ {and  $\NN_{>0}$} to denote the set of non-negative and positive integers, respectively. For any argument, $\ind(\text{argument})=1$ when the argument holds, and equals $\ind(\text{argument})=0$ otherwise. For any two integers $a\in\NN,b\in\NN_{>0}$, we use $a\% b$ to denote the remainder of $a$ divided by $b$.

\paragraph{Repeated (Matrix) Games.} Let $\cN=\{1,2,...,N\}$  denote the set of all players with  $N\geq 2$. Then, we use $\cA_i$ to denote the action set of player $i\in \cN$, and $\cA$ to denote the joint action set of all players, \emph{i.e.}, $\cA=\prod_{i=1}^N \cA_i$. For any joint action $\ba\in\cA$, we also call it an \emph{action profile}.
We use $\cA_{-i}$ to denote the joint action set of all the players except player $i$. {The game is repeated across timesteps $t\geq 1$.}
At every timestep $t$, each player $i$ chooses an action $a_{t,i}\in\cA_i$ individually, and incurs a {loss\footnote{To follow the convention of online learning, we use \emph{loss} instead of \emph{utility/payoff} in the rest of the paper.} $\cL_i(\ba_t)\in[0,1]$,} where $\ba_t=(a_{t,1},a_{t,2},...,a_{t,N})$. {Note that $\{\cL_i\}_{i\in\cN}$ is referred to as the set of loss \emph{matrices} (which explains the name of repeated \emph{matrix}  games) when $N=2$, and as the set of loss \emph{tensors} in general when $N>2$.}

\paragraph{History and Strategy.}
{Throughout the paper, we use $\bh$ to denote a vector of {history} actions, of which every element is a joint action of the $N$ players. We may also refer to  $\bh$ as \emph{history} for short.}
We use $L(\bh)$ to denote the length of $\bh$, and $\bh_k=(a_{k,1},a_{k,2},...,a_{k,N})\in\cA$ to denote the $k^{\rm th}$ element of $\bh$. Note that $k\in\cbr{1,2,...,L(\bh)}$.
Moreover, we use $\bh_{s:k}=(\bh_s,\bh_{s+1},...,\bh_k)$ to denote the slice of $\bh$ for $s\leq k$, which is $\emptyset$ if $s>k$. For notational convenience, we use $(\bh,\ba)$ or $(\bh,\bh')$ to denote the \emph{concatenation} of a vector and an action profile/a vector. We use $\cH_m\coloneqq \cbr{\bh=(\bh_1,\bh_2,...,\bh_m)\mid \forall k=1,2,...,m, \bh_k\in\cA}$ to denote the set of all histories of length $m$. For convenience, we define $\cH \coloneqq \bigcup_{i=0}^\infty \cH_i$, where $\cH_0=\cbr{\emptyset}$ contains the unique history $\bh=\emptyset$ with length $L(\bh)=0$.

We use $\bpi\ui\colon \cH\to \Delta_{|\cA_i|}$ to denote the (history-dependent) strategy of player $i$.
 $\pi\ui(a_i\given \bh)$ is the probability of choosing action $a_i\in\cA_i$ conditioned on observing the history $\bh$. Let $\cX\ui \coloneqq \cbr{\bpi\ui\given \bpi\ui\colon \cH\to \Delta_{|\cA_i|}}$ be the space $\bpi\ui$ lies in and {$\cX\coloneqq \cbr{\bpi\given \bpi\colon \cH\to \Delta_{|\cA|}}$} be the space of the joint strategy profile $\bpi$ lies in. Unless we specify $\bpi$ to be a \emph{correlated} strategy explicitly, we use the notation of  $\pi(\ba\given\bh)$ to denote $\pi(\ba\given\bh)=\prod_{i=1}^N \pi\ui(a_i\given\bh)$, \emph{i.e.}, as a \emph{product} strategy. For convenience, we also use $\bpi^{(-i)}$ to denote the strategy profile of all players except player $i$, \emph{i.e.}, $\pi\uni(\ba_{-i}\given\bh)=\prod_{j\not=i}\pi^{(j)}(a_j\given\bh)$. In  repeated  games, at each timestep $t$, {based on some history $\bh$,} each player $i$ will draw their action $a_{t,i}$ from their strategy at timestep $t$, {denoted as $\bpi\ui_t$}, \emph{i.e.,} {sampling from} $\pi_t\ui(\cdot\given\bh)$. {Note that $\bh$ here could be either the \emph{full} history from time $1$ to $t-1$, or some \emph{truncated} portion of it.} Lastly, we define $\nbr{\bpi\ui}_p \coloneqq \rbr{\sum_{a_i\in\cA_i,\bh\in\cH} \abr{\pi\ui(a_i\given \bh)}^p}^{1/p}$. %

\paragraph{Expected Loss in Repeated Games.} Firstly, for any history $\bh\in\cH$ and a vector of strategies $\bpi_{1:L(\bh)}$ of length $L(\bh)$, we define $\Pr(\bh\given \bh';\bpi_{1:L(\bh)}) \coloneqq \prod_{s=1}^{L(\bh)}\prod_{i=1}^N \pi_s\ui(h_{s,i}\given (\bh',\bh_{1:s-1}))$ as the probability that $\bh$ occurs when all players observe $\bh'\in\cH$ initially, and sample action profile $\bh_k\in\cA$ according to $\bpi_k$ for $k\in\cbr{1,2,...,L(\bh)}$. For simplicity, we define $\Pr(\bh;\bpi_{1:L}) \coloneqq \Pr(\bh\given \emptyset;\bpi_{1:L})$ for $\bpi_{1:L}$ of any length $L>0$. Then, we can define the \emph{expected loss} of player $i$ when all the players follow a sequence of strategies in $\bpi_{1:m+1}$ for $m+1$ steps, and when observing some
$\bh'\in\cH$ initially,  as $f_i^m(\bpi_{1:m+1}\given \bh')\coloneqq \sum_{\ba\in \cA} \cL_i(\ba)\sum_{\bh\in\cH_m} \Pr((\bh,\ba)\given \bh';\bpi_{1:m+1})$. Note that the input sequence of the strategy profiles {$\bpi_{1:m+1}$ in $f_i^m(\bpi_{1:m+1}\given \bh')$}
always has the length of $m+1$, the same as the length of $(\bh, \ba)$, where $m=L(\bh)$. {When $\bh=\emptyset$, we define $m=L(\bh)=0$.} Additionally, we define $f_i^m(\bpi_{1:m+1}) \coloneqq f_i^m(\bpi_{1:m+1}\given \emptyset)$ for simplicity. Finally, we may denote $\bpi_{1:m+1}$ simply as $\bpi$ when the subscript is clear from the context. Unlike one-shot normal-form games, the expected loss in repeated games is \emph{non-convex} with respect to the strategies $\bpi_{1:m+1}$.
Therefore, to optimize the expected loss, we need to either assume a non-convex optimization oracle (\Cref{sec:oracle-ONCO}), or linearize and thus convexify the function (\Cref{sec:local-RP-Regret-Minimizer}), or lift the dimension of the variables (\Cref{sec:RP-Regret-Minimizer}).

\section{A New Metric: {\tt Repeated Policy Regret (RP-Regret)}}
\label{sec:problem-setup}

\subsection{{\tt RP-Regret} Definition}

{We consider the case in repeated games where the opponents are \emph{adaptive}, \emph{i.e.}, they can make decisions by responding to histories of play. The player is also \emph{aware} that their deviations in action could trigger responses from the opponents. To handle this setting, we}
advocate a new metric of {\tt Repeated Policy Regret (RP-Regret)} in repeated games.
Without loss of generality, we will consider the player of interest as \emph{player 1}, and we refer to all other players as the \emph{opponents}, as {the definition of {\tt RP-Regret} applies to all players.}
Specifically, for player 1, we first define the \emph{accumulated expected loss} $J_T(\bpi_{1:T})$ of the {joint} strategy $\bpi_{1:T}$ when  playing $T$ rounds of the game:
\vspace{-10pt}
\begin{align}\label{equ:def_J_f}
    &J_T(\bpi_{1:T}) := \sum_{t=1}^T f^{t-1}(\bpi_{1:t}),
\end{align}
\vspace{-8pt}

\noindent where, for notational simplicity, we
omit the subscript $1$ of $f_1^m(\bpi_{1:m+1})$.
 Therefore, an intuitive measure of performance would be comparing $J_T(\bpi_{1:T})$ with the \emph{best-response} expected loss of   $\min_{\hat\bpi\uo_{1:T}\in(\cX\uo)^T} J_T((\hat\bpi\uo_{1:T}, \bpi^{(-1)}_{1:T}))$, where we recall that $\cX\ui$ is the space in which the strategy of player $i$ lies. In the following, we refer to  $\hat\bpi\uo_{1:T}$ as the \emph{comparator}.
Then, we can define the {\tt RP-Regret} as
\begin{align}
    &R_T:= J_T(\bpi_{1:T})-\min_{\hat\bpi\uo_{1:T}\in \cC_T\uo} J_T((\hat\bpi\uo_{1:T}, \bpi^{(-1)}_{1:T}))\label{eq:RP-Regret-def}
\end{align}
\vspace{-8pt}

\noindent where $T\in\NN_{>0}$ and $\cC_T\ui\subseteq (\cX\ui)^T$ is the space that the \emph{comparator} $\hat\bpi_{1:T}\ui$ of player $i$ lies in,
subject to some constraints (to be specified later) that make  $\cC_T\ui$ potentially smaller than $(\cX\ui)^T$, and thus lead to a weaker comparator for potentially better tractability, as we shall see later.

{\tt RP-Regret} captures \emph{how much better we could have done if we had chosen $\hat\bpi_{1:T}\uo$ when all players are adaptive} (\emph{i.e.,} they are aware of the history). In contrast to external regret, the deviation to the comparator strategy $\hat\bpi_t\uo$ at timestep $t$ will not only affect the expected loss at timestep $t$ but also all expected losses afterward. This is because the distribution over histories has changed, and the adaptive opponents may change their strategies correspondingly.
{\tt RP-Regret} coincides with the \emph{adaptive regret}\footnote{We are aware that there is another definition of adaptive regret \citep{hazan2007logarithmic-adaptive-regret,daniely2015strongly-adaptive-regret,zhang2018dynamic-adaptive-regret}, which refers to the maximum external regret over all intervals $[r,s]\subseteq [T]$. However, given the context of this paper, we will only discuss the adaptive regret for repeated games.} defined in \citet{loftin2022impossibility-RP-Regret}, when the  strategies $\bpi_t\uo$ {for all $t$}  have no  restrictions on {the memory used}, \emph{i.e.}, can react arbitrarily {differently} when observing different histories. However, the regret is linear in the worst case, as shown in Lemma \ref{lemma:w1wo2} and Lemma \ref{lemma:w1wo2-bound} (related  {but different} hardness results also appeared in \citet{schlag2012impossibility,loftin2022impossibility-RP-Regret}). Therefore, we will set additional restrictions on the memory, so that players should behave similarly when observing similar histories (to be discussed in detail in \S\ref{sec:nece_cond}). In this case, {\tt RP-Regret} differs from the adaptive regret defined in \citet{loftin2022impossibility-RP-Regret}, where
the strategies of all players may now change over time.

{More relatedly,} the policy regret in \cite{arora2018policy} restricts the comparator $\hat\bpi_{1:T}\uo$ to be a fixed action across all timesteps as in external regret. Moreover, for some $m>0$ and any history $\bh\in\cH,\tilde{\bh},\bar{\bh}\in\cH_m$, they assume $\sum_{t=m}^T \nbr{\pi_t\uno(\cdot| (\bh, \tilde{\bh})) - \pi_t\uno(\cdot| (\bh, \bar{\bh}))}$ is sublinear in $T$. In other words, the opponents' strategies do not depend on recent history. {This assumption further restricts the adaptability and thus the power of the opponents, limiting its applicability in repeated game playing, the focus of our paper.}

\subsection{When is Minimizing {\tt RP-Regret} Possible?}
\label{sec:nece_cond}

Being fairly natural and native to repeated game playing,
the notion of {\tt RP-Regret} can be hard to minimize. To understand its fundamental limits, we first identify two basic conditions that can be proved \emph{necessary} for achieving a {\tt RP-Regret} sublinear in $T$. The first one is a \emph{variation} condition, which restricts the comparator's strategy from changing too fast across different timesteps. The second one {concerns the imperfect memories of all the players and the comparator.}

\begin{condition} [Sublinear Variation]
\label{assumption:1-variation}
For a given player $i\in\cN$, we say that the strategy $\bpi_{1:T}\ui$ of player $i$ \emph{changes slowly} with \emph{sublinear variation}, when there exists a constant $p\in[0,1)$ such  that
    $
        \sum_{t=2}^T \nbr{\big(\pi_{t-1}\ui(a_i\given \bh)-\pi_{t}\ui(a_i\given \bh)\big)_{a_i\in\cA_i,\bh\in\cH}}_\infty\leq O(T^p).$
\end{condition}

Condition \ref{assumption:1-variation} on the comparator is a common assumption in the literature of \emph{dynamic regret}  minimization  \citep{zinkevich2003online-OGD-original-paper,zhang2018adaptive,zhang2018dynamic,zhao2022non-dynamic-policy-regret}, which restricts the power of the comparator in order to achieve  sublinear dynamic regret. Note that the well-known external regret also implicitly adopts such a condition with the variation of the comparator being zero.

\begin{condition}[Imperfect Recall \citep{piccione1997interpretation-imperfect-recall, waugh2009practical-imperfect-recall}]
\label{assumption:imperfect-memory}
    For a given player $i\in\cN$, we say that
    player $i$ {has} \emph{imperfect recall}, when they  cannot remember all their histories perfectly. In particular, player $i$ cannot perfectly differentiate the observed histories, and thus cannot choose arbitrarily distinct strategies of $\bpi_{1:T}\ui$  for each different history encountered.
\end{condition}

{As a surrogate to \emph{quantitatively} characterize}  Condition \ref{assumption:imperfect-memory}, we propose Condition \ref{assumption:2-forgetful} below, {which instantiates the   ``imperfect'' recall as having  \emph{finite-memory} with an \emph{exponential decay} property.}

\begin{condition}[Exponential Decay Memory (EDM)]
\label{assumption:2-forgetful}
    For a given player $i\in\cN$, the strategy $\bpi_{1:T}\ui$ of player $i$ focuses on only the latest histories {due to}
    an exponential decay memory. Formally, for any timestep $t\in [T]$
    \small
    \begin{align}
    \forall \bh,\bar\bh,\tilde\bh\in\cH,a_i\in\cA_i,\qquad 1-\gamma^{L(\bh)+1}\leq\frac{\pi_t\ui(a_i\given(\tilde\bh,\bh))}{\pi_t\ui(a_i\given(\bar\bh,\bh))}\leq \frac{1}{1-\gamma^{L(\bh)+1}}\label{eq:assumption2-forgetful}
    \end{align}
    \normalsize
    for some constant $\gamma\in[0,1)$. Let $\cX\ui_{\gamma}$ denote the space of all $\bpi\ui\in\cX\ui$ that satisfy the condition.%
\end{condition}

Condition \ref{assumption:2-forgetful} quantitatively bounds how sensitive the strategy is to the distant past: once the most recent suffix $\bh$ is fixed, changing the earlier prefix from $\tilde\bh$ to $\bar\bh$ can only change each conditional action probability by a multiplicative factor that approaches $1$ at an exponential rate in the length of $\bh$, as denoted by $L\rbr{\bh}$. Thus, for any target accuracy $\epsilon>0$, it suffices to retain a suffix of length $L=O\rbr{\log\rbr{1/\epsilon}}$ to make the effect of earlier history negligible. We adopt exponential (rather than polynomial) decay precisely to ensure that this effective memory length grows only logarithmically with $1/\epsilon$, which ensures efficient computation. This assumption is satisfied by several regularized mirror-descent-type updates, see, \emph{e.g.}, \cite{cen2021fast,liu2022power-reg,DBLP:conf/iclr/SokotaDKLLMBK23-MMD}.

An alternative {to quantify imperfect recall} is {via} \emph{bounded memory} \citep{chakraborty2014multiagent-bounded-memory}, where strategies
depend only on the last $m$ rounds for a fixed constant $m$.
However, bounded dependence on the last $m$ observations does \emph{not} prevent a player from using its actions as an
information carrier: by choosing actions according to a protocol, the player can encode information about earlier
events into the recent suffix that future decisions can still access.
As shown in the proof of Lemma \ref{lemma:wo1w2}, this phenomenon can preclude sublinear regret even when $m=1$.
Condition \ref{assumption:2-forgetful} rules out such near-perfect ``state passing'' by forcing policies that share the same recent suffix to be nearly indistinguishable regardless of the earlier prefix, with the indistinguishability improving exponentially in the suffix length.

We summarize the hardness results in Table \ref{table:hardness-result} and postpone the full statements to Appendix \ref{appendix:hardness-result}. Although in Table \ref{table:hardness-result} we only state the necessity of Condition \ref{assumption:imperfect-memory}, Condition \ref{assumption:2-forgetful} is ``\emph{almost necessary}'', as the difference between  these two conditions
{only originates from  those imperfect recall strategies}  with {the choice of} $\gamma=1$ in Condition \ref{assumption:2-forgetful} ({following the convention that ${1}/{0}=+\infty$}). {Given the results above, we now focus on minimizing the {\tt RP-Regret} by assuming the necessary conditions summarized in Table \ref{table:hardness-result} throughout, unless otherwise noted.}

\begin{table}[!t]
\centering
\begin{tabular}{|c|c|c|}
\hline
           & Condition \ref{assumption:1-variation} & Condition \ref{assumption:imperfect-memory} \\ \hline
Comparator &       \checkmark~~~Lemma \ref{lemma:wo1w2}            &\checkmark~~~Lemma \ref{lemma:w1wo2}         \\ \hline
Opponent   &       ---            &    \checkmark~~~Lemma \ref{lemma:w1wo2-bound}     \\ \hline
\end{tabular}
\caption{Summary of the necessary conditions for {\tt RP-Regret} minimization. The \checkmark denotes that the condition is required. The corresponding lemma proves that the condition is necessary. In Appendix \ref{sec:oracle-ONCO}, we show that a nonconvex optimization oracle yields sublinear {\tt RP-Regret} under these necessary conditions, with Condition \ref{assumption:imperfect-memory} replaced by Condition \ref{assumption:2-forgetful}.}
\label{table:hardness-result}
\vspace{-1em}
\end{table}

\section{{\tt RP-Regret} Minimization}\label{sec:4-rm-adaptive}

As mentioned before, unlike the traditional framework of \emph{online convex optimization with memory} \citep{merhav2002sequential-memory-regret-init,anava2015online-OCO-framework}, which exclusively focuses on handling
\emph{convex} loss functions, our loss is \emph{non-convex} with respect to the input argument due to the product of strategies in consecutive timesteps.
Hence, it is hard to minimize {\tt RP-Regret} directly due to the non-convexity.
We now propose three  different ways to minimize {\tt RP-Regret} as follows:
\begin{itemize}[nosep]
    \item Minimizing {\tt RP-Regret} directly with an \emph{oracle} that can optimize a non-convex function, where we only need the {necessary} conditions in Table \ref{table:hardness-result}, except that we use \Cref{assumption:2-forgetful} instead of \Cref{assumption:imperfect-memory}. {This setting with a non-convex optimization oracle, though it may be computationally intractable, can still provide some insights into the players' learning process, and has also been adopted in the literature \citep{suggala2020online-ONCO-oracle}}. This part is presented in detail in Appendix \ref{sec:oracle-ONCO}.%
    \item Minimizing a {surrogate}   notion of {\tt RP-Regret}, \emph{i.e.}, {\tt Local RP-Regret} ({\tt  LRP-Regret}), by \emph{convexifying} the {\tt RP-Regret} locally around the implemented strategies {during regret minimization}, which will be discussed in \S \ref{sec:local-RP-Regret-Minimizer}.
    \item Reformulating the problem as a \emph{Markov game}   \citep{shapley1953stochastic,filar2012competitive} under certain reparameterization, with an additional requirement that Condition \ref{assumption:1-variation} holds also for the \emph{opponent}  (not only for the comparator, as a necessary condition by Lemma \ref{lemma:wo1w2}). This part will be introduced  in \S\ref{sec:RP-Regret-Minimizer}.
\end{itemize}

The first bullet justifies that Condition \ref{assumption:1-variation} on the comparator and Condition \ref{assumption:2-forgetful} on both the comparator and the opponents are sufficient to achieve sublinear {\tt RP-Regret}, when {one leverages a non-convex optimization oracle,
as in online non-convex learning \citep{suggala2020online-ONCO-oracle}. This sufficiency result mirrors the necessity of these conditions (and their variants) in \S\ref{sec:nece_cond}.}
Hereafter, we will focus on introducing the other two approaches above  that account for computational efficiency.

\subsection{Minimizing a   Surrogate:  {\tt Local Repeated Policy Regret}}
\label{sec:local-RP-Regret-Minimizer}

Motivated by the one-step deviation principle \citep{watson2002strategy-one-step-deviate}, which means that the strategy profile of a
repeated game is a subgame perfect Nash equilibrium (SPNE) \emph{if and only if} no player can decrease their expected loss
by deviating from their original strategy via only {a single action at one round of the game,}
we propose the notion of {\tt Local Repeated Policy Regret (LRP-Regret)}. Specifically, instead of computing the regret by  \emph{globally} deviating from the strategy $\bpi\uo_{1:T}$ to $\hat\bpi\uo_{1:T}$, we compute the regret when the player only \emph{locally} deviates from $\bpi\uo_{1:T}$ at one timestep, which is formally given by %
\begin{align}\label{eq:local-RP-Regret-def}
    R_T^{\rm local} \coloneqq  \max_{\hat\bpi\uo_{1:T}\in\cC_T\uo} \sum_{s=1}^T \big(J_T(\bpi_{1:T})-J_T(\tilde\bpi^{(1),s}_{1:T}, \bpi^{(-1)}_{1:T})\big),\quad  \text{where~~}\tilde\bpi^{(1),s}_t=\begin{cases}
    \hat\bpi\uo_t&t=s\\
    \bpi\uo_t&t\not=s,
    \end{cases}
\end{align}
and $\cC_T\uo\subseteq\rbr{\cX\uo_\gamma}^T$ is the space of the comparator $\hat\bpi_{1:T}\uo$, which satisfies Condition \ref{assumption:1-variation} and Condition \ref{assumption:2-forgetful} with $\gamma\leq \frac{1}{2(N+2)}$.  $\lim_{T\to\infty} \frac{R_T^{\rm local}}{T}=0$ implies that when player 1 deviates to $\hat\bpi_{1:T}\uo$ at only one timestep $t\in\cbr{1,2,...,T}$, the \emph{accumulated} loss $J_T$ upon averaging over all timesteps will not decrease.

Interestingly, one can verify that Lemmas  \ref{lemma:wo1w2}, \ref{lemma:w1wo2}, and \ref{lemma:w1wo2-bound} still hold for this weaker regret notion, which implies that the necessary conditions in Table \ref{table:hardness-result} are still necessary {for {\tt LRP-Regret} minimization. We defer the formal results and proofs to  Appendix \ref{appendix:hardness-local-RP-Regret}}.

For the {\tt LRP-Regret} given in \Cref{eq:local-RP-Regret-def}, it is easy to verify that the expected loss at timestep $t$, $f_t^{m,{\rm local}}\colon \Delta_{|\cH_m|\times |\cA_1|}\to \RR$, can be written as
\begin{align}
    f_t^{m,{\rm local}}(\bar \bpi\uo_{1:t}) \coloneqq (T-m-1) f^m(\bpi\uo_{t-m:t},\bpi\uno_{t-m:t})+\sum_{s=t-m}^t f^m((\bpi\uo_{t-m:s-1}, \bar\bpi\uo_s, \bpi\uo_{s+1:t}),\bpi\uno_{t-m:t}).\label{eq:local-loss-def}
\end{align}
Moreover, for an arbitrary strategy vector $\bar\bpi\uo_{1:T}$ of length $T$, the corresponding expected loss at timestep $t$, $f_t^{t-1,{\rm local}}(\bar \bpi\uo_{1:t})$, is linear with respect to $\bar\bpi_t\uo$. Therefore, {one may update the strategy using the simple  projected gradient descent (PGD) algorithm as:}
\begin{align}
    \pi_{t+1}\uo=\Proj{\cX\uo_{\gamma}}{\pi_t\uo - \eta\nabla f_t^{m,{\rm local}}(\bpi\uo_{1:t})},\label{eq:PGD-def}
\end{align}
where $\eta>0$ is the learning rate.

\begin{theorem}
\label{theorem:local-regret}
    Consider the update rule in \Cref{eq:PGD-def}. By choosing learning rate $\eta=\Theta\rbr{\sqrt{\frac{P_T}{T}}}$ and $\gamma\leq \frac{1}{2(N+2)}$, when all the opponents satisfy  Condition \ref{assumption:2-forgetful}, we have  $
        \frac{R_T^{\rm local}}{T} \leq  \tilde O\rbr{|\cA|^{m+1}\sqrt{P_T / T}+C_m^\gamma},$
    where $P_T$ is an upper bound of the variation of the comparator strategies, such that any sequence of the comparator $\hat\bpi\uo_{1:T}$ must satisfy $\sum_{t=2}^T \nbr{\hat\bpi\uo_{t-1}-\hat\bpi\uo_t}_\infty\leq P_T$, and $C_m^\gamma=(2N+1)^{m+1}\gamma^{m+1}$.
\end{theorem}

The proof is postponed to Appendix \ref{appendix:proof-of-local-regret-theorem}. In \Cref{theorem:local-regret}, the $\tilde O$ notation hides factors polynomial in $m$ (and logarithmic in $T$). The exponential dependence on $|\cA|^{m+1}$ stems from maintaining an $\rbr{m+1}$-step history. Consequently, if the comparator variation $P_T$ is sublinear in $T$, for any $\epsilon>0$, $\frac{R_T^{\rm local}}{T} \leq\epsilon$ when $m=\Theta(\log\frac{1}{\epsilon})$.

\subsection{Minimizing {\tt RP-Regret} with Slowly-Changing Opponents}
\label{sec:RP-Regret-Minimizer}

We also propose  another approach  that can \emph{directly minimize} the original {\tt RP-Regret} \Cref{eq:RP-Regret-def},
under the additional condition that the opponent changes their strategies slowly, \emph{i.e.}, Condition \ref{assumption:1-variation} also holds for the opponents' strategies.

\subsubsection{Reformulating Repeated Games with Bounded Memory
as Markov Games}
\label{sec:reformulating-mg}

We will show that the expected loss at each timestep can be approximated by the expected loss of an infinite-horizon average-loss {Markov/stochastic  game}  \citep{shapley1953stochastic,gillette1957stochastic,filar2012competitive},
and develop a regret minimization algorithm based on this reformulation.

Naively, one can use {the whole} history as the \emph{state} of the Markov game. However, the state space will increase exponentially as the timestep increases. Hence, we will consider the case where all the players have an $M$-bounded memory  with some fixed $M$, such that they only make decisions conditioned on the history of length $M$. In Lemma \ref{lemma:original-to-markov-value}, we will show that under Condition \ref{assumption:2-forgetful}, we can convert the unbounded memory of all players to $M$-bounded memory, with a small approximation error.
Particularly, for all players $i\in\cN$, we can convert a strategy $\pi\ui$ with unbounded memory to a strategy $\overbar\pi\ui$ with $M$-bounded memory by letting
\begin{align}
    &\forall a_i\in\cA_i,\bh\in\cH,~~~~~\overbar\pi\ui(a_i\given\bh)=\begin{cases}
        \pi\ui(a_i\given\bh_{L(\bh)-M+1:L(\bh)})&L(\bh)\geq M\\
       \frac{1}{|\cH_{M-L(\bh)}|} \sum_{\bh'\in\cH_{M-L(\bh)}}\pi\ui(a_i\given (\bh',\bh))&L(\bh)<M.
    \end{cases}
\end{align}
In the remainder of this section, we will only focus on strategies with $M$-bounded memory.

\begin{definition}[Induced Markov Game]
\label{def:induced-MDP}
For a fixed $M\in \NN$, and a repeated game given in \S\ref{sec:prelim}, we can define the induced Markov game as follows:
\begin{itemize}[nosep]
    \item Number of players: $N$;
    \item State space: $\cS \coloneqq \cH_M$;
    \item Action space (which coincides with the action space of the matrix game): $\cA_i$ for player $i$;
    \item Transition probability: {for any $\bh',\bh\in \cH_M$, and $\ba\in\cA$,} %
   $\Pr(\bh'\given\bh,\ba) \coloneqq \ind(\bh'_{M}=\ba{\rm~and~}\bh'_{1:M-1}=\bh_{2:M})$;
    \item Stage  loss: $\cL_i(\bh,\ba)  \coloneqq  \cL_i(\ba)$ for any $i=1,2,...,N$.
\end{itemize}
\end{definition}

By definition, the expected time-average loss of the Markov Game in Definition \ref{def:induced-MDP}\footnote{The expected time-average loss always exists and does not depend on the initial distribution when strategies of all players satisfy Condition \ref{assumption:2-forgetful} (see Lemma \ref{lemma:markov-connect} for a formal proof).} is equivalent to the expected average loss of the infinitely repeated matrix game when the joint (Markov) strategy $\bpi$ is the same at each timestep: {in particular, for player $1$} %
\begin{align}
    \lim_{T\to\infty}\EE_{\bh_{t+1}\sim \Pr(\cdot\given\bh_{t},\ba_t),\ba_t\sim\bpi(\cdot\given \bh_{t})}\sbr{\frac{1}{T}\sum_{t=0}^{T-1} \cL_1(\bh_{t},\ba_t)}
    = \lim_{T\to\infty} \frac{1}{T}\sum_{t=0}^{T-1} f^t(\underbrace{(\bpi,...,\bpi)}_{t+1})\eqqcolon f^\infty(\bpi),\label{eq:f-infty-def}
\end{align}
where $\bh_{t}\in\cH_M$ denotes the state at timestep $t$.

\subsubsection{Occupancy-Measure-based Regret Minimization in the Markov Game}

{The key challenge in learning in the induced Markov game is the \emph{non-convexity} of the expected time-average loss with respect to the (stationary) strategy $\bpi$. To \emph{convexify} the problem, we propose to} change the variable from the strategy to the occupancy measure of the MG. Thus, instead of updating the strategy $\pi_t^{\uo}$, we will update the occupancy measure $\bq_t$ at each timestep $t$.

For an infinite-horizon MG, its expected loss can be represented as a linear function with respect to the occupancy measure \citep{puterman2014markov}. Formally, for any history $\bh\in\cH_M$ and joint action $\ba\in\cA$, the occupancy measure $\bq^{\pi}$ that corresponds to joint strategy $\pi$ can be written as
\vspace{-5pt}
\begin{align}
q^{\bpi}(\bh,\ba) \coloneqq \EE_{\bh_0\sim\mu_0,\bpi}\left[\lim_{T\to\infty}\frac{1}{T}\sum_{t=0}^{T-1}\ind(\bh_t=\bh)\right]\cdot \pi(\ba\given \bh)\label{eq:def-occupancy-measure}
\end{align}%

\vspace{-10pt}
\noindent where $\mu_0$ is the initial distribution over $\cS=\cH_M$\footnote{$\mu_0$ is omitted here since $\bq^{\pi}$ is invariant to the initial distribution under  Condition \ref{assumption:2-forgetful}. Please refer to Appendix \ref{subsection:milder-constraint} for a detailed discussion.}.
Moreover, it is known that the joint strategy can be recovered as $\pi(\ba\given\bh)=\frac{q^{\bpi}(\bh,\ba)}{\sum_{\ba'\in\cA} q^{\bpi}(\bh,\ba')}$,
and there is a correspondence between $\bpi$ and $\bq^{\pi}$ \citep{puterman2014markov}.
We highlight that
{in the regret-minimization procedure concerned in this subsection,}
only the strategy $\bpi_t\uo$ of player $1$ is controlled by {the (regret minimization)  algorithm}, which is determined by the occupancy measure, $\bq_t$, at timestep $t$ proposed by the algorithm. All other players are adaptive opponents. Hence, we define the strategy of player 1 at timestep $t$ as $\pi_t\uo(a_1\given \bh)\coloneqq \frac{\sum_{\ba_{-1}'\in\cA_{-1}} q_t(\bh,(a_1,\ba_{-1}'))}{\sum_{\ba'\in\cA} q_t(\bh,\ba')}$ {for any $\bh\in \cH_M$ and $a_1\in\cA_1$}, where $\bq_t$ is controlled by the regret minimizer.

To ensure that $\frac{q_t(\bh,\ba)}{\sum_{\ba'\in\cA} q_t(\bh,\ba')}$ can always be represented as $\pi_t\uo(a_1\given\bh)\prod_{i=2}^N \pi_t\ui(a_i\given\bh)$ for all $\bh\in\cH_M,\ba\in\cA$, we additionally need the following constraints on $\bq_t$ at timestep $t$, as characterized by a constraint function $g_t\colon \RR^{|\cH_M|\times |\cA|}\to \RR$ defined below: for any $\bq\in \RR^{|\cH_M|\times |\cA|}$,
\begin{align}
    &g_t(\bq) \coloneqq \sum_{\bh\in\cH_M,\ba\in\cA} (-1)^{\ind(D_t(\bh,\ba,\bq_t)\leq 0)}D_t(\bh,\ba,\bq)\leq 0,\label{eq:def-constraint-no-timestep}\\
    &\text{where~~~~}D_t(\bh,\ba,\bq)\coloneqq  q(\bh,\ba)-\pi_t^{(-1)}(\ba_{-1}\given \bh)\sum_{\ba_{-1}'\in\cA_{-1}} q(\bh,(a_1,\ba_{-1}')).\notag
\end{align}
\vspace{-12pt}

\noindent It is straightforward to see that $g_t(\bq_t)\leq 0$, with equality if and only if  $\frac{q_t(\bh,\ba)}{\sum_{\ba'\in\cA} q_t(\bh,\ba')}=\pi_t\uo(a_1\given\bh)\prod_{i=2}^N \pi_t\ui(a_i\given\bh)$.  Note that the constraint at timestep $t$, $D_t(\bh,\ba,\bq^{\pi})$, is a linear constraint on $\bq^{\pi}$ determined after $\bq_t$ is proposed, which {falls into the realm} of online convex optimization with time-varying constraints \citep{paternain2016online,chen2017online,cao2018online-constrained-time-varying}, and {it will be the core of our algorithm}. 
Compared to simply letting $g_t(\bq) = \sum_{\bh\in\cH_M,\ba\in\cA} \abr{D_t(\bh,\ba,\bq)}{\leq 0}$, our novel constraints in \Cref{eq:def-constraint-no-timestep} can avoid the non-differentiable issue of the absolute value function at zero.

\subsubsection{{Convexifying}  Condition \ref{assumption:2-forgetful} {and the Overall Algorithm}}

The Markov game reformulation in \S\ref{sec:reformulating-mg}  provides an alternative way to approximate the expected loss  {under Conditions \ref{assumption:1-variation} and \ref{assumption:2-forgetful}.}  However,  the forgetful constraint in  Condition \ref{assumption:2-forgetful} becomes \emph{non-convex} with respect to the occupancy measure. Therefore, instead of enforcing Condition \ref{assumption:2-forgetful}, we {propose to}  enforce the following  weaker constraint, which is \emph{convex}  with respect to the occupancy measure and can be guaranteed when Condition \ref{assumption:2-forgetful} is satisfied.
\begin{condition}[{Convexification} of Condition \ref{assumption:2-forgetful}]
    \label{assumption:2-forgetful-prime}
    For a given player $i\in[N]$, the strategy $\bpi_{1:T}\ui$ of player $i$ has a bounded memory of length $M\in \NN$, and the player places at least probability {$\gamma\in(0,1]$} on an exploration strategy $\bnu\ui$, regardless of the observed history. Formally, for any timestep $t\in[T]$, we have $
        \forall  \bh\in\cH_M,a_i\in\cA_i,~\pi_t\ui(a_i\given\bh)\geq \gamma\bnu\ui(a_i)$
    for some $\gamma\in(0,1]$, where  $\bnu\ui\in\Delta_{|\cA_i|}$ is a distribution fixed for every history $\bh\in\cH_M$ of  player $i$.
\end{condition}

It is straightforward to see that when Condition \ref{assumption:2-forgetful} is satisfied, for every action $a_i\in\cA_i$, we must have either $\pi\ui(a_i\given \bh)>0$ for all $\bh\in\cH$ or equal to $0$ for all $\bh\in\cH$. Then, Condition \ref{assumption:2-forgetful-prime} is satisfied since the $\bnu\ui$ in the condition always exists  (note that the $\gamma$ in Condition \ref{assumption:2-forgetful} and Condition \ref{assumption:2-forgetful-prime} are different).%

In the following lemma, whose proof is deferred to Appendix \ref{sec:MDP-contraction-property}, we show that {the unary expected loss}  $f^K(\underbrace{(\bpi,...,\bpi)}_{K+1})$ can be {further} approximated by the average loss of {the induced  Markov game given by Definition \ref{def:induced-MDP}}.
Therefore, we will solve the
Markov game instead.
\begin{lemma}
\label{lemma:finite-f-to-infinite}
    When all the players satisfy Condition
    \ref{assumption:2-forgetful-prime} \footnote{In fact, this lemma also holds when all players satisfy a weaker condition (cf.  Condition \ref{assumption:L1-forgetful} in the Appendix).}, {the   average loss of the induced  Markov game given by Definition \ref{def:induced-MDP},} \emph{i.e.,} $f^\infty(\bpi)$ as defined in \Cref{eq:f-infty-def}, always exists and we have {that for any $K>0$}
    \small
    \begin{align}
        \bigg|f^K(\underbrace{(\bpi,...,\bpi)}_{K+1})-f^{\infty}(\bpi)\bigg|\leq 2\rbr{1-(\frac{\gamma^N}{|\cA|})^{M\useconstant{constant:go-back-root-length}}}^{\floor{\frac{K}{M\useconstant{constant:go-back-root-length}}}},
    \end{align}
    \normalsize
    where $\newconstant{constant:go-back-root-length}\coloneqq\log_2^2|\cH_M|+4\log_2|\cH_M|+3$.
\end{lemma}
For any initial distribution, the expected loss after $K+1$ rounds equals that after infinite rounds, up to an exponentially decreasing error. Moreover, the expected loss after infinite rounds is irrelevant to the initial distribution, which implies that the dependence on history is exponentially decaying.

We are now ready to introduce our algorithm, which is essentially online convex optimization over time-varying constraints (as depicted by \Cref{eq:def-constraint-no-timestep}), as
tabulated in \Cref{alg:self-play}.

\subsubsection{Theoretical Guarantees}
\label{sec:guarantee-occupancy-measure}

We show next that \Cref{alg:self-play} achieves a sublinear {\tt RP-Regret} and a sublinear {accumulated constraint violation} $\sum_{t=1}^T g_t(\bq_t)$. {A} detailed version {of the theorem} and {its} proof are in Appendix \ref{appendix:proof-of-occupancy-measure-theorem}.

\begin{theorem}[Informal]
\label{theorem:occupancy-measure-RM-informal}
    {Suppose player $1$ follows} Algorithm \ref{alg:self-play}, and all the players and  the comparator of player $1$ satisfy Condition \ref{assumption:2-forgetful-prime} with $\bnu\ui$ {being} the uniform strategy over $\Delta_{|\cA_i|}$, and {suppose}  $T$ is large enough {such that} $T\geq \tilde \Omega(\frac{\Delta_T}{\epsilon^4})$, where $\Delta_T$ is the summation of all opponents' and the comparator's variations over time, then we can guarantee that {\tt RP-Regret} satisfies
    $
    \frac{R_T}{T}\leq \epsilon$.
\end{theorem}
Note that to find a $T$ as a polynomial in $\frac{1}{\epsilon}$ and satisfy $T\geq \Omega(\frac{\Delta_T}{\epsilon^4})$ simultaneously, $\Delta_T$ needs to be sublinear in $T$. \Cref{assumption:2-forgetful-prime} is a linear constraint on the occupancy measure, which can be implemented efficiently in our algorithm by projection.

\section{Equilibrium Computation via {\tt RP-Regret} Minimization}
\label{sec:equilibrium-computation}

In this section, we investigate one important implication of our new regret notion -- \emph{equilibrium computation}, in repeated games. We will focus on the \emph{infinitely} repeated game setting, instead of the \emph{finitely} repeated one. This is because, for a one-shot matrix game, when its NE is unique, the only subgame perfect Nash equilibrium of its finitely repeated version is playing the NE of the one-shot game at every timestep \citep{benoit1984finitely}. Therefore, solving the NE of the finitely repeated game degenerates to solving the NE of the one-shot matrix game.
{On the other hand,} the uniqueness of NE for a one-shot matrix game is {not uncommon} --- for two-player zero-sum matrix games, the set of games with non-unique NEs has Lebesgue measure zero \citep{van1991stability-unique, bailey2018multiplicative-unique}.
{Hence, we will hereafter direct our attention to the infinitely repeated general-sum game setting.}

\subsection{Equilibria in Repeated Games}

{In light of our {\tt RP-Regret} definition, we first introduce the following notions of (subgame perfect) equilibria in infinitely repeated games. We start with the coarse correlated equilibrium.}

\begin{definition}
\label{def:approx-SPCCE}
\textnormal{\bf (Approximate Subgame Perfect Coarse Correlated Equilibrium (SPCCE) with Bounded Deviation)} A correlated strategy $\bpi_{1:\infty}$  is called an $\epsilon$-approximate SPCCE with $P_T$-bounded deviation in repeated games if it satisfies
    \begin{align*}
\limsup_{T\to+\infty}\max_{i\in\cN}\sup_{t_0\in \NN_{>0}}\sup_{\bh_0\in\cH_{t_0-1}} \frac{1}{T}\sum_{t=t_0}^{t_0+T-1} \Big(f_i^{t-t_0}(\bpi_{t_0:t}\given\bh_0) -f_i^{t-t_0}((\hat\bpi\ui_{t_0:t},\bpi\uni_{t_0:t})\given\bh_0)\Big)\leq \epsilon
    \end{align*}
    where
    $\hat\bpi_{1:T}\ui\in\cC_T\ui$ is an  arbitrary strategy of player  $i$ satisfying $\sum_{t=2}^T\nbr{\hat\bpi_t\ui-\hat\bpi_{t-1}\ui}_\infty\leq P_T$  for any arbitrary $T>0$, with  $P_T$ being a function of $T$. Recall that $f_i^m(\bpi_{1:m+1}\given \bh_0)$ is the expected loss of player  $i$ at timestep $m+1$ when observing $\bh_0$ initially {and executing $\bpi_{1:m+1}$ afterward}.
\end{definition}
Compared to standard subgame-perfect Nash equilibrium (SPNE), SPCCE admits correlated strategies. Moreover, SPNE (SPCCE) can be viewed as SPNE (SPCCE) with $\Theta(T)$-bounded deviation, since there is no restriction on the comparator.

{Similarly, we can define the corresponding \emph{Nash equilibrium} when the joint strategy $\bpi_{1:\infty}$ is a \emph{product}  strategy conditioned on each history.}

\begin{definition}
\label{def:approx-SPE}
\textnormal{\bf (Approximate Subgame Perfect Nash Equilibrium (SPNE) with Bounded Deviation)} A non-correlated product strategy $\bpi_{1:\infty}$ is called an $\epsilon$-approximate $P_T$-robust SPNE in repeated games when it is an $\epsilon$-approximate SPCCE with $P_T$-bounded deviation and $\pi_t(\ba\given\bh)=\prod_{i=1}^N \pi\ui_t(a_i\given\bh)$ for some $\{\bpi_{1:\infty}\ui\}_{i\in\cN}$ at any timestep $t=1,2,...$, and for any $\ba\in\cA,\bh\in\cH$.
\end{definition}

\begin{remark}[Bounded Deviation]
{Note that \emph{bounded deviation} means that the set of strategies to which each player can \emph{deviate} in the equilibrium definition, which corresponds to the \emph{comparator} in the regret definition, is restricted to satisfy certain variation bound, characterized by $P_T$. When $P_T=\Theta(T)$, the notation is exactly the classical SPNE (or SPCCE) \citep{roughgarden2010algorithmic-game-theory}.}
\end{remark}

\subsection{Relationship between {\tt RP-Regret} and Equilibria}
\label{sec:regret-equilibrium}

We now discuss the relationship between {\tt RP-Regret} and the aforementioned equilibrium notions.

\begin{theorem}[Equilibrium and {\tt RP-Regret}]
\label{lemma:RP-Regret-SPE}
    For a fixed $T_0{\in\NN_{>0}}$, when each player $i\in \cN$ obtains a sublinear {\tt RP-Regret} $R_{T_0}=O\rbr{T_0^p P_{T_0}^{1-p}}$ with strategies $\tilde\bpi\ui_{1:T_0}$ for some $p\in[0,1)$ against the comparator with $\sum_{t=2}^{T_0}\nbr{\hat\bpi_t\ui-\hat\bpi_{t-1}\ui}_\infty\leq P_{T_0}$ {for all $i\in\cN$},  and all {the} players and their comparators satisfy Condition \ref{assumption:2-forgetful-prime} until $T_0$. Then,  when $T\to+\infty$, for any timestep $t\in\NN_{>0}$, we will choose $\bpi_t\ui=\tilde\bpi_{(t-1)\% T_0+1}\ui$ for each $i\in\cN$. Then, $\bpi_{1:\infty}$ is an $O\rbr{\rbr{\frac{P_{T_0}}{T_0}}^{1-p}}$ approximate SPNE with $O(P_T)$-bounded deviation, where $P_T$ is the upper bound of {player $i$'s} comparator variation for any $i\in \cN$: $\sum_{t=2}^{T}\nbr{\hat\bpi_t\ui-\hat\bpi_{t-1}\ui}_\infty\leq P_{T}$.
\end{theorem}

The proof of \Cref{lemma:RP-Regret-SPE} is postponed to Appendix \ref{appendix:RP-Regret-to-SPE}. This theorem builds up the relationship between {approximate robust} SPNE (and thus SPCCE) and {\tt RP-Regret} {minimization}. When all the players can obtain a sublinear {\tt RP-Regret} against their comparators whose accumulated variations are bounded by $P_{T_0}$, then we can build an approximate SPNE with $P_T$-bounded deviation.

\begin{theorem}[Equilibrium and {\tt LRP-Regret}]
\label{lemma:LRP-Regret-SPE}
For a fixed $T_0{\in\NN_{>0}}$, when all the players can achieve a sublinear $R_{T_0}^{\rm local}\leq O(T_0^p)$ with strategies $\tilde\bpi\ui_{1:T_0}$ for some $p\in[0,1)$ against any comparator satisfying $\sum_{t=2}^{T_0}\nbr{\hat\bpi_t\ui-\hat\bpi_{t-1}\ui}_\infty\leq P_{T_0}=\Theta(T_0)$ {for all $i\in\cN$}, and all the players (including their comparators) satisfy Condition \ref{assumption:2-forgetful-prime} with $\bnu\ui$ as the uniform distribution over $\Delta_{|\cA_i|}$. Then, when $T\to+\infty$, for any timestep $t$, we will choose $\bpi_t\ui=\tilde\bpi_{(t-1)\% T_0+1}\ui$ for any $i\in \cN$. Therefore, we will obtain an $O(T_0^{p-1})$ approximate SPNE with $O(T)$-bounded deviation.
\end{theorem}

We defer the proof of the theorem to Appendix \ref{appendix:LPR-Regret-SPE}. This theorem establishes the connection between equilibrium in repeated games and {\tt LRP-Regret}. When all the players are running no-{\tt LRP-regret} learning and achieve sublinear {\tt LRP-Regret}, we can obtain an approximate SPNE. Compared to \Cref{lemma:RP-Regret-SPE}, \Cref{lemma:LRP-Regret-SPE} only holds when sublinear regret is obtained with respect to the comparator without variation budget ($P_T=\Theta(T)$).

We note that \Cref{lemma:LRP-Regret-SPE} only states the relationship between {approximate} $O(T)$-robust  SPNE and {\tt LRP-Regret} minimization, but does not provide an algorithm to achieve so. Indeed, directly running our regret minimization algorithm for {\tt LRP-Regret} (\S \ref{sec:local-RP-Regret-Minimizer}) {may not} achieve the equilibrium, since it requires the variation of the comparator $P_{T_0}$ to be sublinear in $T_0$. We defer the development of such no-{\tt LRP-Regret} algorithms as an immediate future work.

\begin{remark}
There are also counterparts of \Cref{lemma:LRP-Regret-SPE} and \Cref{lemma:RP-Regret-SPE} for finitely repeated games. A strategy sequence $\bpi_{1:T}$ (possibly correlated) that achieves sublinear  {\tt RP-Regret} ({\tt LRP-Regret}) is an approximate (coarse correlated) equilibrium (not subgame perfect anymore) of the $T_0$  repeated game, \emph{i.e.}, for any player $i\in\cN$, we have $\frac{1}{T_0}\sum_{t=1}^{T_0} \Big(f_i^{t-1}(\bpi_{1:t}\given\emptyset) -f_i^{t-1}((\hat\bpi\ui_{1:t},\bpi\uni_{1:t})\given\emptyset)\Big)\leq \frac{R\ui}{T_0}$, where $R\ui$ is the upper bound of player $i$'s regret. %
\end{remark}

\subsection{Computing the Equilibria}
\label{sec:CCE-computation}

We now propose an algorithm to find an approximate SPCCE with $O(T)$-bounded deviation  of infinitely repeated matrix games. The algorithm is proposed in \Cref{alg:self-play-finite-MDP}. {Guarantee of the algorithm is provided below.}

\begin{theorem}
\label{theorem:SPCCE-computation}
    For an infinitely repeated game, when all {the} players and their comparators satisfy Condition \ref{assumption:2-forgetful-prime} with $\bnu\ui$ {being} the uniform strategy over $\Delta_{|\cA_i|}$, within $T_0\in\NN_{>0}$ iterations, the output of \Cref{alg:self-play-finite-MDP} converges to an $O\rbr{\frac{1}{T_0^{2/7}}}$-approximate SPCCE with $O(T)$-bounded deviation of the infinitely repeated game.
\end{theorem}

The proof is deferred to \Cref{appendix:computation-SPCCE}. The proof relies on modeling the repeated game as a Markov game as in \S\ref{sec:reformulating-mg}, and then resorting to the existing result about solving CCE in Markov games  \citep{jin2021v,songcan,mao2022provably-CCE-MDP,daskalakis2022complexity}.

\section{Conclusion}

In this paper, we studied regret-minimization in repeated games with adaptive opponents who can respond based on the histories of play. To this end, we advocated a new metric, {\tt RP-Regret}, which is native to this setting, and identified a series of necessary
conditions for obtaining {\tt RP-Regret} sublinear in time. We then developed additional conditions and provable algorithms to minimize {\tt RP-Regret}, followed by the connection of {\tt RP-Regret} minimization to certain known sub-game perfect equilibria computation. Our work opens new directions for future research, including the development of weaker equilibrium notions induced by {\tt RP-Regret} minimization in \S\ref{sec:4-rm-adaptive}, for which  weaker assumptions may be required to compute. 
It also raises the possibility of obtaining provable equilibrium-selection results via {\tt RP-Regret} minimization in games with particular payoff structures.

\section*{Acknowledgement} 
M.L. was supported by the MathWorks Fellowship. A.O. was supported in part by the ONR grant
N000142512296. K.Z. acknowledges the support from the 
Army Research Office (ARO) grant W911NF-24-1-0085, the NSF CAREER Award 2443704, the AFOSR YIP Award FA9550-25-1-0258, a Cisco Faculty Research Award, and a JP Morgan Faculty Research Award.

\bibliography{main}

\begin{thebibliography}{67}
\providecommand{\natexlab}[1]{#1}
\providecommand{\url}[1]{\texttt{#1}}
\expandafter\ifx\csname urlstyle\endcsname\relax
  \providecommand{\doi}[1]{doi: #1}\else
  \providecommand{\doi}{doi: \begingroup \urlstyle{rm}\Url}\fi

\bibitem[Agarwal et~al.(2019)Agarwal, Bullins, Hazan, Kakade, and Singh]{DBLP:conf/icml/AgarwalBHKS19-control-OCO-memory}
Naman Agarwal, Brian Bullins, Elad Hazan, Sham~M. Kakade, and Karan Singh.
\newblock Online control with adversarial disturbances.
\newblock In \emph{International Conference on Machine Learning (ICML)}, 2019.

\bibitem[Anava et~al.(2015)Anava, Hazan, and Mannor]{anava2015online-OCO-framework}
Oren Anava, Elad Hazan, and Shie Mannor.
\newblock Online learning for adversaries with memory: price of past mistakes.
\newblock In \emph{Neural Information Processing Systems (NeurIPS)}, 2015.

\bibitem[Arora et~al.(2012)Arora, Dekel, and Tewari]{arora2012online}
Raman Arora, Ofer Dekel, and Ambuj Tewari.
\newblock Online bandit learning against an adaptive adversary: from regret to policy regret.
\newblock In \emph{International Conference on Machine Learning (ICML)}, pages 1747--1754, 2012.

\bibitem[Arora et~al.(2018)Arora, Dinitz, Marinov, and Mohri]{arora2018policy}
Raman Arora, Michael Dinitz, Teodor~Vanislavov Marinov, and Mehryar Mohri.
\newblock Policy regret in repeated games.
\newblock In \emph{Neural Information Processing Systems (NeurIPS)}, 2018.

\bibitem[Aumann and Shapley(1976)]{aumann1976long-folk-theorem1}
Robert Aumann and Lloyd Shapley.
\newblock Long term competition: A game theoretic analysis', mimeograph.
\newblock \emph{Hebrew University}, 1976.

\bibitem[Axelrod and Hamilton(1981)]{axelrod1981evolution-prisoner}
Robert Axelrod and William~D Hamilton.
\newblock The evolution of cooperation.
\newblock \emph{science}, 211\penalty0 (4489):\penalty0 1390--1396, 1981.

\bibitem[Bailey and Piliouras(2018)]{bailey2018multiplicative-unique}
James~P Bailey and Georgios Piliouras.
\newblock Multiplicative weights update in zero-sum games.
\newblock In \emph{ACM Conference on Economics and Computation (EC)}, 2018.

\bibitem[Benoit and Krishna(1985)]{benoit1984finitely}
Jean-Pierre Benoit and Vijay Krishna.
\newblock Finitely repeated games.
\newblock \emph{Econometrica}, 53\penalty0 (4):\penalty0 905--22, 1985.
\newblock URL \url{https://EconPapers.repec.org/RePEc:ecm:emetrp:v:53:y:1985:i:4:p:905-22}.

\bibitem[Blocki et~al.(2013)Blocki, Christin, Datta, and Sinha]{blocki2013adaptive}
Jeremiah Blocki, Nicolas Christin, Anupam Datta, and Arunesh Sinha.
\newblock Adaptive regret minimization in bounded-memory games.
\newblock In \emph{Decision and Game Theory for Security (GameSec)}, 2013.

\bibitem[Borgs et~al.(2008)Borgs, Chayes, Immorlica, Kalai, Mirrokni, and Papadimitriou]{borgs2008myth-hardness-NE}
Christian Borgs, Jennifer Chayes, Nicole Immorlica, Adam~Tauman Kalai, Vahab Mirrokni, and Christos Papadimitriou.
\newblock The myth of the folk theorem.
\newblock In \emph{ACM Symposium on Theory of Computing (STOC)}, 2008.

\bibitem[Brown and Sandholm(2018)]{brown2018superhuman-libratus}
Noam Brown and Tuomas Sandholm.
\newblock Superhuman ai for heads-up no-limit poker: Libratus beats top professionals.
\newblock \emph{Science}, 359\penalty0 (6374):\penalty0 418--424, 2018.

\bibitem[Brown and Sandholm(2019)]{brown2019superhuman-pluribus}
Noam Brown and Tuomas Sandholm.
\newblock Superhuman ai for multiplayer poker.
\newblock \emph{Science}, 365\penalty0 (6456):\penalty0 885--890, 2019.

\bibitem[Cao(1999)]{cao1999single-performance-difference-lemma-average-reward}
Xi-Ren Cao.
\newblock Single sample path-based optimization of markov chains.
\newblock \emph{Journal of optimization theory and applications}, 100:\penalty0 527--548, 1999.

\bibitem[Cao and Liu(2018)]{cao2018online-constrained-time-varying}
Xuanyu Cao and KJ~Ray Liu.
\newblock Online convex optimization with time-varying constraints and bandit feedback.
\newblock \emph{IEEE Transactions on Automatic Control}, 64\penalty0 (7):\penalty0 2665--2680, 2018.

\bibitem[Cen et~al.(2021)Cen, Wei, and Chi]{cen2021fast}
Shicong Cen, Yuting Wei, and Yuejie Chi.
\newblock Fast policy extragradient methods for competitive games with entropy regularization.
\newblock In \emph{Neural Information Processing Systems (NeurIPS)}, 2021.

\bibitem[Cesa-Bianchi and Lugosi(2006)]{cesa2006prediction}
Nicolo Cesa-Bianchi and G{\'a}bor Lugosi.
\newblock \emph{Prediction, Learning, and Games}.
\newblock Cambridge University Press, 2006.

\bibitem[Chakraborty and Stone(2014)]{chakraborty2014multiagent-bounded-memory}
Doran Chakraborty and Peter Stone.
\newblock Multiagent learning in the presence of memory-bounded agents.
\newblock \emph{Autonomous agents and multi-agent systems}, 28\penalty0 (2):\penalty0 182--213, 2014.

\bibitem[Chen et~al.(2017)Chen, Ling, and Giannakis]{chen2017online}
Tianyi Chen, Qing Ling, and Georgios~B Giannakis.
\newblock An online convex optimization approach to proactive network resource allocation.
\newblock \emph{IEEE Transactions on Signal Processing}, 65\penalty0 (24):\penalty0 6350--6364, 2017.

\bibitem[Chen et~al.(2006)Chen, Deng, and Teng]{chen2006computing-complexity-NE}
Xi~Chen, Xiaotie Deng, and Shang-Hua Teng.
\newblock Computing nash equilibria: Approximation and smoothed complexity.
\newblock In \emph{Symposium on Foundations of Computer Science (FOCS)}, 2006.

\bibitem[Chen et~al.(2007)Chen, Teng, and Valiant]{chen2007approximation-ppad}
Xi~Chen, Shang-Hua Teng, and Paul Valiant.
\newblock The approximation complexity of win-lose games.
\newblock 2007.

\bibitem[Daniely et~al.(2015)Daniely, Gonen, and Shalev-Shwartz]{daniely2015strongly-adaptive-regret}
Amit Daniely, Alon Gonen, and Shai Shalev-Shwartz.
\newblock Strongly adaptive online learning.
\newblock In \emph{International Conference on Machine Learning (ICML)}, 2015.

\bibitem[Daskalakis et~al.(2009)Daskalakis, Goldberg, and Papadimitriou]{daskalakis2009complexity-NE}
Constantinos Daskalakis, Paul~W Goldberg, and Christos~H Papadimitriou.
\newblock The complexity of computing a nash equilibrium.
\newblock \emph{Communications of the ACM}, 52\penalty0 (2):\penalty0 89--97, 2009.

\bibitem[Daskalakis et~al.(2023)Daskalakis, Golowich, and Zhang]{daskalakis2022complexity}
Constantinos Daskalakis, Noah Golowich, and Kaiqing Zhang.
\newblock The complexity of {M}arkov equilibrium in stochastic games.
\newblock In \emph{Conference on Learning Theory (COLT)}, 2023.

\bibitem[de~Farias and Megiddo(2003)]{de2003combine-flexible}
Daniela de~Farias and Nimrod Megiddo.
\newblock How to combine expert (and novice) advice when actions impact the environment?
\newblock In \emph{Neural Information Processing Systems (NeurIPS)}, volume~16, 2003.

\bibitem[Filar and Vrieze(2012)]{filar2012competitive}
Jerzy Filar and Koos Vrieze.
\newblock \emph{Competitive {M}arkov decision processes}.
\newblock Springer Science \& Business Media, 2012.

\bibitem[Foerster et~al.(2017)Foerster, Chen, Al-Shedivat, Whiteson, Abbeel, and Mordatch]{foerster2017learning-lola}
Jakob~N Foerster, Richard~Y Chen, Maruan Al-Shedivat, Shimon Whiteson, Pieter Abbeel, and Igor Mordatch.
\newblock Learning with opponent-learning awareness.
\newblock \emph{arXiv preprint arXiv:1709.04326}, 2017.

\bibitem[Fudenberg and Maskin(2009)]{fudenberg2009folk-theorem2}
Drew Fudenberg and Eric Maskin.
\newblock The folk theorem in repeated games with discounting or with incomplete information.
\newblock In \emph{A long-run collaboration on long-run games}, pages 209--230. World Scientific, 2009.

\bibitem[Gillette(1957)]{gillette1957stochastic}
Dean Gillette.
\newblock Stochastic games with zero stop probabilities.
\newblock \emph{Contributions to the Theory of Games}, 3\penalty0 (39):\penalty0 179--187, 1957.

\bibitem[Hannan(1957)]{hannan1957approximation}
James Hannan.
\newblock Approximation to bayes risk in repeated play.
\newblock \emph{Contributions to the Theory of Games}, 3:\penalty0 97--139, 1957.

\bibitem[Hart and Mas-Colell(2000)]{hart2000simple}
Sergiu Hart and Andreu Mas-Colell.
\newblock A simple adaptive procedure leading to correlated equilibrium.
\newblock \emph{Econometrica}, 68\penalty0 (5):\penalty0 1127--1150, 2000.

\bibitem[Hazan et~al.(2007)Hazan, Agarwal, and Kale]{hazan2007logarithmic-adaptive-regret}
Elad Hazan, Amit Agarwal, and Satyen Kale.
\newblock Logarithmic regret algorithms for online convex optimization.
\newblock \emph{Machine Learning}, 69\penalty0 (2):\penalty0 169--192, 2007.

\bibitem[Hazan et~al.(2016)]{hazan2016introduction}
Elad Hazan et~al.
\newblock Introduction to online convex optimization.
\newblock \emph{Foundations and Trends{\textregistered} in Optimization}, 2\penalty0 (3-4):\penalty0 157--325, 2016.

\bibitem[Jin et~al.(2018)Jin, Allen-Zhu, Bubeck, and Jordan]{jin2018q-alpha-lemma}
Chi Jin, Zeyuan Allen-Zhu, Sebastien Bubeck, and Michael~I Jordan.
\newblock Is q-learning provably efficient?
\newblock In \emph{Neural Information Processing Systems (NeurIPS)}, 2018.

\bibitem[Jin et~al.(2021)Jin, Liu, Wang, and Yu]{jin2021v}
Chi Jin, Qinghua Liu, Yuanhao Wang, and Tiancheng Yu.
\newblock V-learning -- {A} simple, efficient, decentralized algorithm for multiagent {RL}.
\newblock \emph{arXiv preprint arXiv:2110.14555}, 2021.

\bibitem[Kalai and Stanford(1988)]{kalai1988finite}
Ehud Kalai and William Stanford.
\newblock Finite rationality and interpersonal complexity in repeated games.
\newblock \emph{Econometrica: Journal of the Econometric Society}, pages 397--410, 1988.

\bibitem[Kim et~al.(2021)Kim, Liu, Riemer, Sun, Abdulhai, Habibi, Lopez-Cot, Tesauro, and How]{kim2021policy-shaping}
Dong~Ki Kim, Miao Liu, Matthew~D Riemer, Chuangchuang Sun, Marwa Abdulhai, Golnaz Habibi, Sebastian Lopez-Cot, Gerald Tesauro, and Jonathan How.
\newblock A policy gradient algorithm for learning to learn in multiagent reinforcement learning.
\newblock In \emph{International Conference on Machine Learning (ICML)}, 2021.

\bibitem[Letcher et~al.(2019)Letcher, Foerster, Balduzzi, Rockt{\"{a}}schel, and Whiteson]{DBLP:conf/iclr/LetcherFBRW19-SOS}
Alistair Letcher, Jakob~N. Foerster, David Balduzzi, Tim Rockt{\"{a}}schel, and Shimon Whiteson.
\newblock Stable opponent shaping in differentiable games.
\newblock In \emph{International Conference on Learning Representations (ICLR)}, 2019.

\bibitem[Littman and Stone(2003)]{littman2003polynomial-two-player-folk}
Michael~L Littman and Peter Stone.
\newblock A polynomial-time nash equilibrium algorithm for repeated games.
\newblock In \emph{ACM Conference on Electronic Commerce}, 2003.

\bibitem[Liu et~al.(2022)Liu, Ozdaglar, Yu, and Zhang]{liu2022power-reg}
Mingyang Liu, Asuman~E Ozdaglar, Tiancheng Yu, and Kaiqing Zhang.
\newblock The power of regularization in solving extensive-form games.
\newblock In \emph{International Conference on Learning Representations (ICLR)}, 2022.

\bibitem[Loftin and Oliehoek(2022)]{loftin2022impossibility-RP-Regret}
Robert Loftin and Frans~A Oliehoek.
\newblock On the impossibility of learning to cooperate with adaptive partner strategies in repeated games.
\newblock In \emph{International Conference on Machine Learning (ICML)}, pages 14197--14209. PMLR, 2022.

\bibitem[Lu et~al.(2022)Lu, Willi, De~Witt, and Foerster]{lu2022model-meta-learning}
Christopher Lu, Timon Willi, Christian A~Schroeder De~Witt, and Jakob Foerster.
\newblock Model-free opponent shaping.
\newblock In \emph{International Conference on Machine Learning (ICML)}, pages 14398--14411. PMLR, 2022.

\bibitem[Mao and Ba{\c{s}}ar(2022)]{mao2022provably-CCE-MDP}
Weichao Mao and Tamer Ba{\c{s}}ar.
\newblock Provably efficient reinforcement learning in decentralized general-sum markov games.
\newblock \emph{Dynamic Games and Applications}, pages 1--22, 2022.

\bibitem[Mao et~al.(2021)Mao, Zhang, Zhu, Simchi-Levi, and Basar]{pmlr-v139-mao21b-non-stationary-restart}
Weichao Mao, Kaiqing Zhang, Ruihao Zhu, David Simchi-Levi, and Tamer Basar.
\newblock Near-optimal model-free reinforcement learning in non-stationary episodic {MDPs}.
\newblock In \emph{International Conference on Machine Learning (ICML)}, 2021.

\bibitem[Merhav et~al.(2002)Merhav, Ordentlich, Seroussi, and Weinberger]{merhav2002sequential-memory-regret-init}
Neri Merhav, Erik Ordentlich, Gadiel Seroussi, and Marcelo~J Weinberger.
\newblock On sequential strategies for loss functions with memory.
\newblock \emph{IEEE Transactions on Information Theory}, 48\penalty0 (7):\penalty0 1947--1958, 2002.

\bibitem[Neyman(1985)]{neyman1985bounded}
Abraham Neyman.
\newblock Bounded complexity justifies cooperation in the finitely repeated prisoners' dilemma.
\newblock \emph{Economics Letters}, 19\penalty0 (3):\penalty0 227--229, 1985.

\bibitem[Norris(1998)]{norris1998markov}
James~R Norris.
\newblock \emph{Markov chains}.
\newblock Number~2. Cambridge university press, 1998.

\bibitem[Paternain and Ribeiro(2016)]{paternain2016online}
Santiago Paternain and Alejandro Ribeiro.
\newblock Online learning of feasible strategies in unknown environments.
\newblock \emph{IEEE Transactions on Automatic Control}, 62\penalty0 (6):\penalty0 2807--2822, 2016.

\bibitem[Piccione and Rubinstein(1997)]{piccione1997interpretation-imperfect-recall}
Michele Piccione and Ariel Rubinstein.
\newblock On the interpretation of decision problems with imperfect recall.
\newblock \emph{Games and Economic Behavior}, 20\penalty0 (1):\penalty0 3--24, 1997.

\bibitem[Puterman(2014)]{puterman2014markov}
Martin~L Puterman.
\newblock \emph{Markov decision processes: {D}iscrete stochastic dynamic programming}.
\newblock John Wiley \& Sons, 2014.

\bibitem[Radner(1986)]{radner1986can}
Roy Radner.
\newblock Can bounded rationality resolve the {Prisoner's} dilemma.
\newblock \emph{Essays in honor of Gerard Debreu}, pages 387--399, 1986.

\bibitem[Roughgarden(2010)]{roughgarden2010algorithmic-game-theory}
Tim Roughgarden.
\newblock Algorithmic game theory.
\newblock \emph{Communications of the ACM}, 53\penalty0 (7):\penalty0 78--86, 2010.

\bibitem[Schlag and Zapechelnyuk(2012)]{schlag2012impossibility}
Karl Schlag and Andriy Zapechelnyuk.
\newblock On the impossibility of achieving no regrets in repeated games.
\newblock \emph{Journal of Economic Behavior \& Organization}, 81\penalty0 (1):\penalty0 153--158, 2012.

\bibitem[Shapley(1953)]{shapley1953stochastic}
Lloyd~S Shapley.
\newblock Stochastic games.
\newblock \emph{Proceedings of the National Academy of Sciences}, 39\penalty0 (10):\penalty0 1095--1100, 1953.

\bibitem[Sokota et~al.(2023)Sokota, D'Orazio, Kolter, Loizou, Lanctot, Mitliagkas, Brown, and Kroer]{DBLP:conf/iclr/SokotaDKLLMBK23-MMD}
Samuel Sokota, Ryan D'Orazio, J.~Zico Kolter, Nicolas Loizou, Marc Lanctot, Ioannis Mitliagkas, Noam Brown, and Christian Kroer.
\newblock A unified approach to reinforcement learning, quantal response equilibria, and two-player zero-sum games.
\newblock In \emph{International Conference on Learning Representations (ICLR)}, 2023.

\bibitem[Song et~al.(2022)Song, Mei, and Bai]{songcan}
Ziang Song, Song Mei, and Yu~Bai.
\newblock When can we learn general-sum markov games with a large number of players sample-efficiently?
\newblock In \emph{International Conference on Learning Representations (ICLR)}, 2022.

\bibitem[Suggala and Netrapalli(2020)]{suggala2020online-ONCO-oracle}
Arun~Sai Suggala and Praneeth Netrapalli.
\newblock Online non-convex learning: Following the perturbed leader is optimal.
\newblock In \emph{International Conference on Algorithmic Learning Theory (ALT)}, 2020.

\bibitem[Sylvester(1882)]{sylvester1882subvariants-frobenius-number}
James~J Sylvester.
\newblock On subvariants, ie semi-invariants to binary quantics of an unlimited order.
\newblock \emph{American Journal of Mathematics}, 5\penalty0 (1):\penalty0 79--136, 1882.

\bibitem[Van~Damme(1991)]{van1991stability-unique}
Eric Van~Damme.
\newblock \emph{Stability and perfection of Nash equilibria}, volume 339.
\newblock Springer, 1991.

\bibitem[Watson(2002)]{watson2002strategy-one-step-deviate}
Joel Watson.
\newblock \emph{Strategy: an introduction to game theory}, volume 139.
\newblock WW Norton New York, 2002.

\bibitem[Waugh et~al.(2009)Waugh, Zinkevich, Johanson, Kan, Schnizlein, and Bowling]{waugh2009practical-imperfect-recall}
Kevin Waugh, Martin Zinkevich, Michael Johanson, Morgan Kan, David Schnizlein, and Michael~H Bowling.
\newblock A practical use of imperfect recall.
\newblock In \emph{SARA}, 2009.

\bibitem[Willi et~al.(2022)Willi, Letcher, Treutlein, and Foerster]{willi2022cola}
Timon Willi, Alistair~Hp Letcher, Johannes Treutlein, and Jakob Foerster.
\newblock Cola: consistent learning with opponent-learning awareness.
\newblock In \emph{International Conference on Machine Learning (ICML)}, pages 23804--23831. PMLR, 2022.

\bibitem[Zhang et~al.(2018{\natexlab{a}})Zhang, Lu, and Zhou]{zhang2018adaptive}
Lijun Zhang, Shiyin Lu, and Zhi-Hua Zhou.
\newblock Adaptive online learning in dynamic environments.
\newblock In \emph{Neural Information Processing Systems (NeurIPS)}, 2018{\natexlab{a}}.

\bibitem[Zhang et~al.(2018{\natexlab{b}})Zhang, Yang, Zhou, et~al.]{zhang2018dynamic}
Lijun Zhang, Tianbao Yang, Zhi-Hua Zhou, et~al.
\newblock Dynamic regret of strongly adaptive methods.
\newblock In \emph{International Conference on Machine Learning (ICML)}, pages 5882--5891. PMLR, 2018{\natexlab{b}}.

\bibitem[Zhang et~al.(2018{\natexlab{c}})Zhang, Yang, Zhou, et~al.]{zhang2018dynamic-adaptive-regret}
Lijun Zhang, Tianbao Yang, Zhi-Hua Zhou, et~al.
\newblock Dynamic regret of strongly adaptive methods.
\newblock In \emph{International Conference on Machine Learning (ICML)}, 2018{\natexlab{c}}.

\bibitem[Zhao et~al.(2022)Zhao, Wang, and Zhou]{zhao2022non-dynamic-policy-regret}
Peng Zhao, Yu-Xiang Wang, and Zhi-Hua Zhou.
\newblock Non-stationary online learning with memory and non-stochastic control.
\newblock In \emph{International Conference on Artificial Intelligence and Statistics (AISTATS)}, 2022.

\bibitem[Zinkevich(2003)]{zinkevich2003online-OGD-original-paper}
Martin Zinkevich.
\newblock Online convex programming and generalized infinitesimal gradient ascent.
\newblock In \emph{International Conference on Machine Learning (ICML)}, 2003.

\bibitem[Zinkevich(2005)]{zinkevich2005response-regret}
Martin Zinkevich.
\newblock Response regret.
\newblock In \emph{AAAI Fall Symposium: Coevolutionary and Coadaptive Systems}, page~41, 2005.

\end{thebibliography}
{\bibliographystyle{plainnat}}

\onecolumn

~\\
\centerline{{\fontsize{13.5}{13.5}\selectfont \textbf{Supplementary Materials for}}}

\vspace{6pt}
\centerline{\fontsize{13.5}{13.5}\selectfont \textbf{
	  ``Regret Minimization with Adaptive Opponents in Repeated Games''}}

\tableofcontents
\clearpage

\appendix

\section{Motivating Example Details}\label{appendix:motivating_example}

\begin{figure}[h]

\begin{minipage}[t][][b]{0.45\textwidth}

 \centering
\begin{tabular}{|c|c|c|}
\hline
 Prisoner's Dilemma & Cooperate  ($C$)      & Defect  ($D$)      \\ \hline
Cooperate ($C$) & 0.6, 0.6 & 0.0, 1.0 \\ \hline
Defect ($D$) & 1.0, 0.0 & 0.2, 0.2 \\ \hline
\end{tabular}
\\[2em]

    \begin{tabular}{|c|c|c|}
\hline
 Stag Hunt & Stag        & Hare        \\ \hline
Stag & 1.0, 1.0 & 0.1, 0.8 \\ \hline
Hare & 0.8, 0.1 & 0.5, 0.5 \\ \hline
\end{tabular}
\end{minipage}
\hfill
\begin{minipage}[t][][b]{0.45\textwidth}

\centering
\includegraphics[width=0.95\textwidth]{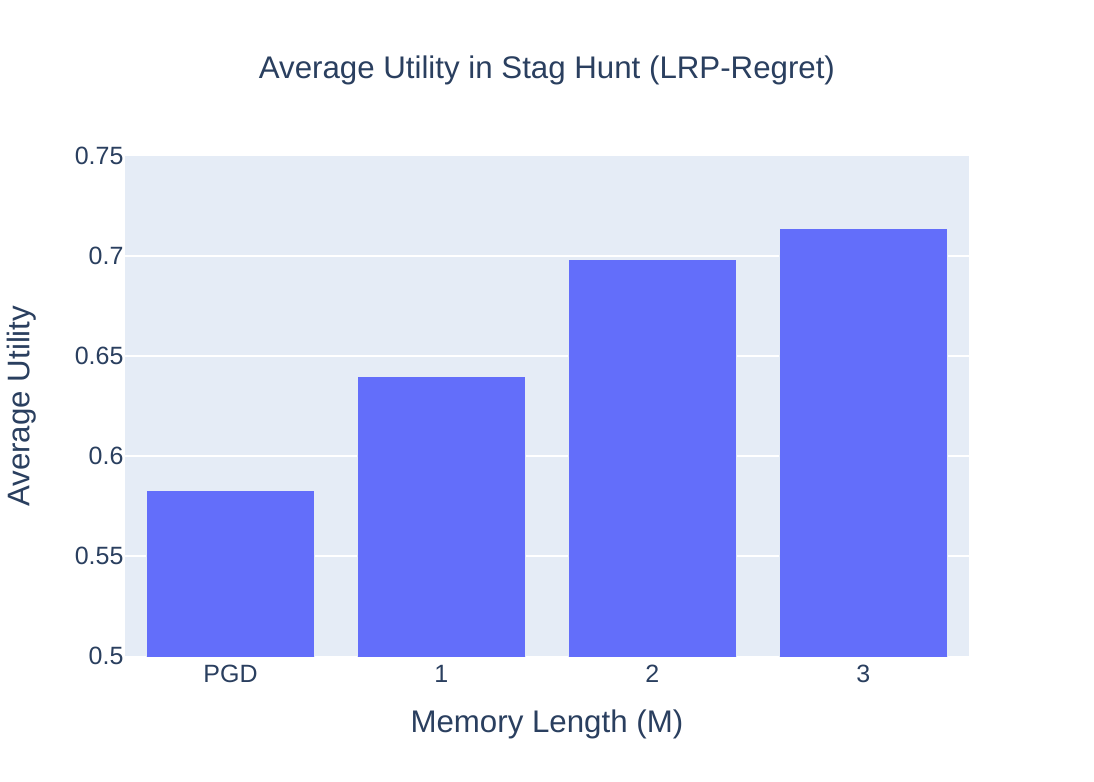}
\end{minipage}

\captionof{figure}{Examples of the utility matrices for Prisoner's  Dilemma (upper left) and Stag Hunt (lower left). In each cell, the number on the left (right) is the utility for the row (column) player.
The figure on the right plots the average utility across $10,000$ runs of the experiments when both players in Stag-Hunt (lower left) are minimizing {\tt LRP-Regret} for $100,000$ iterations. In each run of the experiment, the initial strategy is randomly sampled uniformly. We choose the learning rate $\eta=0.01$ in \Cref{eq:PGD-def} with different memory lengths $M=1,2,3$. %
As shown in the figure, {\tt LRP-Regret} minimization can encourage players to converge to a \emph{better equilibrium} in terms of utility.}
\label{fig:stag-hunt-utility}

\end{figure}

\begin{example}[Finding a Better Equilibrium]
    \label{example:experiment}
     The Stag-Hunt game  (cf. utility matrix in \Cref{fig:stag-hunt-utility}) has two NEs, Stag-Stag and Hare-Hare. In Figure \ref{fig:stag-hunt-utility}, we show that by minimizing {\tt LRP-Regret} (a convexified version of our new regret metric, cf. \S\ref{sec:local-RP-Regret-Minimizer}), the algorithm converges to a better equilibrium more often (Stag-Stag,  with utility $1.0$  for both players). This is because the equilibria induced by minimizing our new notions of regret can potentially constitute a \emph{larger set} than the set of equilibria in the one-shot game. %
\end{example}

\section{Detailed Related Work}
\label{sec:related-work}

\paragraph{{Equilibrium Computation in} Repeated Games.} %
Learning and equilibrium computation  in  repeated normal-form  games is a well-studied  area. One well-known folklore result is that when all the players run no-(external-)regret learning algorithms, the average iterates will converge to the coarse correlated equilibrium of the one-shot matrix game \citep{cesa2006prediction}.
In playing repeated games,
since the players could change their strategies based on past observations, {the equilibrium for the repeated game can be different from that of the one-shot game. In this regard, using the well-known   folk theorem\footnote{Folk theorem gives a way to construct ``trigger" strategies. Specifically, if one of the players deviates, say player 1, then other players will also change their strategy to the \emph{minimax} one,  to guarantee player 1 will suffer at least $\min_{\bpi\uno}\max_{\bpi\uo} \cL_1(\bpi\uo,\bpi\uno)$ loss in every later timestep. Therefore, no player can benefit from deviation if all of them can obtain a lower loss than their minimax loss.} \citep{aumann1976long-folk-theorem1, fudenberg2009folk-theorem2}, one can}
find the NE in infinitely repeated matrix games.
\cite{littman2003polynomial-two-player-folk} gave a polynomial-time algorithm based on the folk theorem to compute the NE of an infinitely repeated two-player general-sum game, which is known to be {\sf PPAD}-hard for the one-shot case \citep{daskalakis2009complexity-NE,chen2006computing-complexity-NE}. However, \cite{borgs2008myth-hardness-NE} showed that finding the minimax value   is {\sf NP}-hard in general-sum games with 3 or more players, which prevented the technique in \cite{littman2003polynomial-two-player-folk} from being extended to the multi-player cases.

\paragraph{Impossibility Results.} \cite{borgs2008myth-hardness-NE} proved that computing an NE for an $N$-player infinitely repeated game with discount factors is as hard as computing an NE of an $(N-1)$-player one-shot game. By \cite{chen2007approximation-ppad,daskalakis2009complexity-NE,chen2006computing-complexity-NE}, computing an NE of a 2-player general-sum one-shot game is {\sf PPAD}-hard and thus computing an NE of a 3-player infinitely repeated game with discount factor is also {\sf PPAD}-hard. The result does not conflict with this paper since in \S \ref{sec:CCE-computation}, we are computing a coarse correlated equilibrium (CCE)  instead of an NE of the repeated game. Also, our result implies that directly minimizing {\tt RP-Regret} without further assumptions may not be computationally tractable. Otherwise, if each player achieves a sublinear {\tt RP-Regret}, by \Cref{lemma:RP-Regret-SPE}, the trajectory converges to a subgame perfect NE of an infinitely repeated game without discount factor. %
\cite{schlag2012impossibility, loftin2022impossibility-RP-Regret} gave  impossibility results on minimizing regret  in repeated games when the regret minimizer cannot observe the full strategy of the opponent. Therefore, to the best of our knowledge, it is still unclear under which conditions we can minimize regret, when the other players deviate themselves as a reaction to the deviation of one player{, the setting our work focuses on}.

\paragraph{Rationalizing  Cooperation.} Our study is also motivated by those of rationalizing  \emph{cooperation} in Iterated Prisoner's Dilemma with bounded rationality and strategy complexity  \citep{axelrod1981evolution-prisoner,neyman1985bounded,radner1986can,kalai1988finite}. These papers investigated the Nash equilibrium of the (finitely) repeated game when the strategies are modeled as  an \emph{automaton} with finite states, and showed the existence of equilibrium with payoff close to the cooperative outcome.  In this paper, we show, in contrast, that when players have \emph{unlimited} memory power (automata with \emph{infinite} states), our regret notion of {\tt RP-Regret} cannot be minimized.%

\paragraph{Opponent Shaping and  Modeling.} {Our work also takes significant inspiration from the recent empirical literature of \emph{opponent shaping} in multi-agent (reinforcement) learning \citep{foerster2017learning-lola, kim2021policy-shaping, willi2022cola, lu2022model-meta-learning}, which developed algorithms to shape the opponents' future strategies, knowing that they will \emph{adapt} to the actions of the learning agent/player.}
However,
most of the algorithms do not enjoy theoretical guarantees, \emph{e.g.,}  on the convergence, and/or on the quality of the strategies they converge to.
\cite{DBLP:conf/iclr/LetcherFBRW19-SOS}
proved that their algorithm called Stable Opponent Shaping will converge. However, the guarantee only applies when assuming the opponents are naive learners\footnote{At each  timestep, the learner takes one step of gradient descent according to the loss at the last timestep.}, and there is no guarantee on the performance against general unknown opponents. %
By contrast, our method{, through the lens of \emph{regret-minimization}, is more general and robust against}
any possible players/opponents satisfying some provably necessary conditions (see \S \ref{sec:nece_cond}).

\paragraph{Online Learning.}
Our work is closely related to the problem of \emph{online convex optimization with memory}, which was first studied in \cite{merhav2002sequential-memory-regret-init} and later generalized to developing the concept of \emph{Policy Regret} and studied in \cite{arora2012online,anava2015online-OCO-framework}.  Recently, \cite{zhao2022non-dynamic-policy-regret} showed how to achieve sublinear policy regret when the comparator is allowed to choose a time-varying strategy with a sublinear variation budget. One of the applications of OCO with memory is \cite{DBLP:conf/icml/AgarwalBHKS19-control-OCO-memory}, which discussed how to fit the online control problem into the framework of OCO with memory. However, existing analyses along these lines typically either assume that the loss function is convex or impose an $m$-bounded-memory condition on the loss, namely that the loss at timestep $t$ depends only on the recent action window $\sbr{\ba_{t-m},\dots,\ba_t}$. In contrast, in repeated games the expected per-round loss induced by strategic interaction need not be convex in the player's decision variables, and it can depend on the entire history.
\cite{blocki2013adaptive} considered regret minimization in bounded-memory games. However, the $k$-adaptive regret considered in \cite{blocki2013adaptive} will restart the game every $k+1$ rounds (\emph{i.e.,} the adversary will forget the history every $k+1$ rounds), which is a different setting from classical repeated games, where the adversary plays repeated games and
will remember the history from the beginning to the end.

Another line of related work in online learning is \emph{dynamic}  regret minimization \citep{zinkevich2003online-OGD-original-paper}. In this setup, the accumulated loss is compared with a comparator that can change over time, but usually with a sublinear variation budget. Therefore, dynamic regret is more suitable when the environment is \emph{non-stationary}. To encourage the player's strategy to adapt to the changes in the opponents' strategies, we also consider dynamic comparators, but in game-theoretic settings, and with different regret notions.

\section{Proof for Example \ref{example:IPD}}
\label{appendix:proof-example-IPD}

In Iterated Prisoner's Dilemma, since the players are symmetric, we will focus on analyzing player 1 without loss of generality.
{Note that tit-for-tat (where both players start with $C$) promotes cooperation, since the players will stick to $C$ and the time-average utility will be $0.6$, higher than that of the one-shot Nash equilibrium $(D,D)$.}

However, tit-for-tat may not be a good strategy in terms of \emph{external regret}: mutually playing it will cause a \emph{linear} external regret, because the regret of player 1 when compared to always playing $D$ is $(1.0-0.6)T$. In fact, for any player that achieves sublinear external regret in IPD for any $T$, the time-average strategy must converge to $D$ as $T\to\infty$, since otherwise the player will suffer a linear regret compared to the fixed action $D$, the strictly dominant strategy.

On the other hand, for the {\tt RP-Regret} defined in \Cref{eq:RP-Regret-def}, consider the first timestep $t$ that player 1 is going to deviate from deterministically choosing $C$ (since the beginning of playing tit-for-tat mutually)
to some comparator strategy. In this case, since the expected utility is a multi-linear function of the comparator strategy, and the comparator is in the Cartesian product of each timestep's strategy space, there exists an optimal deviation strategy {of {\tt RP-Regret}} in the comparator space that is \emph{deterministic} -- choosing either $C$ or $D$ at every timestep.
Furthermore, at this
timestep $t$ when the comparator is different from deterministically choosing $C$ for the first time,
we can restrict the comparator strategy to be an \emph{oblivious} (in addition to deterministic) one, since the history $\bh_{1:t-1}$ has always been   $(C,C)$ before, deterministically. {With a tit-for-tat opponent, their strategy is deterministic at timestep $t+1$ as well. By induction, with such a deterministic history, it does not lose optimality to restrict the comparator at $t+1$ to be oblivious as well.}
Next, we will show that at timestep $t$, the best deviation for the comparator is actually $C$. Then, by induction, the comparator should choose $C$ deterministically at all timesteps except possibly at the last timestep $t=T$.

Suppose the comparator at timestep $t$ is going to play action $C$ with some  probability $w$ and $D$ with probability $1-w$, then they will receive an expected utility of $w\cdot u_1({C,C})+(1-w)\cdot u_1({D,C})=0.6w+(1-w)=1-0.4w$, where $u_1(a_1,a_2)$ denotes the utility of player 1 when the actions of both players are $a_1,a_2$, respectively. When $t<T$, for timestep $t+1$, suppose  the comparator will play action $C$ with probability $w'$ and $D$ with probability $1-w'$. Now, player 2, still following tit-for-tat,  changes from playing $C$ deterministically at timestep $t$, to playing $C$ with probability $w$,  due to the deviation of player 1 at timestep $t$. Therefore, the expected utility of player 1 at timestep $t+1$ now is $ww'\cdot u_1({C,C})+w(1-w')\cdot u_1({D,C})+(1-w) w'\cdot u_1({C,D})+(1-w)(1-w')\cdot u_1({D,D})=0.8w-0.2ww'-0.2w'+0.2$.

The summation of the expected utility of player 1 at timestep $t$ and $t+1$ is thus  $0.4w-0.2ww'-0.2w'+1.2=1.2-0.2w'+w(0.4-0.2w')$. Note that we analyze the deviation at timestep $t$ by only looking at the expected utility at these two timesteps, since $w$ will not affect the utility after $t+1$ (if exists), given player $2$ following tit-for-tat. Noting that $w'\leq 1$, we know $0.4-0.2w'>0$, and further have that $w$ should be $1.0$ in order to maximize the summed expected utility (regardless of the choice of $w'$). Hence, the comparator should not deviate from $C$ at this timestep $t$. Continuing the induction, we further know that
the comparator has no incentive to deviate from $C$ for  $t<T$, until
$t=T$. However, this will at most increase the utility by $1.0-0.6=0.4$, leading to a constant (and thus sublinear) {\tt RP-Regret} {in terms of $T$}.
\qed

The key ingredient that differentiates these two regret notions is that {\tt RP-Regret} considers the influence of a deviation at the current timestep on future timesteps, {by accounting for the \emph{adaptivity} of the opponents,} while external regret assumes that the opponents are non-adaptive.

\begin{remark}
    In fact, {the argument above is also aligned with the known result that tit-for-tat is an NE in an infinitely repeated Prisoner's Dilemma:} since we proved that deviating from tit-for-tat is beneficial only when $t=T$, such a final timestep is {nonexistent in the infinitely repeated setting. Hence, there is in effect no incentive to deviate.}
\end{remark}

\section{Full Statements of Hardness Results}
\label{appendix:hardness-result}

\subsection{Hardness Results in Table \ref{table:hardness-result}}

We now present the full statements and proofs of the hardness results in Table \ref{table:hardness-result}.

\begin{table}[h]
\centering
\begin{tabular}{|c|c|c|}
\hline
   Coin Tossing        & Up     & Down   \\ \hline
Guess Up   & 0 & 1 \\ \hline
Guess Down & 1 & 0 \\ \hline
\end{tabular}
\caption{The loss matrix for player 1 (the row player) of a coin-tossing game.} %
\label{table:coin-tossing}
\end{table}

 \begin{lemma}[Comparator Should Have Sublinear Variation]
\label{lemma:wo1w2}
Without Condition \ref{assumption:1-variation} on the comparator, {the player will have to suffer}
$\Omega(T)$ {\tt RP-Regret} in the worst case, even if Condition \ref{assumption:2-forgetful} holds for  both the comparator and the opponent.
\end{lemma}

\begin{lemma}[Comparator Should Not Have Perfect Recall]
\label{lemma:w1wo2}
When the  comparator has perfect recall {(\emph{i.e.,} without Condition \ref{assumption:imperfect-memory})},  {the player will have to suffer} $\Omega(T)$ {\tt RP-Regret} in the worst case, even if  Condition \ref{assumption:1-variation} holds for the comparator and Condition \ref{assumption:2-forgetful} holds for the opponent.
\end{lemma}

\begin{lemma}[Opponent Should Not Have Perfect Recall]
\label{lemma:w1wo2-bound}
When the opponent has perfect recall {(\emph{i.e.,} without Condition \ref{assumption:imperfect-memory})}, {the player will have to suffer}  $\Omega(T)$ {\tt RP-Regret} in the worst case, even if  Condition \ref{assumption:1-variation} and Condition \ref{assumption:2-forgetful} hold for the comparator.

\end{lemma}

{We note that since Condition \ref{assumption:2-forgetful} is stronger than Condition \ref{assumption:imperfect-memory}, the results in Lemmas \ref{lemma:wo1w2}-\ref{lemma:w1wo2-bound} above with ``even if  Condition \ref{assumption:2-forgetful} holds'' remain true when replacing it with
``even if Condition \ref{assumption:imperfect-memory}  holds''. This completes the statements in Table \ref{table:hardness-result}.}
We also note that
Lemma \ref{lemma:w1wo2-bound} holds even when the opponent  is further subject to a bounded memory length constraint as \cite{chakraborty2014multiagent-bounded-memory} (\emph{i.e.}, only make decisions based on the past $M$-length histories where $M$ is a constant), as our proof later will show.

\subsubsection{Proofs of Results in Table \ref{table:hardness-result}}

\begin{proof}[Proof of Lemma \ref{lemma:wo1w2}]
Consider a two-player coin-flipping game shown in Table \ref{table:coin-tossing}. The opponent (the column player) can flip the coin to any side they want, and player $1$ (the row player) needs to guess which side the coin is,  with loss $0$ for a correct guess and  loss $1$ for a wrong guess. For any strategy sequence of player $1$, at each timestep, there exists a side that player $1$ guesses with probability no larger than $0.5$. Let the column player deterministically flip to that side. This fixes an oblivious deterministic opponent sequence against which player $1$ incurs expected loss at least $0.5$ at every timestep. Without Condition \ref{assumption:1-variation} on the comparator, there exists a comparator that deterministically guesses this fixed sequence correctly every time, so that the corresponding accumulated loss is $0$, while player 1's accumulated loss {under any strategy sequence} is no less than $0.5 T$.
Hence, player 1 will get a linear {\tt RP-Regret}.

In this case, both the opponent and the comparator only use \emph{deterministic} strategies, and are
\emph{oblivious} in the sense that at timestep $t$, they
do not depend on the \emph{history} before $t$. In particular, the ratio in Condition \ref{assumption:2-forgetful} is always $1$ under any history. Hence, both the opponent and the comparator satisfy
Condition \ref{assumption:2-forgetful}, completing the proof.
\end{proof}

\begin{proof}[Proof of Lemma \ref{lemma:w1wo2}]
We consider the same coin-flipping game in Table \ref{table:coin-tossing}. The opponent still  adversarially flips the coin as in the proof of Lemma \ref{lemma:wo1w2}, so that player $1$ will have to incur a loss of at least  $0.5$ in expectation every timestep, under any strategy sequences. 
At the same time, since the comparator is not subject to Condition \ref{assumption:imperfect-memory} and has perfect recall, even a \emph{time-invariant}  comparator strategy can still behave differently at different timesteps, by noticing the different lengths of the history at the time. In particular, the comparator can choose the fixed strategy $\pi\uo(g(L(\bh))\given \bh)\equiv 1$ at all timesteps, where $g\colon \cbr{1,2,...,T}\to\{{\rm Guess~Up},{\rm Guess~Down}\}$. Back to the coin-flipping game, the comparator can arbitrarily and deterministically guess Up or Down at every timestep, by letting $g(t)$ be {the particular value to correctly guess the opponent's choice at that timestep $t$.}
 In this case,
 the {\tt RP-Regret} is $\Omega(T)$. {Note that the $\pi\uo$ above satisfies Condition \ref{assumption:1-variation} since given any  $\bh\in\cH$, $\pi\uo(\cdot\given \bh)$ remains unchanged over time. Moreover, the opponent's strategy satisfies Condition \ref{assumption:2-forgetful}, as shown in the proof of Lemma \ref{lemma:wo1w2}.}
\end{proof}

\begin{proof}[Proof of Lemma \ref{lemma:w1wo2-bound}]

\begin{table}[h]
\centering
\begin{tabular}{|c|c|c|c|c|}
\hline
\begin{tabular}[c]{@{}c@{}}Augmented\\ Prisoner's Dilemma\end{tabular} & $C$      & $D$     & $M_1$&$M_2$ \\ \hline
$C$                  & -3 & 0&0&0 \\ \hline
$D$                  & -4 & -1&0&0 \\ \hline
\end{tabular}
\caption{The loss  matrix for player 1 (the row player) of an augmented Prisoner's Dilemma game. $C$ stands for cooperate and $D$ stands for defect.
} %
\label{table:augmented-prisoners-dilemma}
\end{table}

Consider a variant of the Prisoner's Dilemma where the opponent (the column player) has two additional actions called $M_1$ and $M_2$, with the loss matrix shown in Table \ref{table:augmented-prisoners-dilemma}. If losses are required to lie in $[0,1]$, replace every entry $x$ in this matrix by $(x+4)/4$; this positive affine transformation scales every regret gap below by $1/4$ and therefore preserves the $\Omega(T)$ lower bound.
\begin{itemize}
    \item At timestep $1$, the column player (the opponent) chooses $M_1$ or $M_2$ deterministically.
    \item If $M_1$ was chosen in the first timestep, then at timestep $2$ the column player mimics the previous action of the row player.
    \item If $M_2$ was chosen in the first timestep, then at timestep $2$ the column player plays the action different from the previous action of the row player.
    \item From timestep $3$ onward, the column player mimics its own previous action deterministically.
\end{itemize}

{Note that the column player can implement the strategy above since the column player has perfect recall.}
In this case, the column player can adversarially choose $M_1$ or $M_2$ at the beginning to guarantee that they will always play $D$ at all later timesteps  with large probability,
while playing against {any strategies of} the row player. {Specifically, at timestep $1$, if the row player chooses $C$ with probability $\geq 0.5$, they will choose $M_2$; otherwise, they will choose $M_1$. This way, in the next timestep onwards, the column player will play $D$ with probability $\geq 0.5$.}

However, when the column player chooses $M_1$ at  timestep $1$, we let the comparator always play $C$ at all timesteps; otherwise, if the column player chooses $M_2$ at  timestep $1$, we let the comparator always play $D$ at all timesteps. In this case, {due to the adaptivity of the column player, they will always choose $C$ in later timesteps,} making the comparator suffer a loss of either $-3$ at each timestep with $(C,C)$, or $-4$ with $(D,C)$,
from $t=2$ to $T$. Meanwhile, the row player will suffer  either a  loss of $-1$ at each timestep with $(D,D)$, or $0$ at each timestep with $(C,D)$
from $t=2$ to $T$, {with probability at least $0.5$, since following the above argument, the column player will play $D$ with probability $\geq 0.5$.}

Therefore, a constant gap between the comparator's and the row player's expected losses is incurred from timestep $2$ to $T$: the comparator's largest loss is $-3$ (from always encountering $(C,C)$), whereas the row player's smallest possible expected loss is $-2.5$ (from encountering $(D,D)$ with probability $0.5$ and $(D,C)$ with probability $0.5$). This yields a linear {\tt RP-Regret}.
{In fact, it can be seen that the opponent in the above example only needs to have perfect recall within the latest $1$-step (or multi-step) bounded memory (the setting considered in \cite{chakraborty2014multiagent-bounded-memory}), and the regret remains linear. Note that, this does not contradict our positive results later, since such a finite-memory perfect recall condition does not satisfy our Condition \ref{assumption:2-forgetful}. This completes the proof.}
\end{proof}

\section{Approximation of {\tt RP-Regret}}\label{subsec:approx_RP_regret}

Unfortunately, even under the necessary conditions given in Table \ref{table:hardness-result}, minimizing {\tt RP-Regret} as defined in Eq.  \eqref{eq:RP-Regret-def} can still be computationally challenging, since each $f^{t-1}$ depends on every strategy from timestep $1$ to $t$,  and the number of optimization variables blows up as $T$ becomes large.
Hence, we further approximate the notion of {\tt RP-Regret} by {truncating}  the strategies far from the current timestep $t$. Specifically, instead of calculating the expected loss over all possible {full histories from $1$ to $t$}, we only take the latest histories of length no more than $m$ into account, where $m$ is a constant. With Condition \ref{assumption:2-forgetful} satisfied for all the players, in \Cref{corollary:approximation}, we prove that such an approximation will only cause an error of $2C_m^\gamma$ that decays exponentially with respect to $m$.
Formally, we define
\begin{align}
    &J_T^m(\bpi_{1:T})  \coloneqq  \sum_{t=1}^T f^{\min\{t-1,m\}}(\bpi_{t-\min\{t-1,m\}:t}),\quad R_T^m \coloneqq  J_T^m(\bpi_{1:T})-\min_{\hat\bpi\uo_{1:T}\in\cC_T\uo} J_T^m((\hat\bpi\uo_{1:T}, \bpi^{(-1)}_{1:T}))\label{eq:m-regret-def}
\end{align}
where $m\in\NN$, and  $\cC_T\ui\subseteq \rbr{\cX\ui}^T$ is the set of all the strategies of player $i$ with a bounded variation satisfying Conditions \ref{assumption:1-variation} and  \ref{assumption:2-forgetful}.
Note that the expected loss $f^{\min\{t-1,m\}}(\bpi_{t-\min\{t-1,m\}:t})$ is only related to the strategies from timestep $t-m$ to $t$, which is  more efficient to optimize compared to $f^{t-1}(\bpi_{1:t})$.

Next, we will verify that $f^{\min\{t-1,m\}}(\bpi_{t-\min\{t-1,m\}:t})$ can be  a good approximation of $f^{t-1}(\bpi_{1:t})$. In this case, $J_T^m$ and $R_T^m$ will also approximate $J_T$ and $R_T$ well, respectively. 

\begin{lemma}
\label{corollary:approximation}
 Suppose Condition \ref{assumption:2-forgetful} is satisfied for every player $i\in\cN$.  At any timestep $t$, for any $m\leq t-1$, any history $\bh'\in\cH_{t-m-1}$ and $\ba\in\cA$, when $\gamma\leq \frac{1}{2(N+2)}$, we have
\begin{align}
    |f^{\min\{t-1,m\}}(\bpi_{t-\min\{t-1,m\}:t})-f^{t-1}(\bpi_{1:t})|\leq 2C_m^\gamma,
\end{align}
where $C_m^\gamma \coloneqq (2N+1)^{m+1}\gamma^{m+1}$.
\end{lemma}
Therefore, when $\gamma\leq \frac{1}{2(N+2)}$, the approximation error $|R_T^m-R_T|$ decays exponentially with respect to $m$. The proof is postponed to Appendix \ref{appendix:auxiliary-lemmas}.

\section{Bounding $R_T^m$ by $\bar R_T^m$ and Switching Cost}
\label{appendix:Lipschitz-property-OCO}

\subsection{{Mitigating} the Dependence on Past Strategies}
\label{sec:unary-loss}

The expected loss at timestep $t$ in the online learning literature typically only depends on the strategy at that single timestep \citep{hazan2016introduction}. Therefore, to facilitate the analysis,
we define the following vanilla {\tt RP-Regret} over the unary loss, \emph{i.e.,} the expected loss $f_t^m$ defined in the following only depends on $\bpi\uo_t$, instead of  $\bpi\uo_{t-m:t}$. {Such a technique has also been exploited in the literature when \emph{memory} is involved in online learning \citep{merhav2002sequential-memory-regret-init,arora2012online,arora2018policy,anava2015online-OCO-framework,zhao2022non-dynamic-policy-regret}.}
Specifically, for any strategy profile $\bpi\uo\in\cX^{(1)}$ (of length $1$), any integer $m\geq 0$, and any timestep $t>m$, we can define the unary expected loss:
\begin{align}
    &f_t^m(\bpi\uo) \coloneqq f^m(((\bpi\uo,\bpi_{t-m}\uno),(\bpi\uo,\bpi_{t-m+1}\uno),...,(\bpi\uo,\bpi_t\uno))).\label{eq:unary-loss-def}
\end{align}
Then, we can define the corresponding cumulative loss and regret as follows:
\begin{align}
    \bar J_T^m(\bpi_{1:T})  \coloneqq  \sum_{t=1}^T f_t^{\min\{t-1,m\}}(\bpi\uo_t),\qquad~~~
    \bar R_T^m  \coloneqq  \bar J_T^m(\bpi_{1:T})-\min_{\hat\bpi\uo_{1:T}\in \cC_T\uo} \bar J_T^m((\hat\bpi\uo_{1:T},\bpi\uno_{1:T})).\label{eq:m-regret-no-memory-def}
\end{align}

Note that \mbox{$f_t^m(\bpi\uo_t)=f^m(((\underbrace{\bpi_t\uo,\bpi_t\uo,...,\bpi_t\uo}_{m+1}), \bpi\uno_{t-m:t}))$}. Then, by the Lipschitz continuity (with some constant $C_{\rm Lips}>0$) of $f^m$ with respect to $\bpi_{t-m:t}\uo$ (cf. Lemma \ref{lemma:Lipschitz}), we have%
\begin{align}
    \abr{R_T^m-\bar R_T^m}\leq C_{\rm Lips}m^2 \max_{\hat\bpi\uo_{1:T}\in \cC_T\uo} \sum_{t=2}^T\rbr{\nbr{\bpi\uo_{t-1}-\bpi\uo_t}_\infty+\nbr{\hat\bpi\uo_{t-1}-\hat\bpi\uo_t}_\infty}.
\end{align}
The right-hand side (RHS) is sublinear in $T$ when we have Condition \ref{assumption:1-variation} on the comparator $\hat\bpi_{1:T}\uo$, and the strategy of player 1, the regret minimizer, also satisfies Condition \ref{assumption:1-variation}\footnote{Condition \ref{assumption:1-variation} on the strategy generated by the regret minimizer is neither an assumption nor a pre-defined condition, but a property that our regret minimization algorithm should have{, as to be detailed later in \S\ref{sec:4-rm-adaptive}.}}. Hence, instead of minimizing $R_T^m$, we can minimize $\bar R_T^m$ using any online learning algorithm that can also guarantee $\sum_{t=2}^T \nbr{\bpi\uo_{t-1}-\bpi\uo_t}_\infty$ to be sublinear in $T$.

\subsection{Bounding the difference between $J_T^m$ and $\bar J_T^m$}

Following the framework of online convex optimization with memory  \citep{anava2015online-OCO-framework, zhao2022non-dynamic-policy-regret}, we will use the Lipschitz continuity of $f^{\min\{t-1,m\}}$ to remove the dependence on past strategies with an additional cost corresponding to the variation defined %
in Condition \ref{assumption:1-variation}.
We first show $f^m(\bpi_{1:m+1})$ is Lipschitz with respect to $\bpi_{1:m+1}$ for any $m\in\NN$ and joint strategy $\bpi_{1:m+1}$.

\begin{lemma}[Lipschitz Continuity {of $f^m$}]
\label{lemma:Lipschitz}
For any $m\in\NN$ and two arbitrary strategy-profile vectors $\bar\bpi,\tilde\bpi$  of length $m+1$, we have
\begin{align}
        |f^m(\bar\bpi)-f^m(\tilde\bpi)|
        \leq C_{\rm Lips}\sum_{t=1}^{m+1}\nbr{\bar\bpi_t-\tilde \bpi_t}_\infty\leq C_{\rm Lips}\sum_{i=1}^N\sum_{t=1}^{m+1}\nbr{\bar\bpi_t\ui-\tilde \bpi_t\ui}_\infty\notag
    \end{align}
    where we denote  
    \begin{align}
        &C_{\rm Lips} \coloneqq |\cA|^2,\qquad\qquad\qquad\qquad\nbr{\bar\bpi_t-\tilde \bpi_t}_\infty=\max_{\ba\in\cA,\bh\in\cH_m} \abr{\bar\bpi_t(\ba\given \bh)-\tilde\bpi_t(\ba\given \bh)}\\
        &\qquad\qquad\nbr{\bar\bpi_t\ui-\tilde \bpi_t\ui}_\infty=\max_{a_i\in\cA_i,\bh\in\cH_m} \abr{\bar\bpi_t\ui(a_i\given \bh)-\tilde\bpi_t\ui(a_i\given \bh)}.    \end{align}
\end{lemma}

The proof is postponed to Appendix \ref{appendix:lemma-Lipschitz}. Then, we can bound $R_T^m$ by $\bar R_T^m$ with an additional switching cost (accumulated variation of $\bpi_t\uo$ over time) by using Lemma \ref{lemma:Lipschitz} and Lemma \ref{lemma:accumulated-variation-bound} in the following.

\begin{lemma}
\label{lemma:accumulated-variation-bound}
For any sequence of vectors $\bx_{1:T}$, integers $0<K\leq T$ and a real number $p>0$ ($p$ can be $+\infty$ for ease of notation), we have
\begin{align}
    \sum_{t=1}^T \sum_{s=\max\cbr{t-K, 1}}^{t-1} \nbr{\bx_t-\bx_s}_p\leq K^2\sum_{t=2}^T \nbr{\bx_t-\bx_{t-1}}_p.
\end{align}
\end{lemma}
The proof is postponed to Appendix \ref{appendix:other-lemma}.

By Lipschitz continuity of $f^m$, for any strategy vector $\tilde\bpi\in\rbr{\cX}^T$ of length $T$, we have
\begin{align*}
    |J_T^m(\tilde\bpi)-\bar J_T^m(\tilde \bpi)|\leq& C_{\rm Lips} \sum_{t=2}^T \sum_{s=\max\cbr{t-m,1}}^{t-1} \nbr{\tilde\bpi_t\uo-\tilde\bpi_s\uo}_\infty \overset{(i)}{\leq} C_{\rm Lips} m^2\underbrace{\sum_{t=2}^T \nbr{\tilde\bpi_{t-1}\uo-\tilde\bpi_t\uo}_\infty}_{\text{Switching Cost}}
\end{align*}
where $(i)$ is obtained directly from Lemma \ref{lemma:accumulated-variation-bound}.

\section{Minimization of {\tt Repeated Policy Regret} with an Oracle}
\label{sec:oracle-ONCO}

In this section, we will extend the result in \cite{suggala2020online-ONCO-oracle} to our setting where the comparator can change over time with a sublinear accumulated variation instead of a fixed comparator. We need the oracle $\cO$ to achieve that.
\begin{definition}[Optimization Oracle \citep{suggala2020online-ONCO-oracle}]
\label{def:oracle-FTPL}
    For a  function $f\colon \cX\to \RR$ and a vector $\bsigma\in\RR^d$, the optimization oracle $\cO(f-\bsigma)$ returns $\bx\in\cX$ so that
\begin{align}
    f(\bx)-\inner{\bsigma}{\bx}\leq\inf_{\bx'\in\cX} f(\bx')-\inner{\bsigma}{\bx'} + \alpha + \beta\nbr{\bsigma}_1.
\end{align}
\end{definition}

We can then design Algorithm \ref{alg:RM-Oracle} by using the oracle above. Note that all the strategies of player $1$ are in the subspace of $\cX\uo$ that satisfies \Cref{assumption:2-forgetful}, which is denoted as $\cX\uo_\gamma$.

\begin{algorithm}
\caption{Minimizing Non-Convex Functions}
\label{alg:RM-Oracle}
\begin{algorithmic}
\For{$s=1,2,...,T/K$}
\For{$k=1,2,...,K$}
\State{Sample $\cbr{\sigma_{(s-1)K+k,j}}_{j=1}^{|\cA_1|\cdot|\bigcup_{m'=0}^m \cH_{m'}|}\overset{i.i.d.}{\sim} {\rm Exp}(\eta)$}
\State{\Comment{$z\sim {\rm Exp}(\eta)$ means that $\Pr(z\geq w)=\exp(-\eta w)$}}
\State{Predict $\bpi\uo_{(s-1)K+k}$ as}
\begin{align}
    \bpi\uo_{(s-1)K+k}\leftarrow\argmin_{\bpi\uo\in\cX\uo_\gamma} \rbr{\sum_{k'=1}^{k-1}f^m(\bpi\uo,\bpi\uno_{(s-1)K+k'-m:(s-1)K+k'}) -\inner{\bsigma_{(s-1)K+k}}{\bpi\uo}}
\end{align}
\State{\Comment{The $\argmin$ above is achieved by the oracle.}}
\EndFor
\EndFor
\end{algorithmic}
\end{algorithm}

Intuitively, we will divide all $T$ timesteps into $T/K$ episodes, with $K$ timesteps in each episode ($K$ is the hyper-parameter that we will specify later). Then, for each episode, we will run \cite[Algorithm 1]{suggala2020online-ONCO-oracle} from scratch. This is similar to the restart techniques in the non-stationary MG literature \cite{pmlr-v139-mao21b-non-stationary-restart}, to overcome the non-stationarity.

Throughout this proof, when $t\leq m$ we interpret expressions such as $f^m(\bpi_{t-m:t})$ and $f^m(\bpi\uo,\bpi_{t-m:t}\uno)$ as their available-history versions $f^{t-1}(\bpi_{1:t})$ and $f^{t-1}(\bpi\uo,\bpi_{1:t}\uno)$, respectively. Equivalently, one may start the displayed $m$-memory bounds after the first $m$ rounds and absorb the first $m$ losses into an additive $O(m)$ term.

Firstly, we have
\begin{align*}
    \EE[R_T]=&\EE\sbr{\sup_{\hat\bpi_{1:T}\uo\in \cC_T}\sum_{t=1}^T f^{t-1}(\bpi_{1:t})-f^{t-1}(\hat\bpi_{1:t}\uo,\bpi_{1:t}\uno)}\\
    \overset{(i)}{\leq}& \EE\sbr{\sup_{\hat\bpi_{1:T}\uo\in \cC_T}\sum_{t=1}^T f^m(\bpi_{t-m:t})-f^m(\hat\bpi_{t-m:t}\uo,\bpi_{t-m:t}\uno)}+4C_m^\gamma T\\
    \overset{(ii)}{\leq}& \EE\sbr{\sup_{\hat\bpi_{1:T}\uo\in \cC_T}\sum_{t=1}^T f^m(\bpi_t\uo,\bpi_{t-m:t}\uno)-f^m(\hat\bpi_t\uo,\bpi_{t-m:t}\uno)}+4C_m^\gamma T\\
    &+C_{\rm Lips} m^2\sum_{t=2}^T \EE\sbr{\nbr{\bpi_t\uo-\bpi_{t-1}\uo}_\infty}+C_{\rm Lips} m^2\sup_{\hat\bpi_{1:T}\uo\in \cC_T}\sum_{t=2}^T \nbr{\hat\bpi_t\uo-\hat\bpi_{t-1}\uo}_\infty\\
    \overset{(iii)}{\leq}& \EE\sbr{\sup_{\hat\bpi_{1:T}\uo\in \cC_T}\sum_{t=1}^{T/K}\sum_{s=t\cdot K-K+1}^{t\cdot K} f^m(\bpi_{s}\uo,\bpi_{s-m:s}\uno)- f^m(\hat\bpi_{t\cdot K}\uo,\bpi_{s-m:s}\uno)}+4C_m^\gamma T\\
    &+C_{\rm Lips} m^2\sum_{t=2}^T \EE\sbr{\nbr{\bpi_t\uo-\bpi_{t-1}\uo}_\infty}+C_{\rm Lips} (m^2+(m+1)K^2)\sup_{\hat\bpi_{1:T}\uo\in \cC_T}\sum_{t=2}^T \nbr{\hat\bpi_t\uo-\hat\bpi_{t-1}\uo}_\infty
\end{align*}
where we define $f^m(\bpi\uo,\bpi_{t-m:t}\uno) \coloneqq f^m(((\bpi\uo,\bpi_{t-m}\uno),(\bpi\uo,\bpi_{t-m+1}\uno),...,(\bpi\uo,\bpi_t\uno)))$. $(i)$ is by Lemma \ref{corollary:approximation}, and $(ii),(iii)$ are by Lemma \ref{lemma:Lipschitz} and Lemma \ref{lemma:accumulated-variation-bound}. Then,

\begin{align*}
    \EE[R_T]\leq& \EE\sbr{\sum_{t=1}^{T/K} \underbrace{\sup_{\hat\bpi\uo\in\cX\uo_\gamma}\sum_{s=t\cdot K-K+1}^{t\cdot K} f^m(\bpi_s\uo,\bpi_{s-m:s}\uno)- f^m(\hat\bpi\uo,\bpi_{s-m:s}\uno)}_{\circled{1}}}+4C_m^\gamma T\\
    &+C_{\rm Lips} m^2\sum_{t=2}^T \EE\sbr{\nbr{\bpi_t\uo-\bpi_{t-1}\uo}_\infty}+C_{\rm Lips} (m^2+(m+1)K^2)\sup_{\hat\bpi_{1:T}\uo\in \cC_T}\sum_{t=2}^T \nbr{\hat\bpi_t\uo-\hat\bpi_{t-1}\uo}_\infty.
\end{align*}
Since $\circled{1}$ is exactly the external regret of a non-convex function as in \cite{suggala2020online-ONCO-oracle}, we can directly use \cite[Theorem 1]{suggala2020online-ONCO-oracle}, then
\begin{align*}
    \EE[R_T]\leq& \eta C_{\rm Lips}^2 |\cA|^{2m+2}(1+m)^2 T+2\frac{|\cA|^{m+1}}{\eta}\frac{T}{K}+4C_m^\gamma T+C_{\rm Lips} (m^2+(m+1)K^2)P_T\\
    &+\alpha T + \beta |\cA|^{m+1}\rbr{\frac{1}{\eta} + C_{\rm Lips}(1+m)} T.
\end{align*}

By choosing $\eta=\frac{1}{\sqrt K}$ and $K=\rbr{\frac{T}{P_T}}^{0.4}$, we have
\begin{align*}
    &\EE\sbr{R_T}\\
    \leq& 4C_m^\gamma T+C_{\rm Lips}^2|\cA|^{2m+2}(m+1)^2P_T^{0.2}T^{0.8}+C_{\rm Lips}m^2 P_T\\
    &+C_{\rm Lips}(m+1)T^{0.8} P_T^{0.2}+T^{0.6}P_T^{0.4}+|\cA|^{m+1}T^{0.8} P_T^{0.2}\\
    &+\alpha T + \beta |\cA|^{m+1}\rbr{\rbr{\frac{T}{P_T}}^{0.2} + C_{\rm Lips}(1+m)} T\\
    =&4C_m^\gamma T+\rbr{C_{\rm Lips}^2|\cA|^{2m+2}(1+m)^2+C_{\rm Lips}(m+1)+|\cA|^{m+1}}T^{0.8} P_T^{0.2}+C_{\rm Lips}m^2 P_T+T^{0.6}P_T^{0.4}\\
    &+\alpha T + \beta |\cA|^{m+1}\rbr{\rbr{\frac{T}{P_T}}^{0.2} + C_{\rm Lips}(1+m)} T.
\end{align*}
Therefore, we have
\begin{align*}
    \EE\sbr{R_T}\leq O \rbr{ T^{0.8}P_T^{0.2} + \alpha T + \beta \frac{T^{1.2}}{P_T^{0.2}} + C_m^\gamma T}.
\end{align*}
For any desired accuracy $\epsilon>0$, when the optimization oracle is accurate enough ($\alpha,\beta$ small enough) and by picking $m=O\rbr{\log \frac{1}{\epsilon}}$, we can achieve $\EE\sbr{\frac{R_T}{T}}\leq \epsilon$.
\qed

\section{Minimization of {\tt Local RP-Regret}}

\subsection{Hardness Results}
\label{appendix:hardness-local-RP-Regret}

\begin{lemma}[Lemma \ref{lemma:wo1w2} for {\tt Local RP-Regret}]
\label{lemma:local-wo1w2}
When we only have Condition \ref{assumption:imperfect-memory} on both the comparator and the opponent, we will get $\Omega(T)$ {\tt Local RP-Regret} in the worst case.
\proof
Consider a two-player coin-flipping game shown in Table \ref{table:coin-tossing}. The opponent (the column player) can flip the coin to any side they want and our player (the row player) needs to guess which side the coin is, with loss $0$ for a correct guess and loss $1$ for a wrong guess. In this case, without Condition \ref{assumption:1-variation}, the opponent can choose a deterministic sequence adversarially so that they flip the coin to a side that we guess with probability no larger than $0.5$. Hence, we get linear regret. Because the comparator can guess this fixed sequence correctly every time, we have $J_T(\bpi_{1:T})-J_T(\tilde \bpi_{1:T}^{(1),s},\bpi_{1:T}\uno)\geq 0.5$ for all $s=1,2,...,T$ so that $R_T^{\rm local}\geq 0.5 T$. In this case, neither the opponent nor the comparator needs any memory so that they both satisfy Condition \ref{assumption:imperfect-memory}.\qed
\end{lemma}

\begin{lemma}[Lemma \ref{lemma:w1wo2} for {\tt Local RP-Regret}]
With Condition \ref{assumption:1-variation} on the comparator and Condition \ref{assumption:imperfect-memory} only for the opponent, we will get $\Omega(T)$ {\tt Local RP-Regret} in the worst case.
\proof
We consider the same coin-flipping game in Table \ref{table:coin-tossing}.
Since the comparator is not subject to Condition \ref{assumption:imperfect-memory}, its fixed strategy can behave differently at different timesteps by noticing the different lengths of history. In particular, the comparator can choose the fixed strategy $\pi\uo(g(L(\bh)+1)\given \bh)\equiv 1$ where $g\colon \cbr{1,2,...,T}\to\{{\rm Guess~Up},{\rm Guess~Down}\}$ at all timesteps. Back to the coin-flipping game, the comparator can deterministically guess up or down at every timestep by letting $g(t)$ take different values.

At the same time, the opponent can adversarially flip the coin as in Lemma \ref{lemma:local-wo1w2} so that we will get $J_T(\bpi_{1:T})-J_T(\tilde \bpi_{1:T}^{(1),s},\bpi_{1:T}\uno)\geq 0.5$ for all $s=1,2,...,T$ as in the proof of Lemma \ref{lemma:local-wo1w2}. In this case, the strategy of the comparator is fixed and the opponent is subject to Condition \ref{assumption:imperfect-memory}, but $R_T^{\rm local}\geq 0.5T$. \qed

\end{lemma}

\begin{lemma}[Lemma \ref{lemma:w1wo2-bound} for {\tt Local RP-Regret}]
When only the comparator satisfies both Condition \ref{assumption:1-variation} and Condition \ref{assumption:imperfect-memory}, there is $\Omega(T)$ {\tt Local RP-Regret}.
\proof

Consider the same augmented Prisoner's Dilemma construction as in the proof of Lemma \ref{lemma:w1wo2-bound}. As noted there, the matrix may be affinely rescaled to $[0,1]$ without changing the linear-regret conclusion. At timestep $1$, the opponent chooses $M_1$ if player $1$ plays $C$ with probability less than $1/2$, and chooses $M_2$ otherwise. Then, against the actually played strategy of player $1$, the opponent plays $D$ with probability at least $1/2$ at every timestep $t\geq 2$, and player $1$'s expected loss at each such timestep is at least $-5/2$.

Now choose the comparator after the opponent strategy has been fixed: if the opponent chose $M_1$, the comparator always plays $C$. If the opponent chose $M_2$, the comparator always plays $D$. This comparator is time-invariant and history-independent, so it satisfies Condition \ref{assumption:1-variation} and Condition \ref{assumption:imperfect-memory}. Consider the single local deviation at timestep $s=1$ to this comparator. Under this deviation, the opponent is driven to play $C$ from timestep $2$ onward, so the deviating player receives loss at most $-3$ at every timestep $t\geq2$. Hence
\begin{align*}
    J_T(\bpi_{1:T})-J_T(\tilde \bpi_{1:T}^{(1),1},\bpi_{1:T}\uno)\geq \frac{1}{2}(T-1),
\end{align*}
and therefore $R_T^{\rm local}\geq \Omega(T)$.
\qed

\end{lemma}

\subsection{Proof of Theorem \ref{theorem:local-regret}}
\label{appendix:proof-of-local-regret-theorem}

In this section, we will use $\bpi\uo_{1:T}$ as the strategy of player 1 generated by the regret minimizer and $\bpi\uno_{1:T}$ to denote the strategy generated by the adversary. Also, we will use $\bpi_{1:T}$ to denote the joint strategy.

Therefore, according to \Cref{eq:local-loss-def},
\begin{align*}
    f_t^{t-1,{\rm local}}(\bar \bpi\uo_{1:t})=(T-t) f^{t-1}(\bpi\uo_{1:t},\bpi\uno_{1:t})+\sum_{s=1}^t f^{t-1}((\bpi\uo_{1:s-1}, \bar\bpi\uo_s, \bpi\uo_{s+1:t}),\bpi\uno_{1:t}).
\end{align*}
and
$f_t^{t-1,{\rm local}}(\bpi_{1:t}\uo)=Tf^{t-1}(\bpi_{1:t})$. Then we have
\begin{align}
    R_T^{\rm local}=\sum_{t=1}^T f^{t-1,\rm local}_t (\bpi_{1:t}\uo)-\min_{\hat\bpi\uo_{1:T}\in\cC_T\uo}\sum_{t=1}^T f^{t-1,\rm local}_t (\hat\bpi_{1:t}\uo).
\end{align}

For any $\hat \bpi\uo_{1:t}$, we can still get
\begin{align}
    &f_t^{t-1,{\rm local}}(\bpi\uo_{1:t})-f_t^{t-1,{\rm local}}(\hat \bpi\uo_{1:t})\notag\\
    =& \sum_{s=1}^{t-m-1} f^{t-1}(\bpi\uo_{1:t},\bpi\uno_{1:t})-f^{t-1}((\bpi\uo_{1:s-1}, \hat\bpi\uo_s, \bpi\uo_{s+1:t}),\bpi\uno_{1:t})\label{eq:local-proof-eq1-1-line}\\
    &+\sum_{s=t-m}^t f^{t-1}(\bpi\uo_{1:t},\bpi\uno_{1:t})-f^{t-1}((\bpi\uo_{1:s-1}, \hat\bpi\uo_s, \bpi\uo_{s+1:t}),\bpi\uno_{1:t}).\label{eq:local-proof-eq1-2-line}
\end{align}
For \Cref{eq:local-proof-eq1-1-line}, we have
\begin{align*}
    &\sum_{s=1}^{t-m-1}\abr{f^{t-1}((\bpi\uo_{1:s-1}, \hat\bpi\uo_s, \bpi\uo_{s+1:t}),\bpi\uno_{1:t})-f^{t-1}(\bpi\uo_{1:t},\bpi\uno_{1:t})}\\
    \leq& \sum_{s=1}^{t-m-1}\Big(\abr{f^{t-1}((\bpi\uo_{1:s-1}, \hat\bpi\uo_s, \bpi\uo_{s+1:t}),\bpi\uno_{1:t})-f^{t-s}(\bpi\uo_{s+1:t}, \bpi\uno_{s+1:t})}\\
    &+\abr{f^{t-s}(\bpi\uo_{s+1:t}, \bpi\uno_{s+1:t})-f^{t-1}(\bpi\uo_{1:t},\bpi\uno_{1:t})}\Big)\\
    \overset{(i)}{\leq}& 4\sum_{s=1}^{t-m-1} C_{t-s}^\gamma \overset{(ii)}{\leq} 4N C_m^\gamma,
\end{align*}
where $(i)$ is by Lemma \ref{corollary:approximation} and $(ii)$ is by the definition of $C_m^\gamma$, and by the condition $\gamma\leq \frac{1}{2(N+2)}$, which ensures Lemma \ref{corollary:approximation} holds. For \Cref{eq:local-proof-eq1-2-line}, we have
\begin{align*}
    &\sum_{s=t-m}^t f^{t-1}(\bpi\uo_{1:t},\bpi\uno_{1:t})-f^{t-1}((\bpi\uo_{1:s-1}, \hat\bpi\uo_s, \bpi\uo_{s+1:t}),\bpi\uno_{1:t})\\
    \leq& \sum_{s=t-m}^t f^m(\bpi\uo_{t-m:t},\bpi\uno_{t-m:t}) - f^m((\bpi\uo_{t-m:s-1}, \hat\bpi\uo_s, \bpi\uo_{s+1:t}),\bpi\uno_{t-m:t})\\
    &+ \sum_{s=t-m}^t \abr{f^{t-1}((\bpi\uo_{1:s-1}, \hat\bpi\uo_s, \bpi\uo_{s+1:t}),\bpi\uno_{1:t})-f^m((\bpi\uo_{t-m:s-1}, \hat\bpi\uo_s, \bpi\uo_{s+1:t}),\bpi\uno_{t-m:t})}\\
    &+\sum_{s=t-m}^t\abr{f^{t-1}(\bpi\uo_{1:t},\bpi\uno_{1:t})-f^m(\bpi\uo_{t-m:t},\bpi\uno_{t-m:t})}\\
    \leq& \sum_{s=t-m}^t f^m(\bpi\uo_{t-m:t},\bpi\uno_{t-m:t}) - f^m((\bpi\uo_{t-m:s-1}, \hat\bpi\uo_s, \bpi\uo_{s+1:t}),\bpi\uno_{t-m:t})+4(m+1)C_m^\gamma.
\end{align*}

For notational simplicity, let $\hat\bpi_{1:T}\uo=\argmin_{\bar\bpi\uo_{1:T}\in\cC_T\uo} \sum_{t=1}^T f^{t-1,\rm local}_t (\bar\bpi_{1:t}\uo)$. Then,
\begin{align*}
    R_T^{\rm local}=&\sum_{t=1}^T f^{t-1,\rm local}_t (\bpi_{1:t}\uo)-\sum_{t=1}^T f^{t-1,\rm local}_t (\hat\bpi_{1:t}\uo)\\
    \leq& \sum_{t=m+1}^T\sum_{s=t-m}^t\rbr{f^m(\bpi\uo_{t-m:t},\bpi\uno_{t-m:t}) - f^m((\bpi\uo_{t-m:s-1}, \hat\bpi\uo_s, \bpi\uo_{s+1:t}),\bpi\uno_{t-m:t})}+4(N+m+1)C_m^\gamma T\\
    &+\sum_{t=1}^m \abr{f^{t-1,\rm local}_t (\bpi_{1:t}\uo)-f^{t-1,\rm local}_t (\hat\bpi_{1:t}\uo)}\\
    \overset{(i)}{\leq}& \sum_{t=m+1}^T\sum_{s=t-m}^t\rbr{f^m((\bpi\uo_{t-m:s-1}, \bpi\uo_t, \bpi\uo_{s+1:t}),\bpi\uno_{t-m:t}) - f^m((\bpi\uo_{t-m:s-1}, \hat\bpi\uo_t, \bpi\uo_{s+1:t}),\bpi\uno_{t-m:t}) }\\
    &+4(N+m+1)C_m^\gamma T+m^2+C_{\rm Lips}\sum_{t=m+1}^T \sum_{s=t-m}^{t-1}\rbr{\nbr{\bpi\uo_s-\bpi\uo_t}_\infty+\nbr{\hat\bpi\uo_s-\hat\bpi\uo_t}_\infty}\\
    \overset{(ii)}{\leq}& \sum_{t=m+1}^T\sum_{s=t-m}^t\rbr{f^m((\bpi\uo_{t-m:s-1}, \bpi\uo_t, \bpi\uo_{s+1:t}),\bpi\uno_{t-m:t}) - f^m((\bpi\uo_{t-m:s-1}, \hat\bpi\uo_t, \bpi\uo_{s+1:t}),\bpi\uno_{t-m:t}) }\\
    &+4(N+m+1)C_m^\gamma T+m^2+C_{\rm Lips}m^2\sum_{t=m+2}^T \rbr{\nbr{\bpi\uo_{t-1}-\bpi\uo_t}_\infty+\nbr{\hat\bpi\uo_{t-1}-\hat\bpi\uo_t}_\infty}
\end{align*}
where $(i)$ uses Lemma \ref{lemma:Lipschitz} and $(ii)$ uses Lemma \ref{lemma:accumulated-variation-bound}.

Notice that $\sum_{s=t-m}^t f^m((\bpi\uo_{t-m:s-1}, \bar\bpi\uo_t, \bpi\uo_{s+1:t}),\bpi\uno_{t-m:t})$ is a linear function with respect to $\bar\bpi_t\uo$. So,
\begin{align*}
    &\sum_{t=m+1}^T\sum_{s=t-m}^t\rbr{f^m((\bpi\uo_{t-m:s-1}, \bpi\uo_t, \bpi\uo_{s+1:t}),\bpi\uno_{t-m:t}) - f^m((\bpi\uo_{t-m:s-1}, \hat\bpi\uo_t, \bpi\uo_{s+1:t}),\bpi\uno_{t-m:t}) }\\
    =& \sum_{t=m+1}^T\inner{\bg_t}{\bpi\uo_t-\hat\bpi\uo_t}
\end{align*}
where, for $t\geq m+1$,
\begin{align*}
    \bg_t\coloneqq \nabla_{\bar\bpi\uo}\left.\sum_{s=t-m}^t f^m((\bpi\uo_{t-m:s-1},\bar\bpi\uo,\bpi\uo_{s+1:t}),\bpi\uno_{t-m:t})\right|_{\bar\bpi\uo=\bpi_t\uo}.
\end{align*}
We set $\bg_t=0$ for $t\leq m$.

Therefore, by Lemma \ref{lemma:PGD-upper-bound}, we have
\begin{align*}
    R_T^{\rm local}\leq& \frac{\eta}{2}\sum_{t=m+1}^T\nbr{\bg_t}^2+\frac{1}{2\eta}\nbr{\bpi_{m+1}\uo-\hat\bpi_{m+1}\uo}^2+\frac{2|\cH_{m+1}|}{\eta} \sum_{t=m+2}^T \nbr{\hat\bpi_{t-1}-\hat\bpi_t}_\infty\\
    &+4(N+m+1)C_m^\gamma T+m^2+C_{\rm Lips}m^2\sum_{t=m+2}^T \rbr{\nbr{\bpi\uo_{t-1}-\bpi\uo_t}_\infty+\nbr{\hat\bpi\uo_{t-1}-\hat\bpi\uo_t}_\infty}
\end{align*}
since $D_1= 2|\bigcup_{k=0}^m\cH_k|\leq 2|\cH_{m+1}|$. By
\begin{align*}
    \nbr{\bg_t}\leq& \nbr{\bg_t}_1\\
    =& \sum_{\bh_1\in\bigcup_{k=0}^m \cH_k,a_1\in\cA_1}\abr{ \sum_{\bh_2\in\cH_{m+1},\bh_{2,1:L(\bh_1)}=\bh_1,h_{2,L(\bh_1)+1,1}=a_1}\cL_1(\bh_{2,m+1})\Pr(\bh_2|h_{2,L(\bh_1)+1,1}=a_1;\bpi_{t-m:t})}\\
    \leq& \sum_{\bh_2\in\cH_{m+1}}\sum_{s=0}^m \Pr(\bh_2|h_{2,s+1,1};\bpi_{t-m:t})=(m+1) |\cA_1|.
\end{align*}
Note that we use $h_{2,s,1}\in\cA_1$ to denote the action player 1 played at timestep $s$ and $\Pr(\bh_2|h_{2,s+1,1};\bpi_{t-m:t})$ means the probability that $\bh_2$ occurs when playing $\bpi_{t-m:t}$ conditioned on observing $h_{2,s+1,1}\in\cA_1$ at timestep $s+1$.

Then, for any $\hat\bpi_{1:T}$, we have
\begin{align*}
    R_T^{\rm local} 
    \overset{(i)}{\leq}& \frac{\eta}{2} (m+1)^2 |\cA_1|^2 T+\frac{2|\cH_{m+1}|^2}{\eta}+\frac{2|\cH_{m+1}|}{\eta} \sum_{t=m+2}^T \nbr{\hat\bpi_{t-1}-\hat\bpi_t}_\infty\\
    &+4(N+m+1)C_m^\gamma T+m^2+C_{\rm Lips}m^2\sum_{t=m+2}^T \nbr{\hat\bpi\uo_{t-1}-\hat\bpi\uo_t}_\infty+C_{\rm Lips}m^2(m+1) |\cA_1| \eta T
\end{align*}
where $(i)$ is because $\bpi_t\uo=\Proj{\cX\uo_\gamma}{\bpi_{t-1}\uo-\eta \bg_t}$ and $\nbr{\bg_t}\leq (m+1) |\cA_1|$. 
Note that by definition of $\cC_T\uo$ and $\hat\bpi\uo_{1:T}\in\cC_T\uo$, $\sum_{t=m+2}^T \nbr{\hat\bpi\uo_{t-1}-\hat\bpi\uo_t}_\infty\leq P_T$. By choosing $\eta=\sqrt{\frac{2|\cH_{m+1}|\rbr{|\cH_{m+1}|+P_T}}{\rbr{C_{\rm Lips}m^2(m+1) |\cA_1|+(m+1)^2|\cA_1|^2/2}T}}$, we have
\begin{align*}
    R_T^{\rm local}\leq& 2\sqrt{\rbr{2|\cH_{m+1}|\rbr{|\cH_{m+1}|+P_T}}\rbr{\rbr{C_{\rm Lips}m^2(m+1) |\cA_1|+(m+1)^2|\cA_1|^2/2}T}}\\
    &\qquad+4(N+m+1)C_m^\gamma T+m^2+C_{\rm Lips}m^2 P_T.\qedhere
\end{align*}

\section{Proof of Theorem \ref{theorem:occupancy-measure-RM-informal}}
\label{appendix:proof-of-regret-MDP-theorem}

\begin{algorithm}[h]
\caption{Regret Minimizer by Optimizing Occupancy Measure}
\label{alg:self-play}
\begin{algorithmic}
\State{Initialize $\pi_0\uo(\cdot\given \bh)\in\Delta_{|\cA_1|}$ as the uniform distribution over $a_1\in\cA_1$} %
\State{Initialize the corresponding $\bq_1$ in the MG with strategies $\cbr{\bpi_0\uo,\bpi_0^{(2)},...,\bpi_0^{(N)}}$.}

\State{Initialize the Policy Gradient Descent with Constraints in Algorithm \ref{algorithm:PGD-with-constraint}.}

\State{Initialize the convex set $\cX$ that $\bq$ lies in.}%
\begin{align}
    \forall \bq\in\cX,~~~\begin{cases}
        &\forall i=1,2,...,N, \bh\in\cH_M, a_i\in\cA_i,~~~\frac{\gamma}{|\cA_i|}\sum_{\ba'\in\cA} q(\bh,\ba')-\sum_{\ba_{-i}'\in\cA_{-i}}q(\bh,(a_i,\ba_{-i}'))\leq 0\\
        &\forall \bh\in\cH_M,~~~\frac{\gamma^{NM}}{|\cA|^M}-\sum_{\ba\in\cA} q(\bh,\ba)\leq 0\\
        &\sum_{\bh\in\cH_M,\ba\in\cA} q(\bh,\ba)=1,~~~~\bq\succeq 0\\
        &\forall \bh\in\cH_M,~~~\sum_{\ba\in\cA} q(\bh,\ba)=\sum_{\ba\in\cA} q((\ba,\bh_{1:M-1}), \bh_M)
    \end{cases}\label{eq:def-convex-set-self-play}
\end{align}

\For{$t=1,2,...,T$}
\For{$i=1,2,...,N$}
\State{We define the following approximate player strategy $\bpi_t^{[i]}$ at timestep $t$ as, }

\Comment{$\bpi_t^{[i]}\equiv \bpi_t\uo$ since player 1's strategy is determined by the algorithm}
\begin{align}
    \forall \bh\in\cH_M,a_i\in\cA_i,~~~\pi^{[i]}_t(a_i\given \bh)\coloneqq \frac{\sum_{\ba_{-i}'\in\cA_{-i}}q_t(\bh,(a_i,\ba_{-i}'))}{\sum_{\ba'\in\cA}q_t(\bh,\ba')}\label{eq:def-self-play-strategy-i}
\end{align}

\EndFor

\State{Receive constraint $g_t$ as}
\begin{align}
    &g_t(\bq) \coloneqq \sum_{\bh\in\cH_M,\ba\in\cA} (-1)^{\ind(D_t(\bh,\ba,\bq_t)\leq 0)}D_t(\bh,\ba,\bq)\leq 0\\
    &D_t(\bh,\ba,\bq)\coloneqq q(\bh,\ba)-\pi_t^{(-1)}(\ba_{-1}\given \bh)\sum_{\ba_{-1}'\in\cA_{-1}} q(\bh,(a_1,\ba_{-1}')).\notag
\end{align}

\State{Run update-rule Eq.  \eqref{eq:update-constraint-variable} and \Cref{eq:update-constraint-lambda} in Algorithm \ref{algorithm:PGD-with-constraint} to obtain $\bq_{t+1}$ from $\bq_t$. }

\EndFor
\end{algorithmic}
\end{algorithm}

\subsection{Important Lemmas}\label{sec:important_lemmas}

Here's the performance difference lemma for average-reward MDPs. For convenience, we will use the shorthand $\EE_{\bpi}$ to denote $\EE_{\ba_t\sim\bpi(\cdot\given \bh_t), s_{t+1}\sim \Pr(\cdot\given \bh_t,\ba_t)}$. Also, we define
\begin{align*}
    \rho^{\bpi} \coloneqq \lim_{T\to+\infty}\frac{1}{T}\EE_{\bpi}\left[\sum_{t=0}^{T-1} \cL_1(\bh_t,\ba_t)\right]
\end{align*}
where $\Pr$ is the transition probability and $\cL_1$ is the loss for player $1$. It is easy to see from the proof of Lemma \ref{lemma:MDP-average-iterate}, the initial state $\bh_0$ does not affect the value of $\rho^{\bpi}$ so we omit the initial state here. Also, in this section, we assume that we are controlling player $1$ so we will omit the subscript $1$ in the following.

\begin{lemma}[Performance Difference Lemma \citep{cao1999single-performance-difference-lemma-average-reward}]
\label{lemma:performance-difference-lemma}
    Consider the MG which is aperiodic unichain. For any $\bh_0\in\cS$ and strategies $\bpi_1,\bpi_2$, we have
    \begin{align}
        \rho^{\bpi_2}-\rho^{\bpi_1}=&\sum_{\bh\in\cS}d^{\bpi_2}(\bh)\sum_{\ba\in\cA} \pi_2(\ba\given \bh)(Q^{\bpi_1}(\bh,\ba)-V^{\bpi_1}(\bh))\\
        =&\sum_{\bh\in\cS}d^{\bpi_2}(\bh)\sum_{\ba\in\cA} (\pi_2(\ba\given \bh)-\pi_1(\ba\given \bh))Q^{\bpi_1}(\bh,\ba),\notag%
    \end{align}
    where
    \begin{align}
        &Q^{\bpi}(\bh,\ba)=\EE_{\bh_0=\bh,\ba_0=\ba,\bpi}[\sum_{t=0}^\infty \big(\cL(\bh_t,\ba_t)-\rho^{\bpi}\big)]\\
        &V^{\bpi}(\bh)=\sum_{\ba\in\cA} \pi(\ba\given \bh)Q^{\bpi}(\bh,\ba)\label{eq:def-V-function}\\
        &d^{\bpi}(\bh)=\lim_{T\to+\infty}\frac{1}{T}\EE_{\bpi}\left[\sum_{t=0}^{T-1}\ind(\bh_t=\bh)\right].
    \end{align}
\end{lemma}
From Lemma \ref{lemma:always-reach}, the MG induced by $M$-bounded length memory defined in Definition \ref{def:induced-MDP} is aperiodic unichain. Under such a condition, $d^{\bpi}$ is fixed regardless of the initial state $\bh_0$.

\begin{lemma}[Upper Bound of $Q^{\bpi}(\bh,\ba)$]
\label{lemma:upper-bound-Q}
    The corresponding $Q^{\bpi}(\bh,\ba)$ defined in Lemma \ref{lemma:performance-difference-lemma} is bounded by
    \begin{align*}
        |Q^{\bpi}(\bh,\ba)|\leq 1+2\frac{M\useconstant{constant:go-back-root-length}}{\delta}\eqqcolon C_Q
    \end{align*}
    for any policy $\bpi$,
    where
    \begin{align*}
        \delta=(\frac{\gamma^N}{|\cA|})^{M\useconstant{constant:go-back-root-length}}.
    \end{align*}
\end{lemma}
The proof is postponed to Appendix \ref{sec:MDP-contraction-property}.
By Lemma \ref{lemma:performance-difference-lemma}, we can prove that $\rho^{\bpi}$ is Lipschitz continuous.

\begin{lemma}[Lipschitz Continuity of $\rho^{\bpi}$]
\label{lemma:Lipschitz-MDP-value}
For any two strategy profiles $\bpi_1,\bpi_2$, we have
\begin{align*}
    \abr{\rho^{\bpi_2}-\rho^{\bpi_1}}\leq C_Q\max_{\bh\in\cH_M}\nbr{\pi_2(\cdot\given \bh)-\pi_1(\cdot\given \bh)}_1.
\end{align*}
    \proof
    \begin{align*}
        \abr{\rho^{\bpi_2}-\rho^{\bpi_1}}=&\abr{\sum_{\bh\in\cH_M} d^{\bpi_2}(\bh) \sum_{\ba\in\cA} \pi_2(\ba\given \bh) \big(Q^{\bpi_1}(\bh,\ba)-V^{\bpi_1}(\bh)\big)}\\
        =&\abr{\sum_{\bh\in\cH_M} d^{\bpi_2}(\bh) \sum_{\ba\in\cA} \big(\pi_2(\ba\given \bh)-\pi_1(\ba\given \bh)\big)\big(Q^{\bpi_1}(\bh,\ba)-V^{\bpi_1}(\bh)\big)}\\
        =&\abr{\sum_{\bh\in\cH_M} d^{\bpi_2}(\bh) \sum_{\ba\in\cA} \big(\pi_2(\ba\given \bh)-\pi_1(\ba\given \bh)\big)Q^{\bpi_1}(\bh,\ba)}\\
        \leq& \max_{\bh\in\cH_M}\sum_{\ba\in\cA} \abr{\pi_2(\ba\given \bh)-\pi_1(\ba\given \bh)}C_Q\\
        =&C_Q\max_{\bh\in\cH_M}\nbr{\pi_2(\cdot\given \bh)-\pi_1(\cdot\given \bh)}_1.
    \end{align*}
    The second line is by definition of $V^{\bpi}$ in \Cref{eq:def-V-function}. The third line is by the fact that $\forall \bh\in\cH_M, \sum_{\ba\in\cA}\pi_1(\ba\given \bh)=\sum_{\ba\in\cA}\pi_2(\ba\given \bh)=1$.\qed
\end{lemma}

\begin{lemma}[Lower Bound of $\bq^{\pi}$]
    \label{lemma:lower-bound-q}
    When $\pi\ui(a_i\given \bh)\geq\frac{\gamma}{|\cA_i|}$ for all $i=1,2,...,N$, we have
    \begin{align}
        &\forall \bh\in\cH_M,~~~~~~\sum_{\ba\in\cA}q^{\bpi}(\bh,\ba)\geq \frac{\gamma^{NM}}{|\cA|^M}\\
        &\forall \bh\in\cH_M,\ba\in\cA,~~~~~~q^{\bpi}(\bh,\ba)\geq \frac{\gamma^{N(M+1)}}{|\cA|^{M+1}}.
    \end{align}
    \proof
    Firstly, we have $\pi(\ba\given \bh)\geq\frac{\gamma^N}{|\cA|}$ for any $\ba\in\cA,\bh\in\cH_M$. Then, for any $\bh\in\cH_M,\ba\in\cA$,
    \begin{equation*}
        q^{\bpi}(\bh,\ba)=\bpi(\ba\given \bh)\sum_{\ba'\in\cA} q^{\bpi}(\bh,\ba')=\bpi(\ba\given \bh)d^{\bpi}(\bh)\overset{(i)}{\geq} \frac{\gamma^N}{|\cA|}\frac{\gamma^{NM}}{|\cA|^M}=\frac{\gamma^{N(M+1)}}{|\cA|^{M+1}}
    \end{equation*}
    where $(i)$ is because $\rbr{(\cP^{\bpi})^M}_{\bh_1,\bh_2}\geq \rbr{\frac{\gamma^N}{|\cA|}}^M$, $d^{\bpi}(\bh)=\rbr{d^{\bpi}(\cP^{\bpi})^M}_{\bh}\geq \rbr{\frac{\gamma^N}{|\cA|}}^M\sum_{\bh'\in\cH_M} d^{\bpi}(\bh')= \rbr{\frac{\gamma^N}{|\cA|}}^M$. Here, $\cP^{\bpi}$ is the state transition matrix induced from strategy $\bpi$. Specifically, we have $\cP^{\bpi}_{\bh_1,\bh_2}=\sum_{\ba\in\cA}\pi(\ba\given \bh_1)\Pr(\bh_2\given \bh_1,\ba)$. \qed
\end{lemma}

\subsection{Formal Version and Proof of Theorem  \ref{theorem:occupancy-measure-RM-informal}}
\label{appendix:proof-of-occupancy-measure-theorem}

\begin{theorem}[Formal Version of Theorem \ref{theorem:occupancy-measure-RM-informal}]
\label{theorem:occupancy-measure-RM-formal}
Suppose player $1$ runs Algorithm \ref{alg:self-play}. All players (including the comparator) satisfy Condition \ref{assumption:2-forgetful-prime} with $\bnu\ui$ as the uniform strategy over $\Delta_{|\cA_i|}$, respectively. Then, we have
\begin{align}
    R_T\leq& \cU_1+(C_Q+2C_{\rm Lips}K^2)\frac{|\cA|^{M+1}}{\gamma^{N(M+1)}|\cA_{-1}|}\cU_2+2K+C_{\rm Lips} K^2 \Delta_T+4\rbr{1-\delta}^{\floor{\frac{K}{M\useconstant{constant:go-back-root-length}}}}T,
\end{align}
where 
\begin{align}
    &\delta\coloneqq(\frac{\gamma^N}{|\cA|})^{M\useconstant{constant:go-back-root-length}}\\
    &C_Q\coloneqq 1+2\frac{M\useconstant{constant:go-back-root-length}}{\delta}\\
    &C_{\rm Lips}\coloneqq |\cA|^2\\
    &\Delta_T \coloneqq \sum_{i=2}^N \sum_{t=2}^T \nbr{\bpi\ui_t-\bpi\ui_{t-1}}_\infty+ \max_{\hat \bpi_{1:T}\in \cC_T^{(1)}}\sum_{t=2}^T \nbr{\hat\bpi\uo_t-\hat\bpi\uo_{t-1}}_\infty\\
    &\cU_1 \coloneqq \frac{5}{2}\sqrt{\frac{T}{(C_Q|\cH_M|+1) |\cA| \Delta_T}}+\rbr{5+4|\cA|^{M+1}+8C_{\rm Lips}K^2\frac{|\cA|^{\frac{3}{2}(M+1)}}{\gamma^{NM}}}\sqrt{(C_Q|\cH_M|+1) |\cA|}\sqrt{T\Delta_T}\\
    &\cU_2 \coloneqq  \sqrt{2\rbr{(8C_{\rm Lips}K^2\frac{|\cA|^{\frac{3}{2}(M+1)}}{\gamma^{NM}}+8|\cA|^{M+1}+1)\sqrt{(C_Q|\cH_M|+1)|\cA|T\Delta_T}+\sqrt{\frac{T}{(C_Q|\cH_M|+1)|\cA|\Delta_T}}}}\notag\\
    &\times \sqrt{T+\frac{5}{2}\sqrt{\frac{T}{(C_Q|\cH_M|+1)|\cA|\Delta_T}}+\rbr{5+4|\cA|^{M+1}+8C_{\rm Lips}K^2\frac{|\cA|^{\frac{3}{2}(M+1)}}{\gamma^{NM}}}\sqrt{(C_Q|\cH_M|+1)|\cA|T\Delta_T}}.
\end{align}
\end{theorem}

For any $t\geq K+1$, by the Lipschitz continuity proved in Lemma \ref{lemma:Lipschitz}, we have
\begin{align*}
    \abr{f^{t-1}(\bpi_{1:t})-f^{t-1}(\bpi_{1:t-K-1},\underbrace{\bpi_t,\bpi_t,...,\bpi_t}_{K+1})}\leq C_{\rm Lips}\sum_{s=t-K}^{t-1} \nbr{\bpi_t-\bpi_s}_\infty\leq& C_{\rm Lips}\sum_{i=1}^N\sum_{s=t-K}^{t-1} \nbr{\bpi\ui_t-\bpi\ui_s}_\infty.
\end{align*}
Then, by Corollary \ref{lemma:MDP-average-iterate}, we have
\begin{align*}
    &\abr{f^{t-1}(\bpi_{1:t-K-1},\underbrace{\bpi_t,\bpi_t,...,\bpi_t}_{K+1})-\rho^{\bpi_t}}\\
    =&\abr{\sum_{\bh\in\cH_M} \cL_1(\bh_M) (\mu(\cP^{\bpi_t})^{K+1})_{\bh}-\sum_{\bh\in\cH_M} \cL_1(\bh_M) (\lim_{T\to+\infty} \frac{1}{T}\sum_{s=0}^{T-1} \mu(\cP^{\bpi_t})^s)_{\bh}}\\
    \leq& 2\rbr{1-\delta}^{\floor{\frac{K}{M\useconstant{constant:go-back-root-length}}}}.
\end{align*}

where $\mu\in\Delta_{|\cH_M|}$ is the state distribution induced by $\bpi_{1:t-K-1}$ (which is common for the two terms above).

Therefore, we have 
\begin{align}
    R_T=&\sum_{t=1}^T f^{t-1}(\bpi_{1:t})-\sum_{t=1}^T f^{t-1}(\hat\bpi\uo_{1:t},\bpi\uno_{1:t})\notag\\
    \leq& \sum_{t=1}^T \rho^{\bpi_t}-\sum_{t=1}^T \rho^{(\hat\bpi\uo_t,\bpi\uno_t)}+\sum_{t=1}^{T} \abr{\rho^{\bpi_t}-f^{t-1}(\bpi_{1:t})}+\sum_{t=1}^{T} \abr{\rho^{(\hat\bpi\uo_t,\bpi\uno_t)}-f^{t-1}(\hat\bpi\uo_{1:t},\bpi\uno_{1:t})}\notag\\
    \leq&\sum_{t=1}^T \rho^{\bpi_t}-\sum_{t=1}^T \rho^{(\hat\bpi\uo_t,\bpi\uno_t)}+\sum_{t=1}^{K} \abr{\rho^{\bpi_t}-f^{t-1}(\bpi_{1:t})}+\sum_{t=1}^{K} \abr{\rho^{(\hat\bpi\uo_t,\bpi\uno_t)}-f^{t-1}(\hat\bpi\uo_{1:t},\bpi\uno_{1:t})}\label{eq:MDP-regret-bound-eq1-line1}\\
    &+C_{\rm Lips}\sum_{t=K+1}^T\sum_{s=t-K}^{t-1} \nbr{\bpi_t-\bpi_s}_\infty+C_{\rm Lips}\sum_{t=K+1}^T\sum_{i=2}^N\sum_{s=t-K}^{t-1} \nbr{\bpi\ui_t-\bpi\ui_s}_\infty+C_{\rm Lips}\sum_{t=K+1}^T\sum_{s=t-K}^{t-1} \nbr{\hat\bpi\uo_t-\hat\bpi\uo_s}_\infty\label{eq:MDP-regret-bound-eq1-line2}\\
    &+4\rbr{1-\delta}^{\floor{\frac{K}{M\useconstant{constant:go-back-root-length}}}}T.\label{eq:MDP-regret-bound-eq1-line3}
\end{align}

Since $\max\left\{\abr{\rho^{\bpi_t}-f^{t-1}(\bpi_{1:t})},\abr{\rho^{(\hat\bpi\uo_t,\bpi\uno_t)}-f^{t-1}(\hat\bpi\uo_{1:t},\bpi\uno_{1:t})}\right\}\leq 1$, \Cref{eq:MDP-regret-bound-eq1-line1} can be upper-bounded by
\begin{align*}
    &\sum_{t=1}^T \rho^{\bpi_t}-\sum_{t=1}^T \rho^{(\hat\bpi\uo_t,\bpi\uno_t)}+\sum_{t=1}^{K} \abr{\rho^{\bpi_t}-f^{t-1}(\bpi_{1:t})}+\sum_{t=1}^{K} \abr{\rho^{(\hat\bpi\uo_t,\bpi\uno_t)}-f^{t-1}(\hat\bpi\uo_{1:t},\bpi\uno_{1:t})}\\
    \leq& \sum_{t=1}^T \rho^{\bpi_t}-\sum_{t=1}^T \rho^{(\hat\bpi\uo_t,\bpi\uno_t)}+2K.
\end{align*}
Also, by Lemma \ref{lemma:accumulated-variation-bound}, \Cref{eq:MDP-regret-bound-eq1-line2} can be bounded by
\begin{align*}
    &C_{\rm Lips}\sum_{t=K+1}^T\sum_{s=t-K}^{t-1} \nbr{\bpi_t-\bpi_s}_\infty+C_{\rm Lips}\sum_{t=K+1}^T\sum_{i=2}^N\sum_{s=t-K}^{t-1} \nbr{\bpi\ui_t-\bpi\ui_s}_\infty+C_{\rm Lips}\sum_{t=K+1}^T\sum_{s=t-K}^{t-1} \nbr{\hat\bpi\uo_t-\hat\bpi\uo_s}_\infty\\
    \leq& C_{\rm Lips} K^2\sum_{t=2}^T \nbr{\bpi_t-\bpi_{t-1}}_\infty+C_{\rm Lips} K^2\sum_{t=2}^T\sum_{i=2}^N \nbr{\bpi\ui_t-\bpi\ui_{t-1}}_\infty+C_{\rm Lips} K^2\sum_{t=2}^T \nbr{\hat\bpi\uo_t-\hat\bpi\uo_{t-1}}_\infty\\
    \leq& C_{\rm Lips} K^2\sum_{t=2}^T \nbr{\bpi_t-\bpi_{t-1}}_\infty+C_{\rm Lips} K^2\Delta_T
\end{align*}
where the last line is by the definition of $\Delta_T$. Hence, we have
\begin{align*}
    R_T\leq& \sum_{t=1}^T \rho^{\bpi_t}-\sum_{t=1}^T \rho^{(\hat\bpi\uo_t,\bpi\uno_t)}+2K+C_{\rm Lips} K^2\sum_{t=2}^T \nbr{\bpi_t-\bpi_{t-1}}_\infty+C_{\rm Lips} K^2 \Delta_T+4\rbr{1-\delta}^{\floor{\frac{K}{M\useconstant{constant:go-back-root-length}}}}T.
\end{align*}

Recall that we defined the marginals $\bpi_t^{[i]}$ in \Cref{eq:def-self-play-strategy-i}. Let $\bpi_t^{[\cdot]}$ denote the joint strategy induced by $\bq_t$, \emph{i.e.},
\begin{align*}
    \pi_t^{[\cdot]}(\ba\given\bh)\coloneqq \frac{q_t(\bh,\ba)}{\sum_{\ba'\in\cA}q_t(\bh,\ba')}.
\end{align*}
Its player-$i$ marginal is $\bpi_t^{[i]}$, and in particular $\bpi_t^{[1]}=\bpi_t\uo$. Then, by Lemma \ref{lemma:Lipschitz-MDP-value}, we have
\begin{align*}
    &\sum_{t=1}^T \rho^{\bpi_t}-\sum_{t=1}^T \rho^{(\hat\bpi\uo_t,\bpi\uno_t)}\\
    \leq& \sum_{t=1}^T \rho^{\bpi_t^{[\cdot]}}-\sum_{t=1}^T \rho^{(\hat\bpi\uo_t,\bpi\uno_t)}+C_Q\sum_{t=1}^T\max_{\bh\in\cH_M} \nbr{\bpi_t(\cdot\given \bh)-\bpi_t^{[\cdot]}(\cdot\given \bh)}_1\\
    =& \sum_{t=1}^T \rho^{\bpi_t^{[\cdot]}}-\sum_{t=1}^T \rho^{(\hat\bpi\uo_t,\bpi\uno_t)}+C_Q\sum_{t=1}^T\max_{\bh\in\cH_M} \sum_{\ba\in\cA}\abr{\bpi_t\uo(a_1\given \bh)\bpi_t\uno(\ba_{-1}\given \bh)-\bpi_t\uo(a_1\given \bh)\frac{\bpi_t^{[\cdot]}(\ba\given \bh)}{\bpi_t\uo(a_1\given \bh)}}\\
    =& \sum_{t=1}^T \rho^{\bpi_t^{[\cdot]}}-\sum_{t=1}^T \rho^{(\hat\bpi\uo_t,\bpi\uno_t)}+C_Q\sum_{t=1}^T\max_{\bh\in\cH_M} \sum_{\ba\in\cA}\bpi_t\uo(a_1\given \bh)\abr{\bpi_t\uno(\ba_{-1}\given \bh)-\frac{\bpi_t^{[\cdot]}(\ba\given \bh)}{\bpi_t\uo(a_1\given \bh)}}
\end{align*}
The difference between $\bpi_t^{[i]}$ and $\bpi_t\ui$ is that $\bpi_t^{[i]}$ is the marginal of the occupancy measure $\bq_t$, while $\bpi_t\ui$ is the true strategy of player $i$.

\begin{lemma}
\label{lemma:sublinear-variation-q}
For any $\bpi_1,\bpi_2$, if $~\forall~\bh\in\cH_M,~~\sum_{\ba\in\cA} q^{\bpi_1}(\bh,\ba),~\sum_{\ba\in\cA} q^{\bpi_2}(\bh,\ba)\geq c$ for some constant $c>0$, then we have
\begin{align}
    \forall~\bh\in\cH_M,\ba\in\cA,~~~~|\pi_{2}(\ba\given \bh)-\pi_1(\ba\given \bh)|\leq \frac{2|\cA|}{c}\nbr{\bq^{\pi_2}-\bq^{\pi_1}}_\infty.
\end{align}
\end{lemma}
The proof is postponed to Appendix \ref{appendix:original-to-markov-value-proof}. Since $\bq_t\in\cX$, the lower-bound constraint in \eqref{eq:def-convex-set-self-play} gives $\sum_{\ba\in\cA} q_t(\bh,\ba)\geq \frac{\gamma^{NM}}{|\cA|^M}$ for any $\bh\in\cH_M$. Therefore, for any $t=2,3,...,T$,
\begin{align*}
    &C_{\rm Lips}K^2\sum_{t=2}^T \nbr{\bpi_t-\bpi_{t-1}}_\infty\\
    \leq& C_{\rm Lips} K^2\sum_{t=2}^T \nbr{\bpi_t^{[\cdot]}-\bpi_{t-1}^{[\cdot]}}_\infty+2C_{\rm Lips} K^2\sum_{t=1}^T \nbr{\bpi_t^{[\cdot]}-\bpi_t}_\infty\\
    \leq& \frac{2 C_{\rm Lips} K^2|\cA|^{M+1}}{\gamma^{NM}}\sum_{t=2}^T\nbr{\bq_t-\bq_{t-1}}_\infty+ 2C_{\rm Lips} K^2\sum_{t=1}^T \max_{\bh\in\cH_M} \sum_{\ba\in\cA}\bpi_t\uo(a_1\given \bh)\abr{\bpi_t\uno(\ba_{-1}\given \bh)-\frac{\bpi_t^{[\cdot]}(\ba\given \bh)}{\bpi_t\uo(a_1\given \bh)}}.
\end{align*}

For each timestep $t$, define
\begin{align*}
    \sigma_t(\bh,\ba)\coloneqq
    \begin{cases}
    1, & D_t(\bh,\ba,\bq_t)\geq 0,\\
    -1, & D_t(\bh,\ba,\bq_t)<0.
    \end{cases}
\end{align*}
The constraint in Algorithm \ref{alg:self-play} is equivalently
\begin{align*}
    g_t(\bq)=\sum_{\bh\in\cH_M,\ba\in\cA}\sigma_t(\bh,\ba)D_t(\bh,\ba,\bq),
\end{align*}
where
\begin{align*}
    D_t(\bh,\ba,\bq)
    =q(\bh,\ba)-\pi_t^{(-1)}(\ba_{-1}\given \bh)\sum_{\ba_{-1}'\in\cA_{-1}}q(\bh,(a_1,\ba_{-1}')).
\end{align*}
In particular,
\begin{align*}
    g_t(\bq_t)=\sum_{\bh\in\cH_M,\ba\in\cA}|D_t(\bh,\ba,\bq_t)|.
\end{align*}
Notice that
\begin{align*}
    \frac{q_t(\bh,\ba)}{\sum_{\ba_{-1}'\in\cA_{-1}} q_t(\bh,(a_1,\ba_{-1}'))}
    =\frac{q_t(\bh,\ba)}{\sum_{\ba'\in\cA} q_t(\bh,\ba')}
    \frac{\sum_{\ba'\in\cA}q_t(\bh,\ba')}{\sum_{\ba_{-1}'\in\cA_{-1}} q_t(\bh,(a_1,\ba_{-1}'))}
    =\pi_t^{[\cdot]}(\ba\given \bh)\cdot \frac{1}{\pi_t\uo(a_1\given \bh)}.
\end{align*}
Here $\pi_t^{[\cdot]}$ denotes the joint distribution induced by $\bq_t$, and $\bpi_t^{[1]}=\bpi_t\uo$ is its player-$1$ marginal. Therefore,
\begin{align*}
    \abr{\bpi_t\uno(\ba_{-1}\given \bh)-\frac{\bpi_t^{[\cdot]}(\ba\given \bh)}{\bpi_t\uo(a_1\given \bh)}}
    &=\frac{|D_t(\bh,\ba,\bq_t)|}{\sum_{\ba_{-1}'\in\cA_{-1}}q_t(\bh,(a_1,\ba_{-1}'))}\\
    &\leq \frac{|\cA|^{M+1}}{\gamma^{N(M+1)}|\cA_{-1}|}|D_t(\bh,\ba,\bq_t)|,
\end{align*}
where the last inequality is by Lemma \ref{lemma:lower-bound-q}.
Then, we have
\begin{align*}
    R_T\leq&\sum_{t=1}^T \rho^{\bpi_t^{[\cdot]}}-\sum_{t=1}^T \rho^{(\hat\bpi\uo_t,\bpi\uno_t)}+2K\\
    &+\frac{2 C_{\rm Lips} K^2|\cA|^{M+1}}{\gamma^{NM}}\sum_{t=2}^T\nbr{\bq_t-\bq_{t-1}}_\infty+C_{\rm Lips} K^2 \Delta_T+4\rbr{1-\delta}^{\floor{\frac{K}{M\useconstant{constant:go-back-root-length}}}}T\\
    &+(C_Q+2C_{\rm Lips} K^2)\frac{|\cA|^{M+1}}{\gamma^{N(M+1)}|\cA_{-1}|}
      \sum_{t=1}^T\max_{\bh\in\cH_M} \sum_{\ba\in\cA}\bpi_t\uo(a_1\given \bh) D_t(\bh,\ba,\bq_t)\\
    \leq&\sum_{t=1}^T \rho^{\bpi_t^{[\cdot]}}-\sum_{t=1}^T \rho^{(\hat\bpi\uo_t,\bpi\uno_t)}+2K\\
    &+\frac{2 C_{\rm Lips} K^2|\cA|^{M+1}}{\gamma^{NM}}\sum_{t=2}^T\nbr{\bq_t-\bq_{t-1}}_\infty+C_{\rm Lips} K^2 \Delta_T+4\rbr{1-\delta}^{\floor{\frac{K}{M\useconstant{constant:go-back-root-length}}}}T\\
    &+(C_Q+2C_{\rm Lips}K^2)\frac{|\cA|^{M+1}}{\gamma^{N(M+1)}}\sum_{t=1}^T\max_{\bh\in\cH_M,\ba\in\cA} \cbr{D_t(\bh,\ba)}.
\end{align*}

For the comparator sequence, define the joint strategy $\hat\pi_t^{[\cdot]}(\ba\given\bh)\coloneqq \hat\pi_t\uo(a_1\given\bh)\pi_t\uno(\ba_{-1}\given\bh)$. By Lemma \ref{lemma:sublinear-variation-pi}, we have
\begin{align*}
    \sum_{t=2}^T \nbr{\hat\bq_t-\hat\bq_{t-1}}\leq& (C_Q|\cH_M|+1) |\cA|\sum_{t=2}^T \nbr{\hat\bpi^{[\cdot]}_{t}-\hat\bpi^{[\cdot]}_{t-1}}_\infty\\
    =& (C_Q|\cH_M|+1) |\cA|\sum_{t=2}^T \max_{\bh\in\cH_M,\ba\in\cA} \abr{\hat \pi_t\uo(a_1\given \bh)\prod_{i=2}^N \pi_t\ui(a_i\given \bh)-\hat \pi_{t-1}\uo(a_1\given \bh)\prod_{i=2}^N \pi_{t-1}\ui(a_i\given \bh)}\\
    \leq&(C_Q|\cH_M|+1) |\cA|\sum_{t=2}^T \max_{\bh\in\cH_M,\ba\in\cA} \cbr{\abr{\hat \pi_t\uo(a_1\given \bh)-\hat \pi_{t-1}\uo(a_1\given \bh)}+\sum_{i=2}^N \abr{ \pi_t\ui(a_i\given \bh)-\pi_{t-1}\ui(a_i\given \bh)}}\\
    \leq& (C_Q|\cH_M|+1)|\cA|\rbr{\sum_{t=2}^T\nbr{\hat\bpi_t\uo-\hat\bpi_{t-1}\uo}_\infty+\sum_{i=2}^N\sum_{t=2}^T\nbr{\bpi_t\ui-\bpi_{t-1}\ui}_\infty}\\
    \leq& (C_Q|\cH_M|+1)|\cA|\Delta_T.
\end{align*}
Then, by the convergence of Algorithm \ref{algorithm:PGD-with-constraint} in Lemma \ref{lemma:PGD-constraint-guarantee} and the variation of $\bq$ bounded by variation of $\bpi$ in Lemma \ref{lemma:sublinear-variation-pi}, we have
\begin{align}
    &\sum_{t=1}^T \rho^{\bpi_t^{[\cdot]}}-\sum_{t=1}^T \rho^{(\hat\bpi\uo_t,\bpi\uno_t)}+\frac{2 C_{\rm Lips} K^2|\cA|^{M+1}}{\gamma^{NM}}\sum_{t=2}^T \nbr{\bq_t-\bq_{t-1}}_\infty\notag\\
    \leq& \frac{5}{2}\sqrt{\frac{T}{(C_Q|\cH_M|+1) |\cA| \Delta_T}}+\rbr{5+4|\cA|^{M+1}+8C_{\rm Lips}K^2\frac{|\cA|^{\frac{3}{2}(M+1)}}{\gamma^{NM}}}\sqrt{(C_Q|\cH_M|+1) |\cA|}\sqrt{T\Delta_T}\label{eq:regret-bound-OCO-with-constraint}\\
    &\sum_{t=1}^T g_t(\bq_t)\leq \sqrt{2\rbr{(8C_{\rm Lips}K^2\frac{|\cA|^{\frac{3}{2}(M+1)}}{\gamma^{NM}}+8|\cA|^{M+1}+1)\sqrt{(C_Q|\cH_M|+1)|\cA|T\Delta_T}+\sqrt{\frac{T}{(C_Q|\cH_M|+1)|\cA|\Delta_T}}}}\notag\\
    &\times \sqrt{T+\frac{5}{2}\sqrt{\frac{T}{(C_Q|\cH_M|+1)|\cA|\Delta_T}}+\rbr{5+4|\cA|^{M+1}+8C_{\rm Lips}K^2\frac{|\cA|^{\frac{3}{2}(M+1)}}{\gamma^{NM}}}\sqrt{(C_Q|\cH_M|+1)|\cA|T\Delta_T}}\label{eq:constraint-bound-OCO-with-constraint}
\end{align}
since 
\begin{align*}
    &R=\max_{\bq} \nbr{\bq}_2=1,\\
    &k=1,\\
    &D=\max_{t,\bq} g_t(\bq)\leq 2,\\
    &G=\max\{\nbr{\cL_1}, \nbr{\nabla g_t}\}\leq 2\sqrt{|\cH_M|\cdot |\cA|}=2|\cA|^{\frac{M+1}{2}},\\
    &F=1,\\
    &C=2C_{\rm Lips}K^2\frac{|\cA|^{M+1}}{\gamma^{NM}}.
\end{align*}
For ease of illustration, let $\cU_1$ be the R.H.S. of Eq.  \eqref{eq:regret-bound-OCO-with-constraint} and $\cU_2$ be the R.H.S. of Eq.  \eqref{eq:constraint-bound-OCO-with-constraint}. This leads to the fact that 
\begin{align*}
    &\sum_{t=1}^T \rho^{\bpi_t^{[\cdot]}}-\sum_{t=1}^T \rho^{(\hat\bpi\uo_t,\bpi\uno_t)}+\frac{2 C_{\rm Lips} K^2|\cA|^{M+1}}{\gamma^{NM}}\sum_{t=2}^T \nbr{\bq_t-\bq_{t-1}}_\infty\leq \cU_1,\\
    &\sum_{t=1}^T g_t(\bq_t)\leq \cU_2.
\end{align*}

Let $\hat\bq_t$ be the occupancy measure induced by $(\hat\bpi_t\uo,\bpi_t\uno)$. When $\hat\bq_t$ lies in $\cX$, it corresponds to a product-form strategy profile and, for any $\bh\in\cH_M,\ba\in\cA$,
\begin{align*}
    \frac{\hat q_t(\bh,\ba)}{\sum_{\ba'\in\cA} \hat q_t(\bh,\ba')}
    =\hat\pi_t\uo(a_1\given \bh)\pi_t\uno(\ba_{-1}\given \bh).
\end{align*}
Equivalently, $D_t(\bh,\ba,\hat\bq_t)=0$ for all $\bh,\ba$, and hence $g_t(\hat\bq_t)=0\leq 0$. Therefore, $\hat\bq_{1:T}$ satisfies the feasibility requirement of Lemma \ref{lemma:PGD-constraint-guarantee}.

Therefore,
\begin{align*}
    R_T\leq& \cU_1+2K+C_{\rm Lips} K^2 \Delta_T+4\rbr{1-\delta}^{\floor{\frac{K}{M\useconstant{constant:go-back-root-length}}}}T+(C_Q+2C_{\rm Lips}K^2)\frac{|\cA|^{M+1}}{\gamma^{N(M+1)}|\cA_{-1}|}\sum_{t=1}^T g_t(\bq_t)\\
    \leq& \cU_1+(C_Q+2C_{\rm Lips}K^2)\frac{|\cA|^{M+1}}{\gamma^{N(M+1)}|\cA_{-1}|}\cU_2+2K+C_{\rm Lips} K^2 \Delta_T+4\rbr{1-\delta}^{\floor{\frac{K}{M\useconstant{constant:go-back-root-length}}}}T.\qedhere
\end{align*}

\begin{lemma}
\label{lemma:sublinear-variation-pi}
For any two strategy profiles $\bpi_1,\bpi_2$, the distance between their corresponding occupancy measure satisfies
\begin{align}
    \nbr{\bq^{\pi_1}-\bq^{\pi_2}}\leq (C_Q|\cH_M|+1) |\cA|\cdot \nbr{\bpi_1-\bpi_2}_\infty.
\end{align}
\proof
By letting the loss $\cL(\bh,\ba)=\ind(\bh=\bh_0)$, $\rho^{\bpi}=d^{\bpi}(\bh_0)$. Therefore, by Lemma \ref{lemma:Lipschitz-MDP-value}, for any $\bh_0\in\cH_M$, we have
\begin{align*}
    |d^{\bpi_1}(\bh_0)-d^{\bpi_2}(\bh_0)|\leq C_Q \max_{\bh\in\cH_M} \nbr{\bpi_2(\cdot\given \bh)-\bpi_1(\cdot\given \bh)}_1.
\end{align*}
Therefore, for any $\bh_0\in\cH_M,\ba\in\cA$, we have
\begin{align*}
    |q^{\bpi_1}(\bh_0,\ba)-q^{\bpi_2}(\bh_0,\ba)|=&|d^{\bpi_1}(\bh_0)\pi_1(\ba\given \bh_0)-d^{\bpi_2}(\bh_0)\pi_2(\ba\given \bh_0)|\\
    \leq& |d^{\bpi_1}(\bh_0)-d^{\bpi_2}(\bh_0)|\pi_1(\ba\given \bh_0)+d^{\bpi_2}(\bh_0)|\pi_1(\ba\given \bh_0)-\pi_2(\ba\given \bh_0)|.
\end{align*}
Therefore, 
\begin{align*}
    \nbr{\bq^{\pi_1}-\bq^{\pi_2}}\leq& \sum_{\bh\in\cH_M,\ba\in\cA} |q^{\bpi_1}(\bh,\ba)-q^{\bpi_2}(\bh,\ba)|\\
    \leq& \sum_{\bh\in\cH_M} |d^{\bpi_1}(\bh)-d^{\bpi_2}(\bh)|+|\cA| \nbr{\bpi_1-\bpi_2}_\infty\\
    \leq& C_Q|\cH_M| \max_{\bh\in\cH_M} \nbr{\bpi_2(\cdot\given \bh)-\bpi_1(\cdot\given \bh)}_1+|\cA| \nbr{\bpi_1-\bpi_2}_\infty\\
    \leq& (C_Q|\cH_M|+1) |\cA|\cdot \nbr{\bpi_1-\bpi_2}_\infty.\qedhere
\end{align*}
\end{lemma}

\section{Regret and Subgame Perfect Equilibrium}
\label{appendix:subgame-perfect-equilibrium-local}

\subsection{{\tt LRP-Regret} and Subgame Perfect Equilibrium}
\label{appendix:LPR-Regret-SPE}

In this section, we will prove that when all players get a sublinear $R_T^{\rm local}$ even when the comparator can vary arbitrarily with  $P_T \coloneqq \sum_{t=2}^T \nbr{\hat\bpi\uo_t-\hat\bpi\uo_{t-1}}=O(T)$ and $\bpi\ui\succeq \frac{\gamma}{|\cA_i|}\one$ for $\gamma\in(0,1]$\footnote{Our regret minimization algorithm guarantees sublinear regret when $P_T=o(T)$ in the adversarial online setting. It is unknown whether we can still achieve sublinear regret when $P_T=O(T)$ in the game setting, where we can control all players.}, the equilibrium is actually an approximate SPNE.

In the following, we will only prove that player 1 cannot decrease their loss much by deviating, by the symmetry of players. For any $\bh\in\cH_M$ and $a_1\in\cA_1$, define
\begin{align*}
    &Q_t(\bh,a_1) \coloneqq \sum_{\ba_{-1}\in\cA_{-1}} \bpi_t^{(-1)}(\ba_{-1}\given \bh) \rbr{V_{t+1}((\bh_{2:M},(a_1,\ba_{-1})))+\cL_1((a_1,\ba_{-1}))},\\
    &V_t(\bh) \coloneqq \sum_{\ba\in\cA} \bpi_t(\ba\given \bh) \rbr{V_{t+1}((\bh_{2:M},\ba))+\cL_1(\ba)}.
\end{align*}
and $V_{T_0+1}(\bh)=0$ for all $\bh\in\cH_M$.

Then we have
\begin{lemma}
    \label{lemma:local-regret-to-difference}
    For a fixed finite $T$, when $\bpi\ui\succeq \frac{\gamma}{|\cA_i|}$ for any player $i$, we have
    \begin{align*}
        \sum_{t=1}^T \max\cbr{\max_{\bh\in\cH_M} \inner{\bpi_t\uo(\cdot\given \bh)-\hat\bpi_t\uo(\cdot\given \bh)}{Q_t(\bh,\cdot)},0}\leq \frac{|\cA|^M}{\gamma^{NM}}R_T^{\rm local}
    \end{align*}
    for any $\hat\bpi_{1:T}\uo$.
\end{lemma}
The proof is postponed to the end of this section.

Note that $R_{T}^{\rm local}\leq o(T)$ is the regret for player 1. The relation above can be extended to any player with their corresponding local regret on the right-hand side. With the lemma above, we can conclude subgame perfection.

For notational simplicity, we define $f_i^m(\bpi|\mu) \coloneqq \sum_{\bh\in\cH}\mu(\bh)f_i^m(\bpi_{1:m+1}\given \bh)$ as a generalization of $f_i^m(\bpi_{1:m+1}\given \bh)$. For an infinitely repeated matrix game as in Lemma \ref{lemma:LRP-Regret-SPE}, we can divide the $T$ ($T\to\infty$) timesteps into $K$ epochs, with $T_0$ timesteps in each epoch.

\begin{lemma}
\label{lemma:finite-to-infinite}
For any sequence of strategies $\bpi_1,\bpi_2,...,\bpi_K$, which form an approximate equilibrium for a $K$-repeated matrix game satisfying Condition \ref{assumption:L1-forgetful}\footnote{The joint-strategy version of Condition \ref{assumption:L1-forgetful} is $\forall \bar\bh,\tilde\bh\in\cH,~~~~~\frac{1}{2}\nbr{\bpi_k(\cdot\given\bar\bh)-\bpi_k(\cdot|\tilde\bh)}_1\leq 1-\gamma$.} with a bounded memory $M$. Then, it is also an approximate equilibrium for the infinitely repeated game. Formally, for any player $i$, we have
\begin{align}
    &\lim_{T\to\infty} \sup \frac{1}{T}\sum_{B=0}^{T-1} \rbr{\frac{1}{K}\sum_{k=1}^K \inner{\mu_B\prod_{s=1}^k\cP^{\bpi_{s+BK}}}{\cL_i}-\frac{1}{K}\sum_{k=1}^K \inner{\hat\mu_B\prod_{s=1}^k\cP^{(\hat\bpi_{s+BK}\ui,\bpi_{s+BK}\uni)}}{\cL_i}}\notag\\
    \leq& \frac{\sum_{B=0}^{T-1} \epsilon_B}{T} + \frac{4M \useconstant{constant:go-back-root-length}}{K(\gamma^N/|\cA|)^{M\useconstant{constant:go-back-root-length}}}
\end{align}
where
\begin{align}
    \epsilon_B \coloneqq \frac{1}{K}\sum_{k=1}^K \inner{\mu_0\prod_{s=1}^k\cP^{\bpi_s}}{\cL_i}-\frac{1}{K}\sum_{k=1}^K \inner{\mu_0\prod_{s=1}^k\cP^{(\hat\bpi_{s+BK}\ui,\bpi_s\uni)}}{\cL_i}
\end{align}
and $\mu_B$ is the initial distribution over state space $\cH_M$ at the start of epoch $B$ ($\mu_0$ is an arbitrary distribution predetermined by the game and $\mu_B(B>0)$ is determined by $\mu_0$ and $\bpi_{1:BK}$).
\end{lemma}
The proof is deferred to Appendix \ref{appendix:fast-mixing}.

Since when $\bpi\ui\succeq\frac{\gamma}{|\cA_i|}$ we have $\useconstant{constant:go-back-root-length}=1$, by Lemma \ref{lemma:finite-to-infinite},
\begin{align*}
    &\frac{1}{T-t_0+1}\sum_{t=t_0}^T \rbr{f_1^{t-t_0}(\bpi_{t_0:t}\given \bh_0)-f_1^{t-t_0}(\hat\bpi\uo_{t_0:t},\bpi\uno_{t_0:t}\given \bh_0)}\\
    =& \frac{1}{T-t_0+1}\sum_{t=t_0}^{\ceil{t_0/T_0}T_0}\rbr{f_1^{t-t_0}(\bpi_{t_0:t}\given \bh_0)-f_1^{t-t_0}(\hat\bpi\uo_{t_0:t},\bpi\uno_{t_0:t}\given \bh_0)}\\
    &+\frac{1}{T-t_0+1}\sum_{t=\ceil{t_0/T_0}T_0+1}^T \rbr{f_1^{t-t_0}(\bpi_{t_0:t}|\mu_0)-f_1^{t-t_0}(\hat\bpi\uo_{t_0:t},\bpi\uno_{t_0:t}|\mu_0)}\\
    \overset{(i)}{=}&\frac{1}{T-t_0+1}\sum_{t=\ceil{t_0/T_0}T_0+1}^T \rbr{f_1^{t-t_0}(\bpi_{t_0:t}|\mu_0)-f_1^{t-t_0}(\hat\bpi\uo_{t_0:t},\bpi\uno_{t_0:t}|\mu_0)}\\
    \overset{(ii)}{\leq}&\frac{\sum_{B=\ceil{t_0/T_0}+1}^{T/T_0} \epsilon_B}{T/T_0}+\frac{4M|\cA|^M}{T_0\gamma^{NM}}\overset{(iii)}{\leq}\frac{|\cA|^M}{\gamma^{NM}}\frac{R_{T_0}^{\rm local}}{T_0}+\frac{4M|\cA|^M}{T_0\gamma^{NM}}.
\end{align*}
where $(i)$ is by $T\to\infty$ and $(ii)$ is by Lemma \ref{lemma:finite-to-infinite} and $T\to\infty$. $\mu_0$ is the history distribution at timestep $\ceil{t_0/T_0}T_0+1$ given $\bh_0$ at timestep $t_0$.
$(iii)$ is because for any $B=0,1,2,...,T/T_0$, by Lemma \ref{lemma:performance-difference-lemma}, we have
\begin{align*}
    \epsilon_B=&\frac{1}{T_0}\sum_{t=1}^{T_0} \rbr{f_1^{t-1}(\bpi_{1:t}|\mu_0)-f_1^{t-1}(\hat\bpi\uo_{1:t},\bpi\uno_{1:t}|\mu_0)}\\
    =&\frac{1}{T_0}\sum_{t=1}^{T_0} \sum_{\bh_0\in\cH_M}\mu_0(\bh_0) \sum_{\bh\in\cH_{t-1}}\Pr(\bh\given \bh_0;\hat\bpi\uo_{1:t-1},\bpi\uno_{1:t-1})\\
    &\cdot\inner{\bpi_t\uo(\cdot\given \bh_{t-M:t-1})-\hat\bpi_t\uo(\cdot\given \bh_{t-M:t-1})}{Q_t(\bh_{t-M:t-1},\cdot)}\\
    \leq& \frac{1}{T_0}\sum_{t=1}^{T_0}\sum_{\bh_0\in\cH_M}\mu_0(\bh_0)\sum_{\bh\in\cH_{t-1}} \Pr(\bh\given \bh_0;\hat\bpi\uo_{1:t-1},\bpi\uno_{1:t-1})\max\cbr{\max_{\bh'\in\cH_M} \inner{\bpi_t\uo(\cdot\given \bh')-\hat\bpi_t\uo(\cdot\given \bh')}{Q_t(\bh',\cdot)},0}\\
    \leq&\frac{|\cA|^M}{\gamma^{NM}}\frac{R_{T_0}^{\rm local}}{T_0}.
\end{align*}
In the last line, we use Lemma \ref{lemma:local-regret-to-difference}.\qed

\begin{proof}[Proof of Lemma \ref{lemma:local-regret-to-difference}]
Firstly, for any $\hat\bpi\uo_{1:T}$, we can pick a proxy strategy as $\underline\bpi\uo_{1:T}$ so that $\underline\pi_t\uo(\cdot\given \bh)=\hat\pi\uo_t(\cdot\given \bh)$ if and only if (when multiple $\bh$ satisfy this, we can arbitrarily pick one)
\begin{align*}
    \inner{\pi\uo_t(\cdot\given \bh)-\hat\pi\uo_t(\cdot\given \bh)}{Q_t(\bh,\cdot)}=\max\cbr{\max_{\bh'\in\cH_M} \inner{\pi\uo_t(\cdot\given \bh')-\hat\pi\uo_t(\cdot\given \bh')}{Q_t(\bh',\cdot)},0}.
\end{align*}
Otherwise, we have $\underline\pi_t\uo(\cdot\given \bh)=\pi\uo_t(\cdot\given \bh)$\footnote{The reason that Lemma \ref{lemma:LRP-Regret-SPE} only holds when $P_T=O(T)$ is that the variation of $\underline\bpi_{1:T}$ may be linear in $T$ by the construction described above.}. Therefore, by Lemma \ref{lemma:performance-difference-lemma}, we have
\begin{align*}
    &J_T(\bpi_{1:T}\uo, \bpi_{1:T}\uno)-J_T(\tilde\bpi_{1:T}^{s,(1)}, \bpi_{1:T}\uno)\\
    =&\sum_{t=1}^T \sum_{\bh\in\cH_{t-1}} \Pr(\bh;\tilde\bpi^{s,(1)}_{1:t-1},\bpi\uno_{1:t-1})\inner{\bpi_t\uo(\cdot\given \bh_{t-M:t-1})-\tilde\bpi_t^{s,(1)}(\cdot\given \bh_{t-M:t-1})}{Q_t(\bh_{t-M:t-1},\cdot)}\\
    =&\sum_{\bh\in\cH_{s-1}} \Pr(\bh;\bpi\uo_{1:s-1},\bpi\uno_{1:s-1})\inner{\bpi_s\uo(\cdot\given \bh_{s-M:s-1})-\underline\bpi_s\uo(\cdot\given \bh_{s-M:s-1})}{Q_s(\bh_{s-M:s-1},\cdot)}\\
    \geq& \frac{\gamma^{NM}}{|\cA|^M} \max\cbr{\max_{\bh'\in\cH_M} \inner{\pi\uo_s(\cdot\given \bh')-\hat\pi\uo_s(\cdot\given \bh')}{Q_s(\bh',\cdot)},0}.
\end{align*}
where $\tilde\bpi_t^{s,(1)}=\underline\bpi_t\uo$ only when $t=s$ and $\tilde\bpi_t^{s,(1)}=\bpi_t\uo$ otherwise.
The last line is because the probability of each $\bh\in\cH_M$ occurs with probability no less than $\frac{\gamma^{NM}}{|\cA|^M}$. Therefore, we prove that
\begin{align*}
    \sum_{t=1}^T \max\cbr{\max_{\bh'\in\cH_M} \inner{\pi\uo_t(\cdot\given \bh')-\hat\pi\uo_t(\cdot\given \bh')}{Q_t(\bh',\cdot)},0}\leq& \frac{|\cA|^M}{\gamma^{NM}}\sum_{s=1}^T\rbr{J_T(\bpi_{1:T}\uo, \bpi_{1:T}\uno)-J_T(\tilde\bpi_{1:T}^{s,(1)}, \bpi_{1:T}\uno)}\\
    \leq& \frac{|\cA|^M}{\gamma^{NM}}R_T^{\rm local}.\qedhere
\end{align*}

\end{proof}

\subsection{{\tt RP-Regret} and SPNE}

\label{appendix:RP-Regret-to-SPE}

From the discussion above, since $\bpi\ui\succeq\frac{\gamma}{|\cA_i|}$\footnote{In fact, here we do not necessarily require this. Satisfying Condition \ref{assumption:L1-forgetful} is enough.}, by Lemma \ref{lemma:finite-to-infinite}, we have
\begin{align*}
    &\frac{1}{T-t_0+1}\sum_{t=t_0}^T \rbr{f_1^{t-t_0}(\bpi_{t_0:t}\given \bh_0)-f_1^{t-t_0}(\hat\bpi\uo_{t_0:t},\bpi\uno_{t_0:t}\given \bh_0)}\\
    \leq&\frac{\sum_{B=\ceil{t_0/T_0}+1}^{T/T_0} \epsilon_B}{T/T_0}+\frac{4M|\cA|^M}{T_0\gamma^{NM}}
\end{align*}
by the discussion in Appendix \ref{appendix:LPR-Regret-SPE}. Then,
\begin{align*}
    \epsilon_B \coloneqq &\frac{1}{T_0}\sum_{t=1}^{T_0} \rbr{f_1^{t-1}(\bpi_{1:t}|\mu_0)-f_1^{t-1}(\hat\bpi\uo_{1:t},\bpi\uno_{1:t}|\mu_0)}\\
    \overset{(i)}{\leq}&\frac{R_{T_0}}{T_0}\overset{(ii)}{\leq}\frac{O(T_0^p (P^B_{T_0})^q)}{T_0}~~~~~~~~~~(p+q=1,p<1).
\end{align*}
where $(i)$ is by definition of {\tt RP-Regret} and $(ii)$ is the condition of Lemma \ref{lemma:RP-Regret-SPE} (since {\tt RP-Regret} is sublinear such $p,q$ must exist). Note that we define $P_{T_0}^B=\sum_{t=BT_0+2}^{(B+1)T_0} \nbr{\hat\bpi\uo_t-\hat\bpi\uo_{t-1}}_\infty$. So,
\begin{align*}
    &\frac{1}{T-t_0+1}\sum_{t=t_0}^T \rbr{f_1^{t-t_0}(\bpi_{t_0:t}\given \bh_0)-f_1^{t-t_0}(\hat\bpi\uo_{t_0:t},\bpi\uno_{t_0:t}\given \bh_0)}\\
    \leq&\frac{\sum_{B=\ceil{t_0/T_0}+1}^{T/T_0} O(T_0^p (P_{T_0}^B)^q)}{T}+\frac{4M|\cA|^M}{T_0\gamma^{NM}}\\
    \leq& O\rbr{\rbr{\frac{P_T}{T}}^q}+\frac{4M|\cA|^M}{T_0\gamma^{NM}}.
\end{align*}
The last line is due to the concavity of function $x^q(q<1)$. So, $\frac{1}{N}\sum_{i=1}^N x_i^q\leq \rbr{\frac{\sum_{i=1}^N x_i}{N}}^q$.\qed

\section{Finding Subgame Perfect Coarse Correlated Equilibrium in Repeated Games}
\label{appendix:SPCCE}

\subsection{Computation of SPCCE}
\label{appendix:computation-SPCCE}

\begin{lemma}
\label{lemma:convergence-CCE-finite-MDP}
Algorithm \ref{alg:self-play-finite-MDP} will guarantee the following upper-bound on regret ($e$ is the Euler's number) for any player $i$ and comparator $\hat\bpi\ui$ satisfying Condition \ref{assumption:2-forgetful-prime} with $\bnu\ui$ as the uniform strategy over $\Delta_{|\cA_i|}$,
\begin{align*}
    \max_{i=1,2,...,N}\max_{\hat\bpi\ui}\sum_{t=1}^T \rbr{V^{\bar\bpi_t}(\bh^1)-V^{\hat\bpi\ui,\bar\bpi\uni_t}(\bh^1)}\leq& 8eK^2 \sqrt{\max_{i=1,2,...,N}|\cA_i|K T},
\end{align*}
where $K$ is the horizon of the Markov game.
\end{lemma}
We defer the proof to the latter part of this section. %

By Lemma \ref{lemma:convergence-CCE-finite-MDP}, $\max_{i=1,2,...,N}\frac{1}{T_0}\sum_{t=1}^{T_0} \rbr{V^{\bar\bpi_t}(\bh^1)-V^{\hat\bpi\ui,\bar\bpi\uni_t}(\bh^1)} \leq 8eK^2 \sqrt{\frac{\max_{i=1,2,...,N}|\cA_i|K}{T_0}}$  ($e$ is the Euler's number) in $T_0$ iterations when all players apply Algorithm \ref{alg:self-play-finite-MDP}. Then, we will first sample a strategy $\bpi_{t_0}$ uniformly from $1,2,...,T_0$, for any horizon $k$ (possibly larger than $K$), we choose $\bpi_{t_0}(\given \bh^{(k-1)\% K+1})$ as the strategy for any $\bh\in\cH_M$. Therefore, by Lemma \ref{lemma:finite-to-infinite} and following a similar proof as Theorem \ref{lemma:RP-Regret-SPE}, we will get an $O\rbr{\frac{1}{T_0^{2/7}}}$-approximate $O(T)$-robust SPCCE. \qed

\begin{algorithm}
\caption{Full-information version of \cite[Algorithm 1]{mao2022provably-CCE-MDP} (see also \cite{jin2021v,songcan})}
\label{alg:self-play-finite-MDP}
\begin{algorithmic}[1]
\State{Let $\psi(\bx)=\frac{1}{2}\nbr{\bx}^2$ for $\bx\in \Delta_{|\cA_1|}$, $D_\psi(\bx,\by)=\psi(\bx)-\psi(\by)-\inner{\nabla\psi(\by)}{\bx-\by}$ for $\bx,\by\in \Delta_{|\cA_1|}$.}
\State{Initialize $\forall k\in\{1,2,...,K\}, \bh\in\cH_M, a\in\cA_1,~~~\underline{V}_0(\bh^k)\leftarrow 0$.}
\State{Initialize $\bpi\uo(\cdot\given \bh^k)\leftarrow \frac{\one}{|\cA_1|}$ as the uniform distribution over $\Delta_{|\cA_1|}$}
\State{Initialize $\eta_1\leftarrow \frac{1}{\sqrt T}, \alpha_0\leftarrow 1$}
\For{$t=1,2,...,T$}
\State{$\alpha_t\leftarrow \frac{K+1}{K+t},\beta_t\leftarrow 4\sqrt{\frac{K^3|\cA_1|}{t}},\eta_t\leftarrow \sqrt{\frac{1}{|\cA_1|K t}}$}
\State{$\forall \bh\in\cH_M, \underline{V}_t(\bh^{K+1})\leftarrow 0$}
\For{$k=K,K-1,...,1$}
\For{$\bh^k\in\cH_M$}
\For{$a_1\in\cA_1$}
\State{$g_t(\bh^k,a_1)\leftarrow \sum_{\ba_{-1}\in\cA_{-1}}\pi_t\uni(\ba_{-1}\given \bh^k)\rbr{\cL_1(a_1,\ba_{-1})+\underline{V}_t((\bh_{2:M},(a_1,\ba_{-1}))^{k+1})}$}
\EndFor
\State{$\tilde V_{t}(\bh^k)\leftarrow (1-\alpha_t) \underline{V}_{t-1}(\bh^k)+\alpha_t\rbr{\inner{\bpi_t\uo(\cdot\given\bh^k)}{\bg_t(\bh^k,\cdot)}-\beta_t}$}
\State{$\underline{V}_{t}(\bh^k)\leftarrow \max\cbr{\tilde V_{t}(\bh^k),0}$}
\State{$\theta'\leftarrow \argmin_{\theta\in\Delta_{|\cA_1|}} \cbr{\eta_t\inner{\theta}{\bg_t(\bh^k)}+D_{\psi}(\theta,\bpi_t\uo(\cdot\given \bh^k))}$}
\State{$\bpi_{t+1}\uo(\cdot\given \bh^k)\leftarrow \lambda_t \theta'+(1-\lambda_t)\frac{\one}{|\cA_1|}$ where $\lambda_t=\frac{\eta_{t+1}\alpha_t(1-\alpha_{t+1})}{\eta_t\alpha_{t+1}}$}
\EndFor
\EndFor
\EndFor
\end{algorithmic}
\end{algorithm}

Firstly, define $\alpha_t^s=\alpha_s\prod_{j=s+1}^t (1-\alpha_j)$. Notice that $\sum_{s=1}^t \alpha_t^s=1$. Then, we can define $\bar\bpi_t\ui$ as the strategy sampled from $\bpi_1\ui,\bpi_2\ui,...,\bpi_t\ui$ with probability $\alpha_t^1,\alpha_t^2,...,\alpha_t^t$. We will first prove that $\underline{V}_t(\bh^k)\leq V^{\hat\bpi\uo,\bar\bpi_t\uno}(\bh^k)$ for any $\hat\bpi\uo$ in Lemma \ref{lemma:underline-V-bound}. Then, the regret is upper-bounded by $\sum_{t=1}^T \rbr{V^{\bar\bpi_t}(\bh^k)-\underline{V}_t(\bh^k)}$, which can be proved to be upper-bounded by $\sqrt T$.%

\begin{lemma}
\label{lemma:underline-V-bound}
    At any timestep $t$, Algorithm \ref{alg:self-play-finite-MDP} guarantees that $\underline{V}_t(\bh^k)\leq V^{\hat\bpi\uo,\bar\bpi_t\uno}(\bh^k)$ for any $\hat\bpi\uo$.
    \proof

    Firstly, by definition, we have $0=\underline{V}_t(\bh^{K+1})\leq V^{\hat\bpi\uo,\bar\bpi_t\uno}(\bh^{K+1})$. Then, suppose it holds for $k+1$, then we have
    \begin{align*}
        V^{\hat\bpi\uo,\bar\bpi_t\uno}(\bh^k) \coloneqq &\sum_{s=1}^t \alpha_{t}^s \sum_{\ba\in\cA} \pi_s\uno(\ba_{-1}\given \bh^k)\hat\pi\uo(a_1\given \bh^k)\rbr{\cL_1(\ba)+V^{\hat\bpi\uo,\bar\bpi_t\uno}((\bh_{2:M},\ba)^{k+1})}\\
        \geq& \min_{\tilde\bpi\uo}\sum_{s=1}^t \alpha_{t}^s \sum_{\ba\in\cA} \pi_s\uno(\ba_{-1}\given \bh^k)\tilde\pi\uo(a_1\given \bh^k)\rbr{\cL_1(\ba)+V^{\tilde\bpi\uo,\bar\bpi_t\uno}((\bh_{2:M},\ba)^{k+1})}\\
        \geq& \min_{\tilde\bpi\uo}\sum_{s=1}^t \alpha_{t}^s \sum_{\ba\in\cA} \pi_s\uno(\ba_{-1}\given \bh^k)\tilde\pi\uo(a_1\given \bh^k)\rbr{\cL_1(\ba)+\underline{V}_s((\bh_{2:M},\ba)^{k+1})}\\
        =&\min_{\tilde\bpi\uo}\sum_{s=1}^t \alpha_{t}^s \inner{\tilde\bpi\uo(\cdot\given \bh^k)}{\bg_s(\bh^k,\cdot)}.
    \end{align*}
    Note that $(\bh_{2:M},\ba)\in\cH_M$ is a state and $(\bh_{2:M},\ba)^k$ denotes that this state is in the $k^{\rm th}$ level.

    \begin{lemma}
    \label{lemma:OMD-upper-bound-finite-horizon}
    Consider the following update-rule.
    \begin{align*}
        &\theta_{t+1}'=\argmin_{\theta\in\cX}\cbr{\eta_t\inner{\bg_t}{\theta}+\frac{1}{2}\nbr{\theta-\theta_t}^2}\\
        &\theta_{t+1}=\lambda_t \theta_{t+1}'+(1-\lambda_t)\theta_1\\
        &\lambda_t=\frac{\eta_{t+1} w_{t+1}^t}{\eta_t w_{t+1}^{t+1}}.
    \end{align*}
    The weight satisfies that $\frac{w_{T_1}^{t+1}}{w_{T_1}^t}=\frac{w_{T_2}^{t+1}}{w_{T_2}^t}$ for all $T_1,T_2\geq t+1$. We have
    \begin{align*}
        \sum_{t=1}^T w_T^t\inner{\theta_t-\hat\theta}{\bg_t}\leq D\frac{w_T^Tw_{T+1}^{T+1}}{2w_{T+1}^T\eta_T}+\sum_{t=1}^T \eta_t w_T^t\nbr{\bg_t}^2
    \end{align*}
    where $D \coloneqq \max_{\theta,\theta'\in\cX} \nbr{\theta-\theta'}^2$.

\end{lemma}

    Further, by Lemma \ref{lemma:OMD-upper-bound-finite-horizon} (notice that $\lambda_t\leq \frac{\alpha_t (1-\alpha_{t+1})}{\alpha_{t+1}}=\frac{t}{K+t}\leq 1$ so Lemma \ref{lemma:OMD-upper-bound-finite-horizon} holds here), we have
    \begin{align*}
        \sum_{s=1}^t \alpha_{t}^s \inner{\bpi_s\uo(\cdot\given \bh^k)}{\bg_s(\bh^k,\cdot)}-\min_{\hat\bpi\uo}\sum_{s=1}^t \alpha_{t}^s \inner{\hat\bpi\uo(\cdot\given \bh^k)}{\bg_s(\bh^k,\cdot)}\leq& \frac{D\alpha_{t+1}}{(1-\alpha_{t+1})\eta_t}+\sum_{s=1}^t \eta_s \alpha_t^s \nbr{\bg_s(\bh^k,\cdot)}^2\\
        \leq& \frac{K+1}{t}\sqrt {|\cA_1|K t}+\sqrt{|\cA_1|K^3}\sum_{s=1}^t \frac{\alpha_t^s}{\sqrt s}\\
        \leq& \frac{4\sqrt{|\cA_1|K^3}}{\sqrt t}
    \end{align*}
    where the first inequality is by $D\leq 2$. The second line is by $\nbr{\bg_s(\bh^k,\cdot)}^2\leq K^2 |\cA_1|$ and the last line is by \cite[Lemma 4.1]{jin2018q-alpha-lemma} that $\sum_{s=1}^t \frac{\alpha_t^s}{\sqrt s}\leq \frac{2}{\sqrt t}$. Therefore
    \begin{align*}
        V^{\hat\bpi\uo,\bar\bpi_t\uno}(\bh^k)\geq& \min_{\tilde\bpi\uo}\sum_{s=1}^t \alpha_{t}^s \inner{\tilde\bpi\uo(\cdot\given \bh^k)}{\bg_s(\bh^k,\cdot)}\\
        \geq& \sum_{s=1}^t \alpha_{t}^s \inner{\bpi_s\uo(\cdot\given \bh^k)}{\bg_s(\bh^k,\cdot)}-\frac{4\sqrt{|\cA_1|K^3}}{\sqrt t}\\
        =&\underline{V}_t(\bh^k).\qedhere
    \end{align*}
\end{lemma}

\begin{proof}[Proof of Lemma \ref{lemma:convergence-CCE-finite-MDP}]

Define $\delta_t(\bh^k)=V^{\bar\bpi_t}(\bh^k)-\underline{V}_t(\bh^k)$. Then, we have
\begin{align*}
    \delta_t(\bh^k)\leq&\sum_{s=1}^t \alpha_t^s \sum_{\ba\in\cA}\pi_s(\ba\given \bh^k)\rbr{V^{\bar\bpi_s}((\bh_{2:M},\ba)^{k+1})-\underline{V}_s((\bh_{2:M},\ba)^{k+1})+\beta_s}\\
    =&\sum_{s=1}^t \alpha_t^s \sum_{\ba\in\cA}\pi_s(\ba\given \bh^k)\rbr{\delta_s((\bh_{2:M},\ba)^{k+1})+\beta_s}\\
    \leq& \sum_{s=1}^t \alpha_t^s \delta_s^{k+1}+\frac{4\sqrt{|\cA_1|K^3}}{\sqrt t}
\end{align*}
where the last line is by our definition that $\delta_t^k \coloneqq \max_{\bh\in\cH_M} \delta_t(\bh^k)$. Then, taking the sum of $t$, we have
\begin{align*}
    \sum_{t=1}^T \delta_t^k\leq&\sum_{t=1}^T \sum_{s=1}^t \alpha_t^s \delta_s^{k+1}+\sum_{t=1}^T \frac{4\sqrt{|\cA_1|K^3}}{\sqrt t}\\
    =&\sum_{s=1}^T \delta_s^{k+1} \sum_{t=s}^T \alpha_t^s +\sum_{t=1}^T \frac{4\sqrt{|\cA_1|K^3}}{\sqrt t}\\
    \overset{(i)}{\leq}&\sum_{s=1}^T \delta_s^{k+1} \sum_{t=s}^\infty \alpha_t^s + 8\sqrt{|\cA_1|K^3 T}\\
    \overset{(ii)}{=}&(1+\frac{1}{K})\sum_{s=1}^T \delta_s^{k+1}+8\sqrt{|\cA_1|K^3 T}
\end{align*}
where $(i)$ is because $\sum_{i=1}^n \frac{1}{\sqrt i}\leq \int_0^n \frac{1}{\sqrt x}dx=2\sqrt n$. $(ii)$ is by \cite[Lemma 4.1]{jin2018q-alpha-lemma} that $\sum_{t=s}^\infty \alpha_t^s=1+\frac{1}{K}$. Therefore, using the inequalities above recursively, we have
\begin{align*}
    \sum_{t=1}^T \delta_t^1\leq& 8e\sum_{k=1}^K \sqrt{|\cA_1|K^3 T}= 8eK^2 \sqrt{|\cA_1| K T}
\end{align*}
where the first inequality is by $(1+\frac{1}{K})^K\leq e$. Therefore, for any $\bh\in\cH_M$, we have
\begin{align*}
    \sum_{t=1}^T \rbr{V^{\bar\bpi_t}(\bh^1)-V^{\hat\bpi\uo,\bar\bpi\uno_t}(\bh^1)}\overset{(i)}{\leq}& \sum_{t=1}^T \rbr{V^{\bar\bpi_t}(\bh^1)-\underline{V}_t(\bh^1)}= \sum_{t=1}^T \delta_t^1\leq 8eK^2 \sqrt{|\cA_1|K T}
\end{align*}
where $(i)$ is by Lemma \ref{lemma:underline-V-bound}. This proves the claim for player $1$. Applying the same argument to each player $i$ replaces $|\cA_1|$ by $|\cA_i|$, and taking the maximum over $i$ gives the stated bound.
\end{proof}

\begin{lemma}
    \label{lemma:OMD-bound-fundamental}
    For any fixed $\hat\theta\in\cX$ where $\cX$ is a convex set, the weighted regret of Algorithm \ref{alg:self-play-finite-MDP} with respect to $\theta$ can be bounded by
    \begin{align*}
        \sum_{t=1}^T w_T^t\inner{\theta_t-\hat\theta}{\bg_t}\leq \frac{w_T^Tw_{T+1}^{T+1}}{w_{T+1}^T\eta_T} D_{\psi}(\hat\theta,\theta_1)+\sum_{t=1}^T w_T^t\nbr{\bg_t}\cdot \nbr{\theta_t-\theta_{t+1}}
    \end{align*}
    where the update rule is
    \begin{align*}
        &\theta_{t+1}'=\argmin_{\theta\in\cX}\cbr{\eta_t\inner{\bg_t}{\theta}+D_{\psi}(\theta,\theta_t)}\\
        &\theta_{t+1}=\lambda_t \theta_{t+1}'+(1-\lambda_t)\theta_1\\
        &\lambda_t=\frac{\eta_{t+1} w_{t+1}^t}{\eta_t w_{t+1}^{t+1}}.
    \end{align*}
    The weight satisfies that $\frac{w_{T_1}^{t+1}}{w_{T_1}^t}=\frac{w_{T_2}^{t+1}}{w_{T_2}^t}$ for all $T_1,T_2\geq t+1$.
    \proof
    Since $\theta_{t+1}'=\argmin_{\theta\in\cX}\cbr{\inner{\eta_t \bg_t}{\theta}+D_{\psi}(\theta,\theta_t)}$ and $D_{\psi}(\theta,\theta_t)=\psi(\theta)-\psi(\theta_t)-\inner{\nabla \psi(\theta_t)}{\theta-\theta_t}$, by first-order optimality and convexity of $\cX$, we have
    \begin{align*}
        \inner{\eta_t\bg_t+\nabla \psi(\theta_{t+1}')-\nabla\psi(\theta_t)}{\hat\theta-\theta_{t+1}'}\geq 0.
    \end{align*}
    By some algebraic manipulation, we have
    \begin{align}
        \inner{\bg_t}{\theta_t-\hat\theta}\leq& \frac{1}{\eta_t}\inner{\nabla \psi(\theta_{t+1}')-\nabla\psi(\theta_t)}{\hat\theta-\theta_{t+1}'}+\inner{\bg_t}{\theta_t-\theta_{t+1}'}\notag\\
        =&\frac{1}{\eta_t}\rbr{D_{\psi}(\hat\theta,\theta_{t})-D_{\psi}(\hat\theta,\theta_{t+1}')-D_{\psi}(\theta_{t+1}',\theta_t)}+\inner{\bg_t}{\theta_t-\theta_{t+1}'}\notag\\
        \leq&\frac{1}{\eta_t}\rbr{D_{\psi}(\hat\theta,\theta_{t})-D_{\psi}(\hat\theta,\theta_{t+1}')}+\inner{\bg_t}{\theta_t-\theta_{t+1}'}\label{eq:OMD-bound-eq1}
    \end{align}
    where the last line is by the non-negativity of Bregman divergence.

    Note that by convexity of Bregman divergence and $\theta_{t+1}=\lambda_t \theta_{t+1}'+(1-\lambda_t)\theta_1$,
    \begin{align*}
        \lambda_t D_\psi(\hat\theta,\theta_{t+1}')+(1-\lambda_t) D_{\psi}(\hat\theta,\theta_1)\geq D_\psi(\hat\theta, \theta_{t+1}).
    \end{align*}
    By some algebraic manipulation and substituting it back to \Cref{eq:OMD-bound-eq1}, we have
    \begin{align*}
        &w_T^t\inner{\bg_t}{\theta_t-\hat\theta}\\
        \leq& \frac{w_T^t}{\eta_t}\rbr{D_{\psi}(\hat\theta,\theta_{t})-D_{\psi}(\hat\theta,\theta_{t+1}')}+w_T^t\inner{\bg_t}{\theta_t-\theta_{t+1}'}\\
        \leq& \frac{w_T^t}{\eta_t}\rbr{D_{\psi}(\hat\theta,\theta_{t})-\frac{1}{\lambda_t} D_\psi(\hat\theta, \theta_{t+1})+\frac{1-\lambda_t}{\lambda_t} D_\psi(\hat\theta, \theta_1)}+w_T^t\inner{\bg_t}{\theta_t-\theta_{t+1}'}\\
        =&\frac{w_T^t}{\eta_t} D_{\psi}(\hat\theta,\theta_{t})-\frac{w_T^t}{w_{t+1}^t}\frac{w_{t+1}^{t+1}}{\eta_{t+1}} D_\psi(\hat\theta, \theta_{t+1})-\rbr{\frac{w_T^t}{\eta_t}-\frac{w_T^t w_{t+1}^{t+1}}{\eta_{t+1}w_{t+1}^t}} D_\psi(\hat\theta, \theta_1)+w_T^t\inner{\bg_t}{\theta_t-\theta_{t+1}'}\\
        =&\frac{w_T^t}{\eta_t} D_{\psi}(\hat\theta,\theta_{t})-\frac{w_T^{t+1}}{\eta_{t+1}} D_\psi(\hat\theta, \theta_{t+1})+\rbr{\frac{w_T^{t+1}}{\eta_{t+1}}-\frac{w_T^t}{\eta_t}} D_\psi(\hat\theta, \theta_1)+w_T^t\inner{\bg_t}{\theta_t-\theta_{t+1}'}
    \end{align*}
    where the last line is because $\frac{w_{t+1}^{t+1}}{w_{t+1}^t}=\frac{w_T^{t+1}}{w_T^t}$. For consistency here, we define $w_T^{T+1} \coloneqq \frac{w_T^Tw_{T+1}^{T+1}}{w_{T+1}^T}$.

    Summing over $t$ and telescoping, we have
    \begin{align*}
        \sum_{t=1}^T w_T^t\inner{\bg_t}{\theta_t-\hat\theta}\leq& \frac{w_T^1}{\eta_1}D_{\psi}(\hat\theta,\theta_1)+\sum_{t=1}^T \rbr{\frac{w_T^{t+1}}{\eta_{t+1}}-\frac{w_T^t}{\eta_t}} D_{\psi}(\hat\theta,\theta_1)+\sum_{t=1}^T w_T^t\inner{\bg_t}{\theta_t-\theta_{t+1}'}\\
        \leq& \frac{w_T^{T+1}}{\eta_{T+1}}D_{\psi}(\hat\theta,\theta_1)+\sum_{t=1}^T w_T^t\nbr{\bg_t}\cdot \nbr{\theta_t-\theta_{t+1}'}\\
        =&\frac{w_T^Tw_{T+1}^{T+1}}{w_{T+1}^T\eta_T}D_{\psi}(\hat\theta,\theta_1)+\sum_{t=1}^T w_T^t\nbr{\bg_t}\cdot \nbr{\theta_t-\theta_{t+1}'}
    \end{align*}
    where we define $w_T^{T+1} \coloneqq \frac{w_T^Tw_{T+1}^{T+1}}{w_{T+1}^T}$.\qed
\end{lemma}

\begin{proof}[Proof of Lemma \ref{lemma:OMD-upper-bound-finite-horizon}]
    Since the update-rule in Lemma \ref{lemma:convergence-CCE-finite-MDP} is a special case of Lemma \ref{lemma:OMD-bound-fundamental} with $\psi(\bx)=\frac{1}{2}\nbr{\bx}^2$, we have
    \begin{align*}
        \sum_{t=1}^T w_T^t\inner{\theta_t-\hat\theta}{\bg_t}\leq& \frac{w_T^Tw_{T+1}^{T+1}}{w_{T+1}^T\eta_T}D_{\psi}(\hat\theta,\theta_1)+\sum_{t=1}^T w_T^t\nbr{\bg_t}\cdot \nbr{\theta_t-\theta_{t+1}'}\\
        =&\frac{w_T^Tw_{T+1}^{T+1}}{2 w_{T+1}^T\eta_T}\nbr{\hat\theta-\theta_1}^2+\sum_{t=1}^T w_T^t\nbr{\bg_t}\cdot \nbr{\theta_t-\theta_{t+1}'}.
    \end{align*}
    Also, notice that when $\psi(\bx)=\frac{1}{2}\nbr{\bx}^2$, we have
    \begin{align*}
        \theta_{t+1}'=&\argmin_{\theta\in\cX} \cbr{\eta_t\inner{\bg_t}{\theta}+\frac{1}{2}\nbr{\theta-\theta_t}^2}=\Proj{\cX}{\theta_t-\eta_t \bg_t}.
    \end{align*}
    Therefore, $\nbr{\theta_t-\theta_{t+1}'}\leq \nbr{\theta_t-(\theta_t-\eta_t \bg_t)}=\eta_t\nbr{\bg_t}$. So,
    \begin{align*}
        \sum_{t=1}^T w_T^t\inner{\theta_t-\hat\theta}{\bg_t}\leq& \frac{w_T^Tw_{T+1}^{T+1}}{2 w_{T+1}^T\eta_T}\nbr{\hat\theta-\theta_1}^2+\sum_{t=1}^T \eta_t w_T^t\nbr{\bg_t}^2\\
        \leq&D\frac{w_T^Tw_{T+1}^{T+1}}{2w_{T+1}^T\eta_T}+\sum_{t=1}^T \eta_t w_T^t\nbr{\bg_t}^2.
    \end{align*}
    The last line is by the definition that $\max_{\theta,\theta'\in\cX}\nbr{\theta-\theta'}^2\leq D$.
\end{proof}

\section{Auxiliary Lemmas}
\label{appendix:auxiliary-lemmas}

\begin{lemma}
\label{lemma:Euler-func-bound}
For any $x_1,x_2,...,x_n\in [0,1]$, we have
\begin{align}
    \prod_{i=1}^n (1-x_i)\geq 1-\sum_{i=1}^n x_i.
\end{align}
\proof
\begin{equation*}
    \prod_{i=1}^n (1-x_i)=1-x_1-x_2(1-x_1)-x_3(1-x_1)(1-x_2)-...-x_n\prod_{i=1}^{n-1} (1-x_i)\geq 1-\sum_{i=1}^n x_i.\qedhere
\end{equation*}
\end{lemma}

In the following, we will prove Lemma \ref{lemma:approximation}. We prove that with the forgetful condition (Condition \ref{assumption:2-forgetful}) on all agents, the past history will not affect much of the current action distribution, because all of the agents ``forget" what the past history is.

\begin{lemma}[{Finite-Memory Approximation Errors}]
\label{lemma:approximation}
    Suppose Condition \ref{assumption:2-forgetful} is satisfied for every player $i\in\cN$. At any timestep $t$, for any $m\leq t-1$, any initial history $\bh'\in\cH_{t-m-1}$ and $\ba\in\cA$, when $\gamma\leq \frac{1}{2(N+2)}$, we have
    \small
    \begin{align}
        &\bigg|\sum_{\bh\in\cH_m}\Pr((\bh,\ba);\bpi_{t-m:t})-\sum_{\bh\in\cH_m}\Pr((\bh,\ba)\given\bh';\bpi_{t-m:t})\bigg|\notag\\
        \leq& C_m^\gamma \cdot \Big(\sum_{\bh\in\cH_m}\Pr((\bh,\ba);\bpi_{t-m:t})+\sum_{\bh\in\cH_m}\Pr((\bh,\ba)\given\bh';\bpi_{t-m:t})\Big)\notag
    \end{align}
    \normalsize
    where
    \begin{align}
        C_m^\gamma \coloneqq (2N+1)^{m+1}\gamma^{m+1}.
    \end{align}

\proof

Notice that for any $\bh_1\in\cH_1$ and $m\geq 0$, we have
    \begin{align}
        &\Big|\sum_{\bh_2\in\cH_m} \Pr((\bh_2,\bh_1);\bpi_{t-m:t})-\sum_{\bh_2\in\cH_m} \Pr((\bh_2,\bh_1)\given \bh';\bpi_{t-m:t})\Big|\notag\\
        \leq& \Big|\sum_{\bh_2\in\cH_m} \Pr(\bh_1\given \bh_2;\bpi_{t-m:t})\Big(\Pr(\bh_2;\bpi_{t-m:t})-\Pr(\bh_2\given \bh';\bpi_{t-m:t})\Big)\Big|\notag\\
        &+\Big|\sum_{\bh_2\in\cH_m} \Big(\Pr(\bh_1\given \bh_2;\bpi_{t-m:t})-\Pr(\bh_1|(\bh',\bh_2);\bpi_{t-m:t})\Big)\Pr(\bh_2\given \bh';\bpi_{t-m:t})\Big|\notag\\
        \leq&\Big|\sum_{\bh_2\in\cH_m} \sum_{k=0}^{m-1}\Big( \Pr(\bh_1\given \bh_{2,m-k:m};\bpi_{t-m:t})- \Pr(\bh_1\given \bh_{2,m-k+1:m};\bpi_{t-m:t})\Big)\notag\\
        &\cdot\Big(\Pr(\bh_2;\bpi_{t-m:t})- \Pr(\bh_2\given \bh';\bpi_{t-m:t})\Big)\Big|\tag*{\raisebox{.5pt}{\textcircled{\raisebox{-.9pt} {1}}}
}\label{eq:approximation-eq2-l1}\\
        &+\Big|\sum_{\bh_2\in\cH_m} \Pr(\bh_1\given \bh_{2,m+1:m};\bpi_{t-m:t}) \Big(\Pr(\bh_2;\bpi_{t-m:t})- \Pr(\bh_2\given \bh';\bpi_{t-m:t})\Big)\Big|\tag*{\raisebox{.5pt}{\textcircled{\raisebox{-.9pt} {2}}}
}\label{eq:approximation-eq2-l2}\\
        &+\Big|\sum_{\bh_2\in\cH_m} \Big(\Pr(\bh_1\given \bh_2;\bpi_{t-m:t})-\Pr(\bh_1|(\bh',\bh_2);\bpi_{t-m:t})\Big)\Pr(\bh_2\given \bh';\bpi_{t-m:t})\Big|.\tag*{\raisebox{.5pt}{\textcircled{\raisebox{-.9pt} {3}}}
}\label{eq:approximation-eq2-l3}
    \end{align}

 \paragraph{Bounding \ref{eq:approximation-eq2-l2}} \ref{eq:approximation-eq2-l2} The second term is equal to $0$ since
     \begin{align*}
         &\Big|\sum_{\bh_2\in\cH_m} \Pr(\bh_1\given \bh_{2,m+1:m};\bpi_{t-m:t}) \Big(\Pr(\bh_2;\bpi_{t-m:t})- \Pr(\bh_2\given \bh';\bpi_{t-m:t})\Big)\Big|\\
         =&\Big| \Pr(\bh_1;\bpi_{t-m:t}) \sum_{\bh_2\in\cH_m}\Big(\Pr(\bh_2;\bpi_{t-m:t})- \Pr(\bh_2\given \bh';\bpi_{t-m:t})\Big)\Big|\\
         =&\Big| \Pr(\bh_1;\bpi_{t-m:t}) (1-1)\Big|=0.
     \end{align*}

    \paragraph{Bounding \ref{eq:approximation-eq2-l3}} Firstly, \ref{eq:approximation-eq2-l3} can be bounded by,
    \begin{align*}
        &\sum_{\bh_2\in\cH_m} \Big|\Pr(\bh_1\given \bh_2;\bpi_{t-m:t})-\Pr(\bh_1|(\bh',\bh_2);\bpi_{t-m:t})\Big|\Pr(\bh_2\given \bh';\bpi_{t-m:t})\\
        \leq& \sum_{\bh_2\in\cH_m} 2N\gamma^{m+1} \Pr(\bh_1|(\bh',\bh_2);\bpi_{t-m:t})\Pr(\bh_2\given \bh';\bpi_{t-m:t})\\
        =&2N\gamma^{m+1} \sum_{\bh_2\in\cH_m} \Pr((\bh_2,\bh_1)\given \bh';\bpi_{t-m:t})
    \end{align*}
    where the second line is by Lemma \ref{lemma:prob-approx-exp-error}.
    \begin{lemma}
    \label{lemma:prob-approx-exp-error}
    Consider when all players satisfy Condition \ref{assumption:2-forgetful} and $\gamma\leq \frac{1}{2(N+2)}$. For any $\bh_1,\bh_2,\bh_3\in\cH$, we have
    \begin{align}
        &|\Pr(\bh_1\given \bh_2;\bpi)-\Pr(\bh_1|(\bh_3,\bh_2);\bpi)|\notag
        \leq 2N\gamma^{L(\bh_2)+1}\min\cbr{\Pr(\bh_1\given \bh_2;\bpi),\Pr(\bh_1|(\bh_3,\bh_2);\bpi)}\\
        &|\Pr(\bh_1\given \bh_2;\bpi)-\Pr(\bh_1|(\bh_3,\bh_2);\bpi)|\notag
        \leq N\frac{\gamma^{L(\bh_2)+1}}{1-\gamma}\max\cbr{\Pr(\bh_1\given \bh_2;\bpi),\Pr(\bh_1|(\bh_3,\bh_2);\bpi)}.
    \end{align}
    \end{lemma}
    The proof is deferred to the end of this section.

    \paragraph{Bounding \ref{eq:approximation-eq2-l1}} Then, \ref{eq:approximation-eq2-l1} is bounded by
    \begin{align*}
        &\Big|\sum_{\bh_2\in\cH_m} \sum_{k=0}^{m-1}\Big( \Pr(\bh_1\given \bh_{2,m-k:m};\bpi_{t-m:t})- \Pr(\bh_1\given \bh_{2,m-k+1:m};\bpi_{t-m:t})\Big) \Big(\Pr(\bh_2;\bpi_{t-m:t})- \Pr(\bh_2\given \bh';\bpi_{t-m:t})\Big)\Big|\\
        =&\Big|\sum_{k=0}^{m-1}\sum_{\bh_2\in\cH_{k+1}} \Big( \Pr(\bh_1\given \bh_2;\bpi_{t-m:t})- \Pr(\bh_1\given \bh_{2,2:k+1};\bpi_{t-m:t})\Big)\\
        &\cdot\sum_{\bh_3\in\cH_{m-k-1}}\Big(\Pr((\bh_3,\bh_2);\bpi_{t-m:t})- \Pr((\bh_3,\bh_2)\given \bh';\bpi_{t-m:t})\Big)\Big|\\
        \leq& \sum_{k=0}^{m-1}\sum_{\bh_2\in\cH_{k+1}}\Big|\Pr(\bh_1\given \bh_2;\bpi_{t-m:t})- \Pr(\bh_1\given \bh_{2,2:k+1};\bpi_{t-m:t}) \Big|\\
        &\cdot \Big|\sum_{\bh_3\in\cH_{m-k-1}}\Big(\Pr((\bh_3,\bh_2);\bpi_{t-m:t})- \Pr((\bh_3,\bh_2)\given \bh';\bpi_{t-m:t})\Big)\Big|\\
        \leq& \sum_{k=0}^{m-1} \sum_{\bh_2\in \cH_{k+1}}\frac{N}{1-\gamma} \gamma^{k+1}\max\{\Pr(\bh_1\given \bh_2),\Pr(\bh_1\given \bh_{2,2:k+1})\} \\
        &\cdot C^\gamma_{m-k-1}\rbr{\sum_{\bh_3\in\cH_{m-k-1}}\Pr((\bh_3,\bh_2);\bpi_{t-m:t})+\sum_{\bh_3\in\cH_{m-k-1}}\Pr((\bh_3,\bh_2)\given \bh';\bpi_{t-m:t})}
    \end{align*}
    where the last line uses Lemma \ref{lemma:prob-approx-exp-error} and the recursively applying the argument to $m-k-1$, namely
    \begin{align*}
        &\Big|\sum_{\bh_3\in\cH_{m-k-1}}\Big(\Pr((\bh_3,\bh_2);\bpi_{t-m:t})- \Pr((\bh_3,\bh_2)\given \bh';\bpi_{t-m:t})\Big)\Big|\\
        \leq& C^\gamma_{m-k-1} \rbr{\sum_{\bh_3\in\cH_{m-k-1}}\Pr((\bh_3,\bh_2);\bpi_{t-m:t})+\sum_{\bh_3\in\cH_{m-k-1}}\Pr((\bh_3,\bh_2)\given \bh';\bpi_{t-m:t})}.
    \end{align*}
    The base case when $m=0$ follows directly from Lemma \ref{lemma:prob-approx-exp-error}.

    Notice that
    \begin{align*}
        & \max\{\Pr(\bh_1\given \bh_2),\Pr(\bh_1\given \bh_{2,2:k+1})\}\sum_{\bh_3\in\cH_{m-k-1}}\Pr((\bh_3,\bh_2);\bpi_{t-m:t})\\
        =&\sum_{\bh_3\in\cH_{m-k-1}}\Pr((\bh_3,\bh_2);\bpi_{t-m:t})(\max\{\Pr(\bh_1\given \bh_2),\Pr(\bh_1\given \bh_{2,2:k+1})\}-\Pr(\bh_1|(\bh_3,\bh_2)))\\
        &+\sum_{\bh_3\in\cH_{m-k-1}}\Pr((\bh_3,\bh_2);\bpi_{t-m:t})\Pr(\bh_1|(\bh_3,\bh_2))\\
        \leq&\sum_{\bh_3\in\cH_{m-k-1}}\Pr((\bh_3,\bh_2);\bpi_{t-m:t})|\max\{\Pr(\bh_1\given \bh_2),\Pr(\bh_1\given \bh_{2,2:k+1})\}-\Pr(\bh_1|(\bh_3,\bh_2))|\\
        &+\sum_{\bh_3\in\cH_{m-k-1}}\Pr((\bh_3,\bh_2,\bh_1);\bpi_{t-m:t})\\
        \leq& 2N\gamma^{k+1}\sum_{\bh_3\in\cH_{m-k-1}} \Pr((\bh_3,\bh_2);\bpi_{t-m:t})\Pr(\bh_1|(\bh_3,\bh_2))+\sum_{\bh_3\in\cH_{m-k-1}}\Pr((\bh_3,\bh_2,\bh_1);\bpi_{t-m:t})\\
        =&(2N\gamma^{k+1}+1)\sum_{\bh_3\in\cH_{m-k-1}}\Pr((\bh_3,\bh_2,\bh_1);\bpi_{t-m:t}).
    \end{align*}
    Similarly,
    \begin{align*}
         &\max\{\Pr(\bh_1\given \bh_2),\Pr(\bh_1\given \bh_{2,2:k+1})\}\sum_{\bh_3\in\cH_{m-k-1}}\Pr((\bh_3,\bh_2)\given \bh';\bpi_{t-m:t})\\
         \leq&(2N\gamma^{k+1}+1)\sum_{\bh_3\in\cH_{m-k-1}}\Pr((\bh_3,\bh_2,\bh_1)\given \bh';\bpi_{t-m:t}).
    \end{align*}

    Then, \ref{eq:approximation-eq2-l1} can be bounded by
    \begin{align*}
        &\Big|\sum_{\bh_2\in\cH_m} \sum_{k=0}^{m-1}\Big( \Pr(\bh_1\given \bh_{2,m-k:m};\bpi_{t-m:t})- \Pr(\bh_1\given \bh_{2,m-k+1:m};\bpi_{t-m:t})\Big) \Big(\Pr(\bh_2;\bpi_{t-m:t})- \Pr(\bh_2\given \bh';\bpi_{t-m:t})\Big)\Big|\\
        \leq& \sum_{k=0}^{m-1}(\frac{2N\gamma^{k+1}+1}{1-\gamma}N)\gamma^{k+1} C^\gamma_{m-k-1}\\
        &\cdot \sum_{\bh_2\in\cH_{k+1}}\rbr{\sum_{\bh_3\in\cH_{m-k-1}} \Pr((\bh_3,\bh_2,\bh_1);\bpi_{t-m:t})+ \sum_{\bh_3\in\cH_{m-k-1}} \Pr((\bh_3,\bh_2,\bh_1)\given \bh';\bpi_{t-m:t})}\\
        =&\sum_{k=0}^{m-1}(\frac{2N\gamma^{k+1}+1}{1-\gamma}N)\gamma^{k+1} C^\gamma_{m-k-1}\rbr{\sum_{\bh_2\in\cH_m} \Pr((\bh_2,\bh_1);\bpi_{t-m:t})+ \sum_{\bh_2\in\cH_m} \Pr((\bh_2,\bh_1)\given \bh';\bpi_{t-m:t})}
    \end{align*}
    Given that $\gamma\leq \frac{1}{2(N+2)}$, we have $\frac{2N\gamma^{k+1}+1}{1-\gamma}\leq 2$. Therefore, combining \ref{eq:approximation-eq2-l1}, \ref{eq:approximation-eq2-l2}, \ref{eq:approximation-eq2-l3} together, for any $m\geq 1$, we have
    \begin{align*}
        C_m^\gamma=2N\gamma^{m+1}+\sum_{k=0}^{m-1} 2N\gamma^{k+1}C^\gamma_{m-k-1}.
    \end{align*}
    Note that
    \begin{align*}
        |\Pr(\bh_2)-\Pr(\bh_2\given \bh')|\leq \frac{N\gamma}{1-\gamma}\Pr(\bh_2),
    \end{align*}
    which implies taking $C_0^\gamma\geq\frac{N\gamma}{1-\gamma}$ is enough. So, we have $C_m^\gamma\leq (2N+1)^{m+1}\gamma^{m+1}$ for $m\geq 1$. Since $C_m^\gamma$ is an upper-bound, we just take $C_m^\gamma=(2N+1)^{m+1}\gamma^{m+1}$. \qed
\end{lemma}

\begin{proof}[Proof of Lemma \ref{corollary:approximation}]
When $t-1\leq m$, we have $f^{\min\{t-1,m\}}(\bpi_{t-m:t})=f^{t-1}(\bpi_{1:t})$. When $t-1>m$, we have
\begin{align*}
    &|f^m(\bpi_{t-m:t})-f^{t-1}(\bpi_{1:t})|\\
    =&\abr{\sum_{\ba\in\cA}\cL_1(\ba)\Big(\sum_{\bh\in\cH_m} \Pr((\bh,\ba);\bpi_{t-m:t})-\sum_{\bh\in\cH_{t-1}} \Pr((\bh,\ba);\bpi_{1:t})\Big)}\\
    \overset{(i)}{\leq}& \sum_{\ba\in\cA} \Big|\sum_{\bh\in\cH_m} \Pr((\bh,\ba);\bpi_{t-m:t})-\sum_{\bh\in\cH_{t-1}} \Pr((\bh,\ba);\bpi_{1:t})\Big|\\
    =&\sum_{\ba\in\cA}\Big|\sum_{\bh\in\cH_{t-1}} \big(\Pr((\bh_{t-m:t-1},\ba);\bpi_{t-m:t})-\Pr((\bh_{t-m:t-1},\ba)\given \bh_{1:t-m-1};\bpi_{t-m:t})\big)\Pr(\bh_{1:t-m-1};\bpi_{1:t-m-1})\Big|\\
    \leq& \sum_{\ba\in\cA}\sum_{\bh_1\in\cH_{t-m-1}}\Pr(\bh_1;\bpi_{1:t-m-1})\Big|\sum_{\bh_2\in\cH_m} \Big(\Pr((\bh_2,\ba);\bpi_{t-m:t})-\Pr((\bh_2,\ba)\given \bh_1;\bpi_{t-m:t})\Big)\Big|\\
    \overset{(ii)}{\leq}&C_m^\gamma\sum_{\ba\in\cA}\sum_{\bh_1\in\cH_{t-m-1}}\Pr(\bh_1;\bpi_{1:t-m-1})\rbr{\sum_{\bh_2\in\cH_m} \Pr((\bh_2,\ba);\bpi_{t-m:t})+\sum_{\bh_2\in\cH_m}\Pr((\bh_2,\ba)\given \bh_1;\bpi_{t-m:t})}\\
    =& 2C_m^\gamma
\end{align*}
where $(i)$ is because $\forall \ba\in\cA, \cL_1(\ba)\in[0,1]$ and $(ii)$ is by Lemma \ref{lemma:approximation}.
\end{proof}

\begin{proof}[Proof of Lemma \ref{lemma:prob-approx-exp-error}]
Let $\ell=L(\bh_1)$ and $r=L(\bh_2)$. If $\ell=0$, both conditional probabilities are equal to $1$, so the claim is trivial. Otherwise, write
\begin{align*}
    P&\coloneqq \Pr(\bh_1\given \bh_2;\bpi),\\
    Q&\coloneqq \Pr(\bh_1\given (\bh_3,\bh_2);\bpi).
\end{align*}
Multiplying the one-step probability ratios along the $\ell$ action profiles in $\bh_1$, Condition \ref{assumption:2-forgetful} gives
\begin{align*}
    \prod_{u=1}^{\ell}(1-\gamma^{r+u})^N
    \leq \frac{P}{Q}
    \leq \prod_{u=1}^{\ell}(1-\gamma^{r+u})^{-N},
\end{align*}
with the convention that the inequalities are trivial when $P=Q=0$. By Lemma \ref{lemma:Euler-func-bound},
\begin{align*}
    \prod_{u=1}^{\ell}(1-\gamma^{r+u})^N
    \geq 1- N\sum_{u=1}^{\ell}\gamma^{r+u}
    \geq 1- \frac{N\gamma^{r+1}}{1-\gamma}.
\end{align*}
Set $S\coloneqq N\gamma^{r+1}/(1-\gamma)$. Since $\gamma\leq 1/(2(N+2))$, we have $S<1$. The preceding bounds imply
\begin{align*}
    (1-S)Q\leq P\leq \frac{Q}{1-S}.
\end{align*}
Therefore,
\begin{align*}
    |P-Q|\leq S\max\{P,Q\}
    =N\frac{\gamma^{L(\bh_2)+1}}{1-\gamma}\max\{P,Q\},
\end{align*}
which proves the second inequality.

For the first inequality, the same two-sided bound gives
\begin{align*}
    |P-Q|\leq \frac{S}{1-S}\min\{P,Q\}.
\end{align*}
Under $\gamma\leq 1/(2(N+2))$, $S/(1-S)\leq 2N\gamma^{r+1}$. Hence
\begin{align*}
    |P-Q|\leq 2N\gamma^{L(\bh_2)+1}\min\{P,Q\},
\end{align*}
which completes the proof.
\end{proof}

\subsection{Proof of Lemma \ref{lemma:Lipschitz}}
\label{appendix:lemma-Lipschitz}

    Firstly,
    \begin{align}
        &|f^m(\bpi)-f^m(\tilde\bpi)|\notag\\
        =&\Big|\sum_{\ba\in \cA} \cL_1(\ba)\sum_{\bh\in\cH_m} \Big(\Pr((\bh,\ba);\bpi)-\Pr((\bh,\ba);\tilde\bpi)\Big) \Big|\notag\\
        \leq& \sum_{\ba\in \cA}\Big|\sum_{\bh\in\cH_m} (\pi_{m+1}(\ba\given \bh)\Pr(\bh;\bpi)-\tilde \pi_{m+1}(\ba\given \bh)\Pr(\bh;\tilde\bpi)) \Big|\notag\\
        \leq&\sum_{\ba\in \cA}\Big|\sum_{\bh\in\cH_m} (\pi_{m+1}(\ba\given \bh)-\tilde \pi_{m+1}(\ba\given \bh))\Pr(\bh;\bpi) \Big|\label{eq:eq213-first}\\
        &+\sum_{\ba\in \cA}\Big|\sum_{\bh\in\cH_m} (\Pr(\bh;\bpi)-\Pr(\bh;\tilde\bpi))\tilde \pi_{m+1}(\ba\given \bh) \Big|.\label{eq:eq213-second}
    \end{align}
    Term \eqref{eq:eq213-first}  is bounded by
    \begin{align*}
        &\sum_{\ba\in \cA}\Big|\sum_{\bh\in\cH_m} (\pi_{m+1}(\ba\given \bh)-\tilde \pi_{m+1}(\ba\given \bh))\Pr(\bh;\bpi)\Big|\\
        =&\sum_{\ba\in \cA}\Big|\inner{\pi_{m+1}(\ba\given \bh)-\tilde \pi_{m+1}(\ba\given \bh)}{\Pr(\bh;\bpi)}_{\bh\in\cH_m}\Big|\\
        \leq& \sum_{\ba\in \cA}\nbr{(\pi_{m+1}(\ba\given \bh)-\tilde \pi_{m+1}(\ba\given \bh))_{\bh\in\cH_m}}_\infty\cdot \nbr{(\Pr(\bh;\bpi))_{\bh\in\cH_m}}_1\\
        \leq&  \sum_{\ba\in \cA}\nbr{(\pi_{m+1}(\ba\given \bh)-\tilde \pi_{m+1}(\ba\given \bh))_{\bh\in\cH_m}}_\infty\\
        \leq& |\cA|\cdot \nbr{(\pi_{m+1}(\ba\given \bh)-\tilde \pi_{m+1}(\ba\given \bh))_{\ba\in\cA,\bh\in\cH_m}}_\infty.
    \end{align*}

    For the term \eqref{eq:eq213-second}, for a fixed $\ba\in\cA$, we have
    \begin{align*}
        &\Big|\sum_{\bh\in\cH_m}\tilde \pi_{m+1}(\ba\given \bh) \Big(\Pr(\bh;\bpi)-\Pr(\bh;\tilde\bpi)\Big)\Big|\\
        =&\Big|\sum_{\bh\in\cH_m}\tilde \pi_{m+1}(\ba\given \bh) \Big(\prod_{s=1}^{m} \pi_s(\bh_s\given \bh_{1:s-1})-\prod_{s=1}^{m} \tilde \pi_s(\bh_s\given \bh_{1:s-1})\Big)\Big|\\
        =& \Big|\sum_{\bh\in\cH_m}\tilde \pi_{m+1}(\ba\given \bh) \sum_{k=1}^{m}\Big(\prod_{s=1}^{m} \bar \pi_s^{k-1}(\bh_s\given \bh_{1:s-1})-\prod_{s=1}^{m} \bar \pi_s^{k}(\bh_s\given \bh_{1:s-1})\Big)\Big|\\
        \leq& \sum_{k=1}^{m}\Big|\sum_{\bh\in\cH_m}\tilde \pi_{m+1}(\ba\given \bh) \Big(\prod_{s=1}^{m} \bar \pi_s^{k-1}(\bh_s\given \bh_{1:s-1})-\prod_{s=1}^{m} \bar \pi_s^{k}(\bh_s\given \bh_{1:s-1})\Big)\Big|
    \end{align*}
    where
    \begin{align}
        \bar \bpi^{k}_s=
        \begin{cases}
            \bpi_s&s>k\\
            \tilde \bpi_s&s\leq k.
        \end{cases}
    \end{align}

For any $k\in\{1,2,...,m\}$, we have
\begin{align*}
    &\Big|\sum_{\bh\in\cH_m}\tilde \pi_{m+1}(\ba\given \bh) \Big(\prod_{s=1}^{m} \bar \pi_s^{k-1}(\bh_s\given \bh_{1:s-1})-\prod_{s=1}^{m} \bar \pi_s^{k}(\bh_s\given \bh_{1:s-1})\Big)\Big|\\
    =& \Big|\sum_{\bh\in\cH_m}\tilde \pi_{m+1}(\ba\given \bh) \prod_{s=1,s\not=k}^m \bar \pi_s^{k}(\bh_s\given \bh_{1:s-1})\Big(\pi_k(\bh_k\given \bh_{1:k-1})- \tilde \pi_k(\bh_k\given \bh_{1:k-1})\Big)\Big|\\
    =&\Big|\inner{\prod_{s=1,s\not=k}^m \bar \pi_s^{k}(\bh_s\given \bh_{1:s-1})}{\tilde \pi_{m+1}(\ba\given \bh)(\pi_k(\bh_k\given \bh_{1:k-1})- \tilde \pi_k(\bh_k\given \bh_{1:k-1}))}_{\bh\in\cH_m}\Big|\\
    \leq&\nbr{\rbr{\prod_{s=1,s\not=k}^m \bar \pi_s^{k}(\bh_s\given \bh_{1:s-1})}_{\bh\in\cH_m}}_1 \cdot \nbr{\rbr{\tilde \pi_{m+1}(\ba\given \bh)(\pi_k(\bh_k\given \bh_{1:k-1})- \tilde \pi_k(\bh_k\given \bh_{1:k-1}))}_{\bh\in\cH_m}}_\infty\\
    =& |\cA|\cdot \nbr{\rbr{\pi_k(\bh_k\given \bh_{1:k-1})- \tilde \pi_k(\bh_k\given \bh_{1:k-1})}_{\bh\in\cH_m}}_\infty.
\end{align*}
Therefore, term \eqref{eq:eq213-second} is bounded by
\begin{align*}
   \sum_{\ba\in \cA}\Big|\sum_{\bh\in\cH_m} (\Pr(\bh;\bpi)-\Pr(\bh;\tilde\bpi))\tilde \pi_{m+1}(\ba\given \bh) \Big|\leq& |\cA|^2\cdot\sum_{k=1}^m \nbr{\rbr{\pi_k(\bh_k\given \bh_{1:k-1})- \tilde \pi_k(\bh_k\given \bh_{1:k-1})}_{\bh\in\cH_m}}_\infty.
\end{align*}

Finally,
\begin{align*}
    &|f^m(\bpi)-f^m(\tilde\bpi)|\\
    \leq& |\cA|\cdot \nbr{(\pi_{m+1}(\ba\given \bh)-\tilde \pi_{m+1}(\ba\given \bh))_{\ba\in\cA,\bh\in\cH_m}}_\infty+|\cA|^2\cdot\sum_{k=1}^m \nbr{\rbr{\pi_k(\bh_k\given \bh_{1:k-1})- \tilde \pi_k(\bh_k\given \bh_{1:k-1})}_{\bh\in\cH_m}}_\infty\\
    \leq&|\cA|^2\cdot\sum_{k=1}^{m+1} \nbr{\bpi_k-\tilde \bpi_k}_\infty.
\end{align*}
For the second part, since for any $\bh\in\cH,\ba\in\cA,k=1,2,...,m+1$, we have
\begin{align*}
    |\pi_k(\ba\given \bh)-\tilde\pi_k(\ba\given \bh)|=&\abr{\prod_{i=1}^N \pi_k\ui(a_i\given \bh)-\prod_{i=1}^N \tilde\pi_k\ui(a_i\given \bh)}\\
    =&\abr{\sum_{i=1}^N \prod_{j=1}^{i-1} \tilde \pi_k^{(j)}(a_j\given \bh) \prod_{j=i+1}^N \pi_k^{(j)}(a_j\given \bh) \rbr{\pi_k\ui(a_i\given \bh)-\tilde\pi_k\ui(a_i\given \bh)}}\\
    \leq& \sum_{i=1}^N \prod_{j=1}^{i-1} \tilde \pi_k^{(j)}(a_j\given \bh) \prod_{j=i+1}^N \pi_k^{(j)}(a_j\given \bh) \abr{\pi_k\ui(a_i\given \bh)-\tilde\pi_k\ui(a_i\given \bh)}\\
    \leq& \sum_{i=1}^N \abr{\pi_k\ui(a_i\given \bh)-\tilde\pi_k\ui(a_i\given \bh)}.
\end{align*}
As a consequence,
\begin{equation*}
    |f^m(\bpi)-f^m(\tilde\bpi)|\leq C_{\rm Lips}\sum_{t=1}^{m+1}\nbr{\bpi_t-\tilde \bpi_t}_\infty\leq C_{\rm Lips}\sum_{i=1}^N\sum_{t=1}^{m+1}\nbr{\bpi_t\ui-\tilde \bpi_t\ui}_\infty.
\end{equation*}
where $C_{\rm Lips} \coloneqq  |\cA|^2$. \qed

\subsection{Lemma for Markov Game}
\label{appendix:original-to-markov-value-proof}

\begin{lemma}
\label{lemma:markov-connect}
    Suppose Condition \ref{assumption:2-forgetful} is satisfied for all players. The expected time-average loss of the induced Markov game defined in Definition \ref{def:induced-MDP} always exists and does not depend on the initial distribution.

    \proof

    Fix any player $i\in\cN$. Without loss of generality, we may restrict attention to policies $\pi^{(i)}$ that have full support conditioned on every history. Indeed, by \Cref{assumption:2-forgetful}, if $\pi^{(i)}\rbr{a_i\given \bh}=0$ for some $a_i\in\cA_i$ and $\bh\in\cH$, then necessarily $\pi^{(i)}\rbr{a_i\given \bh'}=0$ for all $\bh'\in\cH$. In this case, action $a_i$ is never taken (from any history), so we can remove $a_i$ from $\cA_i$ without changing the limiting time-average loss, because any state that involves $a_i$ is transient and has zero probability of being visited on average in the limit.

    Moreover, if every player uses a policy with full support conditional on any history, then every joint action is selected with positive probability whenever it is available. Consequently, the induced Markov chain is irreducible. Therefore, the time-average loss exists and is independent of the initial distribution \citep{norris1998markov}. \qed
\end{lemma}

\begin{lemma}

\label{lemma:original-to-markov-value}
Suppose Condition \ref{assumption:2-forgetful} is satisfied for all players. For any $K,M\in\NN$ and any strategy profile vector $\bpi=(\bpi_1,\bpi_2,...,\bpi_{K+1})$ of length $K+1$, when $\gamma\leq \frac{1}{2(N+2)}$, we have
\small
    \begin{align*}
        &|f^K(\bpi)- f^{K,M}(\bpi)|
        \leq  2C_K^\gamma+2C_{\rm Lips} N (K+1)\gamma^{M+1}
    \end{align*}
    \normalsize
    where $C_K^\gamma=(2N+1)^{K+1}\gamma^{K+1}$, $C_{\rm Lips}=\abr{\cA}^2$, and,
    \begin{align*}
        &f^{K,M}(\bpi) \coloneqq \frac{1} {|\cH_M|}\sum_{\bh\in\cH_M} f^K(\bpi^M_{1:K+1}\given\bh),\quad
        &\pi_k^M(\ba\given\bh) \coloneqq \pi_k(\ba\given\bh_{L(\bh)-M+1:L(\bh)}).
    \end{align*}
    If $L(\bh)<M$, the suffix $\bh_{L(\bh)-M+1:L(\bh)}$ is understood as the whole available history $\bh$.
\proof

Define $\bar f^{K,M}(\bpi) \coloneqq \frac{1}{|\cH_M|}\sum_{\bh_0\in\cH_M} f^K(\bpi\given \bh_0)$ for convenience. By Lemma \ref{lemma:approximation}, applied with the arbitrary initial history $\bh_0$, we have
\begin{align*}
    &\abr{f^K(\bpi)-\bar f^{K,M}(\bpi)}\\
    \leq&\frac{1}{|\cH_M|}\sum_{\bh_0\in\cH_M}\sum_{\ba\in\cA}\abr{\sum_{\bh\in\cH_K}\Pr((\bh,\ba);\bpi)-\sum_{\bh\in\cH_K}\Pr((\bh,\ba)\given\bh_0;\bpi)}\\
    \leq&2C_K^\gamma .
\end{align*}

Lemma \ref{lemma:Lipschitz} extends to $f^K(\bpi\given \bh_0)$ by defining $\tilde \bpi$ with $\tilde\pi_s(\ba\given \bh)=\pi_s(\ba\given(\bh_0,\bh))$ for all $\ba\in\cA$, $s=1,2,...,K+1$, and $\bh\in\cH_{s-1}$. Therefore,
\begin{align*}
    &\abr{\bar f^{K,M}(\bpi)-f^{K,M}(\bpi)}\\
    \leq&\frac{1}{|\cH_M|}\sum_{\bh_0\in\cH_M}\abr{f^K(\bpi\given \bh_0)-f^K(\bpi^M\given \bh_0)}\\
    \leq& C_{\rm Lips}\max_{\bh_0\in\cH_M}\sum_{s=1}^{K+1}\sum_{i=1}^N
    \max_{\bh\in\cH_{s-1},a_i\in\cA_i}
    \abr{\pi\ui_s(a_i\given(\bh_0,\bh))-\pi\ui_s(a_i\given(\bh_0,\bh)_{L((\bh_0,\bh))-M+1:L((\bh_0,\bh))})}\\
    \leq& C_{\rm Lips}N(K+1)\gamma^{M+1}.
\end{align*}
In the last line, the two histories compared inside each probability share their most recent suffix of length $M$. Hence, Condition \ref{assumption:2-forgetful} gives a multiplicative ratio in $[1-\gamma^{M+1},(1-\gamma^{M+1})^{-1}]$. Putting all pieces together finishes the proof. \qed
\end{lemma}

\begin{proof}[Proof of Lemma \ref{lemma:sublinear-variation-q}]
An observation is that for $a_1,b_1\geq 0,a_2,b_2\geq c$ and $c>0$, we have
\begin{align*}
    |\frac{a_1}{a_2}-\frac{b_1}{b_2}|=|\frac{a_1b_2-b_1a_2}{a_2b_2}|\leq \frac{a_1|b_2-a_2|}{a_2b_2}+\frac{a_2|a_1-b_1|}{a_2b_2}\leq \frac{a_1}{a_2 c}|b_2-a_2|+\frac{|a_1-b_1|}{c}.
\end{align*}

So, for any $\bh\in\cH_M,\ba\in\cA$, we have
\begin{align*}
    &\abr{\pi_{2}(\ba\given \bh)-\pi_1(\ba\given \bh)}\\
    =&\abr{\frac{q^{\bpi_2}(\bh,\ba)}{\sum_{\ba'\in\cA} q^{\bpi_2}(\bh,\ba')}-\frac{q^{\bpi_1}(\bh,\ba)}{\sum_{\ba'\in\cA} q^{\bpi_1}(\bh,\ba')}}\\
    \leq& \frac{q^{\bpi_2}(\bh,\ba)}{c\sum_{\ba'\in\cA} q^{\bpi_2}(\bh,\ba')}\abr{\sum_{\ba'\in\cA} q^{\bpi_2}(\bh,\ba')-\sum_{\ba'\in\cA} q^{\bpi_1}(\bh,\ba')}+\frac{1}{c}|q^{\bpi_2}(\bh,\ba)-q^{\bpi_1}(\bh,\ba)|\\
    =&\frac{\pi_2(\ba\given \bh)}{c}\abr{\sum_{\ba'\in\cA} q^{\bpi_2}(\bh,\ba')-\sum_{\ba'\in\cA} q^{\bpi_1}(\bh,\ba')}+\frac{1}{c}|q^{\bpi_2}(\bh,\ba)-q^{\bpi_1}(\bh,\ba)|\\
    \leq&\frac{2|\cA|}{c}\nbr{\bq^{\pi_2}-\bq^{\pi_1}}_\infty.\qedhere
\end{align*}
\end{proof}

\subsection{Other Lemmas}
\label{appendix:other-lemma}

\begin{proof}[Proof of Lemma \ref{lemma:accumulated-variation-bound}]

\begin{align*}
    \sum_{t=1}^T \sum_{s=\max\cbr{t-K, 1}}^{t-1} \nbr{\bx_t-\bx_s}_p\leq& \sum_{t=1}^T \sum_{s=\max\cbr{t-K, 1}}^{t-1} \sum_{s'=s}^{t-1} \nbr{\bx_{s'}-\bx_{s'+1}}_p\\
    \leq&\sum_{t=2}^T (1+2+...+K)\nbr{\bx_t-\bx_{t-1}}_p\\
    \leq& K^2 \sum_{t=2}^T \nbr{\bx_t-\bx_{t-1}}_p.\qedhere
\end{align*} 
\end{proof}

\subsection{Projected Gradient Descent ({\tt PGD})}

\begin{algorithm}
\caption{Projected Gradient Descent}
\begin{algorithmic}[1]
\State{$\eta$ is the learning rate.}
\State{Initialize $\bx_1\in\cX$ where $\cX$ is the convex set that $\bx$ lies in}
\For{$t=1,2,...,T$}
\State{Propose $\bx_t$.}
\State{Receive the loss $\inner{\cL_t}{\bx_t}.$}
\begin{align}
    &\bx_{t+1}\leftarrow\Proj{\cX}{\bx_t-\eta\cL_t}\label{eq:PGD-update-rule}
\end{align}
\EndFor
\end{algorithmic}
\end{algorithm}

Here we provide the upper bound on the regret $\sum_{t=1}^T \inner{\cL_t}{\bx_t}-\min_{\hat\bx_{1:T},\hat\bx_t\in\cX}\sum_{t=1}^T \inner{\cL_t}{\hat\bx_t}$. We adapted \cite[Theorem 10]{zhao2022non-dynamic-policy-regret} here. %

\begin{lemma}[Adapted from Theorem 10 in \cite{zhao2022non-dynamic-policy-regret}]
\label{lemma:PGD-upper-bound}
Consider the update-rule \Cref{eq:PGD-update-rule}. For any sequence $\hat\bx_1,\hat\bx_2,...,\hat\bx_T$ satisfying $\hat\bx_t\in\cX$, we have
\begin{align}
    \sum_{t=1}^T \inner{\cL_t}{\bx_t}-\sum_{t=1}^T \inner{\cL_t}{\hat\bx_t}\leq \frac{\eta}{2}\sum_{t=1}^T\nbr{\cL_t}^2+\frac{1}{2\eta}\nbr{\bx_1-\hat\bx_1}^2+\frac{D_1}{\eta} \sum_{t=2}^T \nbr{\hat\bx_{t-1}-\hat\bx_t}_\infty
\end{align}
where $D_1 \coloneqq \max_{\bx,\bx'\in\cX} \nbr{\bx-\bx'}_1$.

\proof
Notice that
\begin{align*}
    \nbr{\bx_{t+1}-\hat\bx_t}^2=&\nbr{{\rm Proj}_{\cX}(\bx_t-\eta\cL_t)-\hat\bx_t}^2\\
    \leq& \nbr{\bx_t-\eta\cL_t-\hat\bx_t}^2\\
    =&\eta^2\nbr{\cL_t}^2-2\eta\inner{\cL_t}{\bx_t-\hat\bx_t}+\nbr{\bx_t-\hat\bx_t}^2
\end{align*}
where ${\rm Proj}$ denotes the projection with respect to the $L2$ norm. The second line is because $\cX$ is convex.

Therefore, by rearranging the terms, we have
\begin{align*}
    \inner{\cL_t}{\bx_t-\hat\bx_t}\leq \frac{\eta}{2}\nbr{\cL_t}^2+\frac{1}{2\eta}(\nbr{\bx_t-\hat\bx_t}^2-\nbr{\bx_{t+1}-\hat\bx_t}^2).
\end{align*}
The summation of the second term here can be bounded by
\begin{align*}
    \sum_{t=1}^T(\nbr{\bx_t-\hat\bx_t}^2-\nbr{\bx_{t+1}-\hat\bx_t}^2)\leq& \sum_{t=1}^T \nbr{\bx_t-\hat\bx_t}^2- \sum_{t=2}^T \nbr{\bx_t-\hat\bx_{t-1}}^2\\
    =&\nbr{\bx_1-\hat\bx_1}^2+\sum_{t=2}^T (\nbr{\bx_t-\hat\bx_t}^2-\nbr{\bx_t-\hat\bx_{t-1}}^2)\\
    =&\nbr{\bx_1-\hat\bx_1}^2+\sum_{t=2}^T \inner{\hat\bx_{t-1}-\hat\bx_t}{2\bx_t-\hat\bx_t-\hat\bx_{t-1}}\\
    \leq& \nbr{\bx_1-\hat\bx_1}^2+\sum_{t=2}^T \nbr{\hat\bx_{t-1}-\hat\bx_t}_\infty\cdot\nbr{2\bx_t-\hat\bx_t-\hat\bx_{t-1}}_1\\
    \leq& \nbr{\bx_1-\hat\bx_1}^2+2D_1\sum_{t=2}^T \nbr{\hat\bx_{t-1}-\hat\bx_t}_\infty
\end{align*}
where $D_1 \coloneqq \max_{\bx,\bx'\in\cX} \nbr{\bx-\bx'}_1$.

Therefore,
\begin{equation*}
    \sum_{t=1}^T \inner{\bx_t}{\cL_t}-\sum_{t=1}^T \inner{\hat\bx_t}{\cL_t}\leq \frac{\eta}{2}\sum_{t=1}^T\nbr{\cL_t}^2+\frac{1}{2\eta}\nbr{\bx_1-\hat\bx_1}^2+\frac{D_1}{\eta} \sum_{t=2}^T \nbr{\hat\bx_{t-1}-\hat\bx_t}_\infty.\qedhere
\end{equation*}

\end{lemma}

\subsection{Projected Gradient Descent with Time-varying Constraints}

\begin{algorithm}
\caption{Projected Gradient Descent with Time-varying Constraints \citep{cao2018online-constrained-time-varying}}
\label{algorithm:PGD-with-constraint}
\begin{algorithmic}[1]
\State{$\eta$ is the learning rate, $\delta$ is a non-negative constant hyper-parameter.}
\State{Initialize $\bx_1\in\cX, \blambda_1=\zero$.}
\For{$t=1,2,...,T$}
\State{Propose $\bx_t$.}
\State{Receive the loss $\inner{\cL_t}{\bx_t}$ and constraints $\bg_t(\bx_t) \coloneqq \cbr{g_t^1(\bx_t),g_t^2(\bx_t),...,g_t^k(\bx_t)}\preceq \zero$}
\begin{align}
    &\bx_{t+1}\leftarrow\Proj{\cX}{\bx_t-\eta\rbr{\cL_t+\sum_{i=1}^k \lambda_t^i \nabla g_t^i(\bx_t)}}\label{eq:update-constraint-variable}\\
    &\blambda_{t+1}\leftarrow \Proj{\RR_+^k}{\blambda_t + \eta \rbr{\bg_t(\bx_t)-\delta\eta\blambda_t}}.\label{eq:update-constraint-lambda}
\end{align}
\State{\Comment{$\RR_+^k \coloneqq \cbr{\bx\in\RR^k\given\bx\succeq \zero}$.}}
\EndFor
\end{algorithmic}
\end{algorithm}

\begin{lemma}[Adapted from Theorem 1 in \cite{cao2018online-constrained-time-varying}]
    \label{lemma:PGD-constraint-guarantee}
    Consider Algorithm \ref{algorithm:PGD-with-constraint}. For any sequence of $\hat\bx_1,\hat\bx_2,...,\hat\bx_T$ satisfying that $\forall t\in\cbr{1,2,...,T}, \bg_t(\hat\bx_t)\preceq \zero$, when $\eta=\sqrt{\frac{P_T}{T}}$ and $\delta=2CG+(1+k)G^2+1$, we have 
    \begin{align}
        &\sum_{t=1}^T \inner{\cL_t}{\bx_t}-\sum_{t=1}^T \inner{\cL_t}{\hat\bx_t}+C\sum_{t=2}^T \nbr{\bx_t-\bx_{t-1}}_\infty\leq \frac{5R^2}{2}\sqrt{\frac{T}{P_T}}+\rbr{R+\frac{k+1}{2}G^2+D^2+CG(k+1)}\sqrt{TP_T}\\
        &\forall i\in\{1,2,...,k\},~~~~~ \sum_{t=1}^T g_t^i(\bx_t)\leq \sqrt{2\rbr{(2CG+(k+1)G^2+1)\sqrt{TP_T}+\sqrt{\frac{T}{P_T}}}}\notag\\
        &~~~~~~~~~~~~~~~~~~~~~~\times \sqrt{FT+\frac{5R^2}{2}\sqrt{\frac{T}{P_T}}+\rbr{R+\frac{k+1}{2}G^2+D^2+CG(k+1)}\sqrt{TP_T}}
    \end{align}
    where
    \begin{align}
        &P_T=\sum_{t=1}^{T-1} \nbr{\hat\bx_t-\hat\bx_{t+1}}\\
        &R=\max_{\bx\in\cX} \nbr{\bx}\\
        &G=\max_{t=1,2,...,T}\cbr{\nbr{\cL_t}, \max_{\bx\in\cX,i=1,2,...,k} \nbr{\nabla g_t^i(\bx)}}\\
        &D=\max_{t=1,2,...,T,\bx\in\cX} \nbr{\bg_t(\bx)}\\
        &F=\max_{\bx,\bx'\in\cX} |\inner{\cL_t}{\bx-\bx'}|
    \end{align}
    and $C>0$ is some arbitrary constant.
    \proof
    Different from \cite[Theorem 1]{cao2018online-constrained-time-varying}, we have an additional $C\sum_{t=2}^T \nbr{\bx_t-\bx_{t-1}}_\infty$ to bound. Firstly, by the update rule Eq. \eqref{eq:update-constraint-variable} and the non-expansiveness of the Euclidean projection, for every $t\geq 2$ we have
    \begin{align*}
        \nbr{\bx_t-\bx_{t-1}}_\infty
        &\leq \nbr{\bx_t-\bx_{t-1}}\\
        &=\nbr{\Proj{\cX}{\bx_{t-1}-\eta\rbr{\cL_{t-1}+\sum_{i=1}^k \lambda_{t-1}^i \nabla g_{t-1}^i(\bx_{t-1})}}-\Proj{\cX}{\bx_{t-1}}}\\
        &\leq\eta\nbr{\cL_{t-1}+\sum_{i=1}^k \lambda_{t-1}^i \nabla g_{t-1}^i(\bx_{t-1})}\\
        &\leq \eta G+\eta G\sum_{i=1}^k \lambda_{t-1}^i.
    \end{align*}
    Note that
    \begin{align*}
        \lambda_{t-1}^i\leq (\lambda_{t-1}^i)^2+1.
    \end{align*}
    So, $\nbr{\bx_t-\bx_{t-1}}_\infty\leq \eta G(k+1)+\eta G\nbr{\blambda_{t-1}}^2$.

    By abusing notation and writing $\cL_t(\bx,\blambda) \coloneqq \inner{\cL_t}{\bx}+\sum_{i=1}^k \lambda^i g_t^i(\bx)-\frac{\delta\eta}{2}\nbr{\blambda}^2$, we have the following lemma.
    \begin{lemma}[Lemma 3 in \cite{cao2018online-constrained-time-varying}]
        \label{lemma:lemma-3-in-cao-constraint}
         For any $\blambda\succeq \zero$ and $\hat\bx_{1:T}$, we have
        \begin{align}
            &\sum_{t=1}^T \rbr{\cL_t(\bx_t,\blambda)-\cL_t(\hat\bx_t,\blambda_t)}\\
            \leq& \frac{1}{2\eta}\rbr{5 R^2+2R P_T+\nbr{\blambda}^2}+\frac{\eta T}{2}\rbr{(k+1)G^2+2D^2}+\frac{\eta}{2}\sbr{(1+k)G^2+2\delta^2\eta^2}\sum_{t=1}^T \nbr{\blambda_t}^2
        \end{align}
        where
        \begin{align}
            P_T=\sum_{t=2}^T\nbr{\hat\bx_t-\hat\bx_{t-1}}.
        \end{align}
    \end{lemma}
    So, by Lemma \ref{lemma:lemma-3-in-cao-constraint}, we have
    \begin{align*}
        &\sum_{t=1}^T \inner{\cL_t}{\bx_t-\hat\bx_t}+\sum_{i=1}^k \sum_{t=1}^T \rbr{\lambda^i g_t^i(\bx_t)-\lambda_t^i g_t^i(\hat\bx_t)}-\frac{\delta\eta T}{2}\nbr{\blambda}^2+C \sum_{t=2}^T \nbr{\bx_t-\bx_{t-1}}_\infty\\
        \leq& \frac{\eta}{2}\rbr{2CG+(1+k)G^2+2\delta^2\eta^2-\delta}\sum_{t=1}^T \nbr{\blambda_t}^2+\frac{1}{2\eta}\rbr{5R^2+2R P_T+\nbr{\blambda}^2}+\frac{\eta T}{2}\rbr{(k+1)G^2+2D^2}\\
        \leq& \frac{1}{2\eta}\rbr{5R^2+2R P_T+\nbr{\blambda}^2}+\frac{\eta T}{2}\rbr{(k+1)G^2+2D^2+2CG(k+1)}
    \end{align*}
    where the last line is by choosing $\delta=2CG+(1+k)G^2+1$ and $T$ is large enough so that $\eta^2=\frac{P_T}{T}\leq \frac{1}{2\delta^2}$. Therefore, by rearranging the terms, we have
    \begin{align*}
        &\sum_{t=1}^T \inner{\cL_t}{\bx_t-\hat\bx_t}+\sum_{i=1}^k \rbr{ \lambda^i\sum_{t=1}^T g_t^i(\bx_t)-\rbr{\frac{\delta\eta T}{2}+\frac{1}{2\eta} }(\lambda^i)^2}+C \sum_{t=2}^T \nbr{\bx_t-\bx_{t-1}}_\infty\\
        \leq& \sum_{i=1}^k\sum_{t=1}^T \lambda_t^i g_t^i(\hat\bx_t) + \frac{1}{2\eta}\rbr{5R^2+2R P_T}+\frac{\eta T}{2}\rbr{(k+1)G^2+2D^2+2CG(k+1)}\\
        \leq&\frac{1}{2\eta}\rbr{5R^2+2R P_T}+\frac{\eta T}{2}\rbr{(k+1)G^2+2D^2+2CG(k+1)}
    \end{align*}
    where the last line is by definition that $\hat\bx_t$ satisfies $g_t^i(\hat\bx_t)\leq 0$ for any $i=1,2,...,k$ and $\lambda_t^i\geq 0$ for any $i=1,2,...,k$ and $t=1,2,...,T$. By choosing $\lambda^i=\frac{[\sum_{t=1}^T g_t^i(\bx_t)]^+}{\delta\eta T+\frac{1}{\eta}}$ where $[x]^+ \coloneqq \max\cbr{x,0}$, we have
    \begin{align*}
        &\sum_{t=1}^T \inner{\cL_t}{\bx_t-\hat\bx_t}+\sum_{i=1}^k \frac{\rbr{[\sum_{t=1}^T g_t^i(\bx_t)]^+}^2}{2(\delta\eta T+\frac{1}{\eta})}+C \sum_{t=2}^T \nbr{\bx_t-\bx_{t-1}}_\infty\\
        \leq& \frac{1}{2\eta}\rbr{5R^2+2R P_T}+\frac{\eta T}{2}\rbr{(k+1)G^2+2D^2+2CG(k+1)}.
    \end{align*}
    So,
    \begin{align*}
        &\sum_{t=1}^T \inner{\cL_t}{\bx_t-\hat\bx_t}+C \sum_{t=2}^T \nbr{\bx_t-\bx_{t-1}}_\infty\\
        \leq& \frac{1}{2\eta}\rbr{5R^2+2R P_T}+\frac{\eta T}{2}\rbr{(k+1)G^2+2D^2+2CG(k+1)}\\
        =&\frac{5R^2}{2}\sqrt{\frac{T}{P_T}}+\frac{2R+(k+1)G^2+2D^2+2CG(k+1)}{2}\sqrt{T P_T}.
    \end{align*}
    By definition, we have $\sum_{t=1}^T \inner{\cL_t}{\bx_t-\hat\bx_t}\geq -FT$. Then,
    \begin{align*}
        -FT+\sum_{i=1}^k \frac{\rbr{[\sum_{t=1}^T g_t^i(\bx_t)]^+}^2}{2(\delta\eta T+\frac{1}{\eta})}\leq& \frac{5R^2}{2}\sqrt{\frac{T}{P_T}}+\frac{2R+(k+1)G^2+2D^2+2CG(k+1)}{2}\sqrt{T P_T}.
    \end{align*}
    Lastly, for any $i=1,2,...,k$, we have
    \begin{align*}
        \sum_{t=1}^T g_t^i(\bx_t)\leq& \sqrt{2FT+5R^2 \sqrt{\frac{T}{P_T}}+\rbr{2R+(k+1)G^2+2D^2+2CG(k+1)} \sqrt{T P_T}}\\
        &\times \sqrt{(2CG+(1+k)G^2+1)\sqrt{T P_T}+\sqrt{\frac{T}{P_T}}}.\qedhere
    \end{align*}

\end{lemma}

\section{Auxiliary Lemmas for the Induced Markov Game}
\label{appendix:game-setting}

In this section, we assume all agents have a bounded memory $M$ and we are considering the Induced Markov Game defined in Definition \ref{def:induced-MDP}.

\subsection{Milder Constraint}
\label{subsection:milder-constraint}

It is easy to see that Condition \ref{assumption:2-forgetful-prime} implies the following condition, Condition \ref{assumption:L1-forgetful}. Therefore, in this section, we will show that when all players satisfy Condition \ref{assumption:L1-forgetful}, we can use the expected average loss of an infinite-horizon Markov game to approximate the expected loss of the repeated game at each timestep.

\begin{condition}
\label{assumption:L1-forgetful}
The strategy $\bpi_{1:T}\ui$ is not too different when observing different histories. That is, for any timestep $t$, we have
    \begin{align}
    &\forall \bar\bh,\tilde\bh\in\cH,~~~~~\frac{1}{2}\nbr{\bpi_t\ui(\cdot\given\bar\bh)-\bpi_t\ui(\cdot|\tilde\bh)}_1\leq 1-\gamma
    \end{align}
    for some constant $\gamma\in(0,1]$.
\end{condition}

\begin{proposition}
\label{proposition:weaker-assumption-on-memory}
    Condition \ref{assumption:L1-forgetful} with parameter $1-\gamma$ can be inferred from Condition \ref{assumption:2-forgetful} with parameter $\gamma$.
    \proof
    \begin{align*}
        &\nbr{\bpi\ui(\cdot\given \bh_1)-\bpi\ui(\cdot\given \bh_2)}_1\\
        =&\sum_{a_i\in\cA_i}|\pi\ui(a_i\given \bh_1)-\pi\ui(a_i\given \bh_2)|\\
        =& \sum_{a_i\in\cA_i}\max\{0, \pi\ui(a_i\given \bh_1)-\pi\ui(a_i\given \bh_2)\}+\sum_{a_i\in\cA_i}\max\{0, \pi\ui(a_i\given \bh_2)-\pi\ui(a_i\given \bh_1)\}\\
        \leq& \gamma \sum_{a_i\in\cA_i} \pi\ui(a_i\given \bh_1)+\gamma \sum_{a_i\in\cA_i} \pi\ui(a_i\given \bh_2)\\
        =&2\gamma.\qedhere
    \end{align*}
\end{proposition}

A direct corollary of Condition \ref{assumption:L1-forgetful} is that
\begin{corollary}
\label{corollary:L1-forgetful}
When Condition \ref{assumption:L1-forgetful} is satisfied, we have
    \begin{align}
    \forall \bh,\bh'\in \cH_M, \exists a_i\in\cA_i,~~~\pi\ui(a_i\given \bh),\pi\ui(a_i\given \bh')\geq \frac{\gamma}{|\cA_i|}.
\end{align}
\end{corollary}

For the Markov chain\footnote{By taking $\bpi$ as constant and merging it into the transition probability, we get a Markov chain from the induced Markov game defined in Definition \ref{def:induced-MDP}.} induced by $\bpi$, let $\cP^{\bpi}$ be the transition matrix. Then we have the following lemma.
\begin{lemma}
When all players satisfy Condition \ref{assumption:L1-forgetful}, for the transition matrix $\cP^{\bpi}$ induced by the repeated matrix game, we have
\label{lemma:common-successor-state}
    \begin{align*}
    \forall \bh_1,\bh_2\in \cH_M, \exists \bh_3\in\cH_M,~~~(\cP^{\bpi})^M_{\bh_1,\bh_3}, (\cP^{\bpi})^M_{\bh_2,\bh_3}\geq (\frac{\gamma^N}{|\cA|})^M.
\end{align*}
\proof
For any $\bh_1,\bh_2\in\cH_M$, by Corollary \ref{corollary:L1-forgetful}, for every $i\in[N]$ there exists $a_i\in\cA_i$, so that $\pi\ui(a_i\given \bh_1),\pi\ui(a_i\given \bh_2)\geq \frac{\gamma}{|\cA_i|}$. Therefore, for $\ba=(a_1,a_2,...,a_N)$, we have
\begin{align*}
    \pi(\ba\given \bh_1)=\prod_{i=1}^N \pi\ui(a_i\given \bh_1)\geq \frac{\gamma^N}{|\cA|}.
\end{align*}
Similarly, we have $\pi(\ba\given \bh_2)\geq \frac{\gamma^N}{|\cA|}$. Similarly, there is some $\ba'\in\cA$ so that $\pi(\ba'|(\bh_{1,2:M}, \ba)),\pi(\ba'|(\bh_{2,2:M}, \ba))\geq \frac{\gamma^N}{|\cA|}$. Finally, there is some $\bh_3\in\cH_M$ so that
\begin{align*}
    (\cP^{\bpi})^M_{\bh_1,\bh_3}=\prod_{t=1}^M \pi(\bh_{3,t}|(\bh_{1,t:M}, \bh_{3,1:t-1}))\geq (\frac{\gamma^N}{|\cA|})^M
\end{align*}
and $(\cP^{\bpi})^M_{\bh_2,\bh_3}\geq (\frac{\gamma^N}{|\cA|})^M$ similarly. \qed
\end{lemma}

\subsection{Fast Mixing}
\label{appendix:fast-mixing}

Firstly, we will prove a set of results in the Markov chain with transition matrix $(\cP^{\bpi})^M$. For simplicity, we model it as a directed graph, with vertices as $\cH_M$. There exists a directed edge $\bh_1\to\bh_2$ if and only if $(\cP^{\bpi})^M_{\bh_1,\bh_2}\geq (\frac{\gamma^N}{|\cA|})^M$.

Note that by Lemma \ref{lemma:common-successor-state}, for every two vertices $\bh_1,\bh_2$ in the graph, there exists a vertex $\bh_3$ (might be equal to $\bh_1$ or $\bh_2$), which is a common successor of $\bh_1$ and $\bh_2$. That is, there exist edges $\bh_1\to\bh_3$ and $\bh_2\to\bh_3$.

Firstly, we would like to prove that there exists a vertex $\bh_0\in\cH_M$, which satisfies that every other vertex could reach it within $O(\log_2 |\cH_M|)$ steps.

\begin{lemma}[Short Connectivity]
\label{lemma:log-path}
    There exists a vertex $\bh_0\in\cH_M$, so that every other vertex can reach it within $\log_2|\cH_M|+1$ steps.
    \proof
    We will merge the vertices into rooted trees where all vertices in the tree can reach the root. Initially, we have $\cH_M$ rooted trees and each one is a single vertex. Then, in each episode, we will merge them as follows.

    Let $\cT$ be the set of rooted trees. We will first divide it into $\floor{|\cT|/2}$ pairs arbitrarily. Then, for each pair of rooted trees, say rooted trees with root $\bh_1,\bh_2$, we will try to merge them into one.%
    Then, by the condition of the graph, we know that there exists $\bh_3$ satisfying that $\bh_1,\bh_2$ are both connected to $\bh_3$. Then, there are two cases.
    \begin{itemize}
        \item[(i)] $\bh_3\in\{\bh_1, \bh_2\}$. Without loss of generality, we assume $\bh_3=\bh_1$. Then, we can connect $\bh_2$ to $\bh_1$ and thus merge two rooted trees into one. The depth of all nodes in the new tree is no more than $\max\cbr{{\rm Depth}(\bh_1),{\rm Depth}(\bh_2)}+1$ where we use ${\rm Depth}(\bh)$ to denote the depth of the rooted tree $\bh$ belongs to before merging.
        \item[(ii)] $\bh_3\not\in\{\bh_1, \bh_2\}$. Firstly, we can split $\bh_3$ and its subtree from the rooted tree $\bh_3$ belonging to. Then, we can link $\bh_1,\bh_2$ to $\bh_3$ to form a new rooted tree. The depth of all nodes in the tree rooted at $\bh_1,\bh_2$ is no more than $\max\cbr{{\rm Depth}(\bh_1),{\rm Depth}(\bh_2)}+1$.
    \end{itemize}
    We also provide an illustrative proof in Figure \ref{fig:proof-bound-path}.
    \begin{figure}[t]
        \centering
        \includegraphics[width=0.9\linewidth]{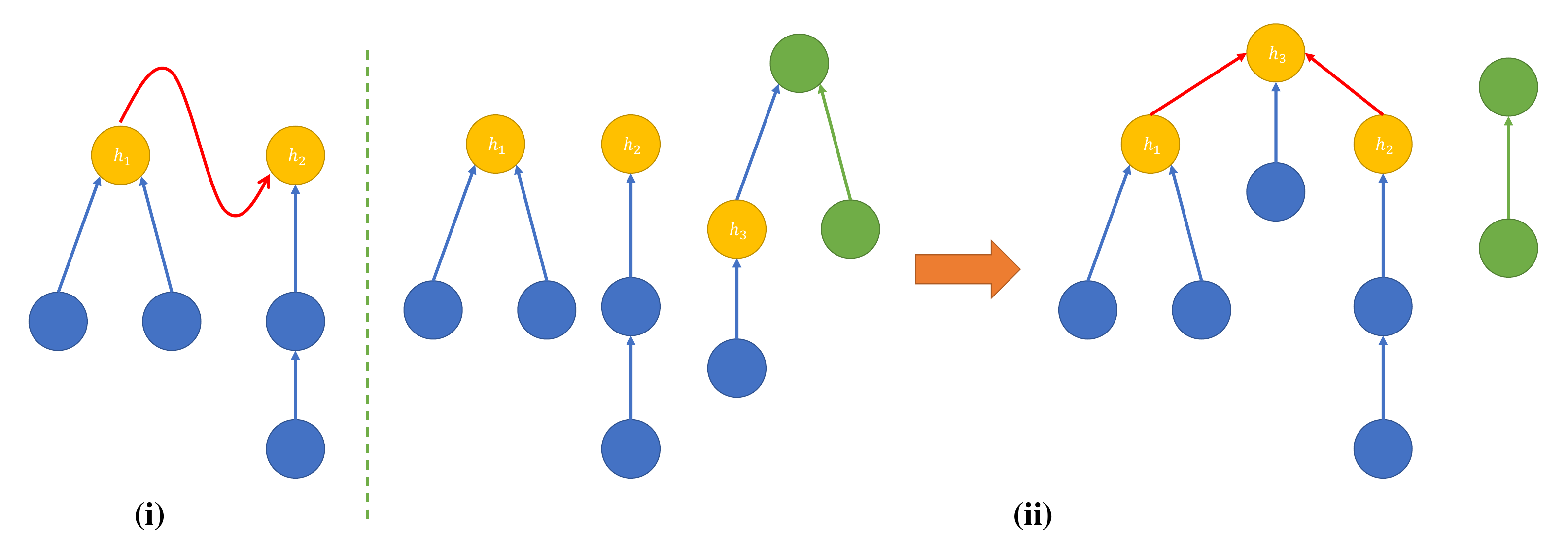}
        \caption{An illustrative proof of Lemma \ref{lemma:log-path}. In (i), the maximum tree depth remains $3$. In $(ii)$, the maximum tree depth increases from $3$ to $4$.}
        \label{fig:proof-bound-path}
    \end{figure}

    Therefore, in each episode, the maximum depth will increase by at most one while $|\cT|$ will decrease to one-half of it (we can assume $|\cH_M|=2^K$ for some $K$ without loss of generality since we can add nodes to the graph). Therefore, the depth of the last rooted tree in $\cT$ is at most $\log_2 |\cH_M|+1$.\qed
\end{lemma}

By Lemma \ref{lemma:log-path}, we can obtain a rooted tree, say rooted at $\bh_0$, with a bounded depth $O(\log_2 |\cH_M|)$.

\begin{lemma}
\label{lemma:always-reach}
    There exists $\useconstant{constant:go-back-root-length}=\log_2^2|\cH_M|+4\log_2|\cH_M|+3$ so that for any $C'\geq \useconstant{constant:go-back-root-length}$, every vertex can reach $\bh_0$ in $C'$ steps.
\proof
Firstly, we will prove that there exists $C_0$ so that for any $C_0'\geq C_0$, $\bh_0$ can reach itself in $C_0'$ steps. When $\bh_0$ has a self-loop, then $C_0=1$ satisfies the requirement.

Otherwise, $\bh_0$ has a successor $\bh_1\not=\bh_0$ according to Lemma \ref{lemma:common-successor-state}. By Lemma \ref{lemma:common-successor-state}, we know that $\bh_0$ and $\bh_1$ have a common successor. Let the common successor be $\bh_2$. Then we have cycles $\bh_0\to\bh_1\to\bh_2\to\bh_0$ and $\bh_0\to\bh_2\to\bh_0$ with length $depth(\bh_2)+2$ and $depth(\bh_2)+1$ where $depth(\bh)$ is the depth of $\bh$ in the rooted tree.
By Frobenius number \citep{sylvester1882subvariants-frobenius-number}, $C_0=depth(\bh_2)\cdot (depth(\bh_2)+1)\leq (\log_2|\cH_M|+1)(\log_2|\cH_M|+2)=\log_2^2|\cH_M|+3\log_2|\cH_M|+2$ by Lemma \ref{lemma:log-path}.

Therefore, for any vertex $\bh_1$, when it can reach $\bh_0$ with $K$ steps, it can also reach $\bh_0$ with $K+C_0,K+C_0+1,K+C_0+2,...$ steps by rotating in the cycles starting from $\bh_0$. So, we have $\useconstant{constant:go-back-root-length}=\log_2|\cH_M|+1+C_0\leq \log_2^2|\cH_M|+4\log_2|\cH_M|+3$.\qed
\end{lemma}

Back to the Markov chain, by Lemma \ref{lemma:always-reach}, for any $\bh_1\in\cH_M$, we have $(\cP^{\bpi})^{MK}_{\bh_1,\bh_0}\geq (\frac{\gamma^N}{|\cA|})^{MK}$ for any $K\geq \useconstant{constant:go-back-root-length}$.

\subsection{Contraction property with bounded memory of length $M$}
\label{sec:MDP-contraction-property}
In this section, we will directly adopt all conditions and notations in Appendix \ref{appendix:fast-mixing} without specifying them.

By Lemma \ref{lemma:always-reach}, there exists some constant $\useconstant{constant:go-back-root-length}$ so that $(\cP^{\bpi})^{M\useconstant{constant:go-back-root-length}}_{\bh_1,\bh_0}\geq (\frac{\gamma^N}{|\cA|})^{M\useconstant{constant:go-back-root-length}}$. For any $\bpi\uo$, the transition matrix of the induced Markov chain can be written as
\begin{align}
    &(\cP^{\bpi})^{M\useconstant{constant:go-back-root-length}}=\delta\underbrace{\begin{bmatrix}
    1 & 0&...&0\\
    1 & 0&...&0\\
    ...&...&...&...\\
    1 & 0&...&0\\
    \end{bmatrix}}_{\cU}+(1-\delta)\hat\cP^{\bpi}\label{eq:matrix-factor-U-P}\\
    &\delta \coloneqq (\frac{\gamma^N}{|\cA|})^{M\useconstant{constant:go-back-root-length}}\label{equ:def_delta}
\end{align}
where $\hat\cP^{\bpi}$ is also a Markov matrix and we assume the index of $\bh_0$ is 1 without loss of generality ($\bh_0$ is defined in Lemma \ref{lemma:always-reach} and we will use $\bh_0$ to denote it in this section by default). So, the distribution $(1,0,...,0)$ places probability one on $\bh_0$.

Notice that for any initial distribution $\mu\in\Delta_{|\cH_M|}$, we have $\mu\cU=(1,0,...,0)=:\bo$. Therefore, we have the following lemma.

\begin{lemma}
\label{lemma:MDP-last-iterate}
For any $\mu\in\Delta_{|\cH_M|}$, we have
\begin{align}
\label{eq:MDP-approximation}
    \mu(\cP^{\bpi})^{KM\useconstant{constant:go-back-root-length}}=(1-\delta)^K \mu(\hat\cP^{\bpi})^K +\delta\sum_{i=0}^{K-1}  (1-\delta)^{i} \bo(\hat\cP^{\bpi})^{i}.
\end{align}
\proof
This can be proved by induction. When $K=0$, it is satisfied. When $K=K_0+1$ and $K_0$ is satisfied, then,
\begin{align*}
    \mu(\cP^{\bpi})^{(K_0+1)M\useconstant{constant:go-back-root-length}}=&\Big((1-\delta)^{K_0} \mu(\hat\cP^{\bpi})^{K_0} +\delta\sum_{i=0}^{K_0-1}  (1-\delta)^{i} \bo(\hat\cP^{\bpi})^{i} \Big)(\cP^{\bpi})^{M\useconstant{constant:go-back-root-length}}\\
    =&\delta\Big((1-\delta)^{K_0}+\delta\sum_{i=0}^{K_0-1}  (1-\delta)^{i}\Big)\bo\\
    &+(1-\delta)^{K_0+1} \mu(\hat\cP^{\bpi})^{K_0+1} +\delta\sum_{i=0}^{K_0-1} (1-\delta)^{i+1} \bo(\hat\cP^{\bpi})^{i+1} \\
    =&\delta\bo+(1-\delta)^{K_0+1} \mu(\hat\cP^{\bpi})^{K_0+1} +\delta\sum_{i=1}^{K_0}  (1-\delta)^i \bo(\hat\cP^{\bpi})^i \\
    =&(1-\delta)^{K_0+1} \mu(\hat\cP^{\bpi})^{K_0+1} +\delta\sum_{i=0}^{K_0}  (1-\delta)^i \bo(\hat\cP^{\bpi})^i .\qedhere
\end{align*}

\end{lemma}

A direct consequence of Lemma \ref{lemma:MDP-last-iterate} is that the average probability distribution of the state will also converge.
\begin{lemma}
\label{lemma:MDP-average-iterate}
For any $\mu\in\Delta_{|\cH_M|}$, we have
\begin{align}
    \nbr{\lim_{T\to+\infty} \frac{1}{T}\sum_{t=0}^{T-1} \mu(\cP^{\bpi})^t-\mu(\cP^{\bpi})^K}_1\leq 2(1-\delta)^{\lfloor\frac{K}{M\useconstant{constant:go-back-root-length}}\rfloor}.
\end{align}
\proof

Firstly, for any $\mu\in\Delta_{|\cH_M|},K=K_0M\useconstant{constant:go-back-root-length}+m$ and $m<M\useconstant{constant:go-back-root-length}$, we have
\begin{align}
    \mu(\cP^{\bpi})^{K}=\mu(\cP^{\bpi})^m(\cP^{\bpi})^{K_0M\useconstant{constant:go-back-root-length}}=\Big(\mu(\cP^{\bpi})^m \Big)(\cP^{\bpi})^{K_0M\useconstant{constant:go-back-root-length}}=\mu'(\cP^{\bpi})^{K_0M\useconstant{constant:go-back-root-length}}
\end{align}
where $\mu'=\mu(\cP^{\bpi})^m$. Then,
\begin{align}
    \mu(\cP^{\bpi})^{K}=(1-\delta)^{K_0} \mu'(\hat\cP^{\bpi})^{K_0} +\delta\sum_{i=0}^{K_0-1}  (1-\delta)^{i} \bo(\hat\cP^{\bpi})^{i}.
\end{align}
The average distribution is
\begin{align*}
    &\lim_{T\to+\infty}\frac{1}{T}\sum_{t=0}^{T-1}  \mu(\cP^{\bpi})^t\\
    =&\lim_{T\to+\infty} \frac{1}{TM\useconstant{constant:go-back-root-length}}\sum_{t=0}^{T-1} (1-\delta)^t\sum_{m=0}^{M\useconstant{constant:go-back-root-length}-1} \mu(\cP^{\bpi})^m (\hat\cP^{\bpi})^t\\
    &+\lim_{T\to+\infty}\delta\sum_{i=0}^{K_0-1}\frac{T-i\cdot M\useconstant{constant:go-back-root-length}}{T}(1-\delta)^i \bo(\hat\cP^{\bpi})^{i} +\delta\sum_{i=K_0}^{\infty}\frac{T-i\cdot M\useconstant{constant:go-back-root-length}}{T}(1-\delta)^i \bo(\hat\cP^{\bpi})^{i}\\
    =&\delta\sum_{i=0}^{K_0-1}(1-\delta)^i \bo(\hat\cP^{\bpi})^{i} +\lim_{T\to+\infty}\Big(\delta\sum_{i=K_0}^{\infty}\frac{T-i\cdot M\useconstant{constant:go-back-root-length}}{T}(1-\delta)^i \bo(\hat\cP^{\bpi})^{i} \Big).
\end{align*}
Therefore, we have
\begin{align*}
    &\nbr{\lim_{T\to+\infty}\frac{1}{T}\sum_{t=0}^{T-1}  \mu(\cP^{\bpi})^t-\mu(\cP^{\bpi})^{K}}_1\\
    \leq& \nbr{(1-\delta)^{K_0} \mu(\cP^{\bpi})^{K-K_0M\useconstant{constant:go-back-root-length}}(\hat\cP^{\bpi})^{K_0} }_1+\lim_{T\to+\infty}\nbr{\delta\sum_{i=K_0}^{\infty}\frac{T-i\cdot M\useconstant{constant:go-back-root-length}}{T}(1-\delta)^i \bo(\hat\cP^{\bpi})^{i}}_1\\
    \leq& (1-\delta)^{K_0}+\delta\lim_{T\to+\infty}\sum_{i=K_0}^{\infty}\frac{T-i\cdot M\useconstant{constant:go-back-root-length}}{T}(1-\delta)^i \nbr{\bo(\hat\cP^{\bpi})^{i}}_1\\
    \leq& (1-\delta)^{K_0}+\delta\lim_{T\to+\infty}\sum_{i=K_0}^{\infty}(1-\delta)^i\\
    =&2(1-\delta)^{K_0}
\end{align*}
where $K_0=\lfloor \frac{K}{M\useconstant{constant:go-back-root-length}}\rfloor$.\qed

\end{lemma}

\begin{proof}[Proof of Lemma \ref{lemma:finite-f-to-infinite}]
Firstly, the existence of $f^{\infty}$ follows from Lemma \ref{lemma:MDP-average-iterate}.

Then, it is easy to verify that
\begin{align*}
    &f^K(\underbrace{\bpi,...,\bpi}_{K+1})=\sum_{\bh\in\cH_M} \cL_1(\bh_{M}) \cdot \rbr{\mu(\cP^{\bpi})^{K+1}}_{\bh}\\
    &f^{\infty}(\bpi)=\sum_{\bh\in\cH_M} \cL_1(\bh_{M}) \cdot \rbr{\lim_{T\to+\infty} \frac{1}{T}\sum_{t=0}^{T-1} \rbr{\mu(\cP^{\bpi})^t}_{\bh}}.
\end{align*}
So,
\begin{align*}
    &\abr{f^K(\underbrace{\bpi,...,\bpi}_{K+1})-f^{\infty}(\bpi)}\\
    \leq& \sum_{\bh\in\cH_M} \cL_1(\bh_{M})\cdot \abr{\rbr{\mu(\cP^{\bpi})^{K+1}}_{\bh}-\lim_{T\to+\infty} \frac{1}{T}\sum_{t=0}^{T-1} \rbr{\mu(\cP^{\bpi})^t}_{\bh}}\\
    \leq& \nbr{\mu(\cP^{\bpi})^{K+1}-\lim_{T\to+\infty} \frac{1}{T}\sum_{t=0}^{T-1} \mu(\cP^{\bpi})^t}_1\\
    \leq& 2(1-\delta)^{\floor{\frac{K}{M\useconstant{constant:go-back-root-length}}}}
\end{align*}
where the last line is by Lemma \ref{lemma:MDP-average-iterate}.
\end{proof}

\begin{proof}[Proof of Lemma \ref{lemma:finite-to-infinite}]

By Lemma \ref{lemma:MDP-last-iterate}, we have
\begin{align*}
    &\frac{1}{K}\sum_{k=1}^K \mu\prod_{s=1}^k\cP^{\bpi_s}\\
    =&\frac{1}{K}\sum_{k=1}^K\rbr{(1-\delta)^{\floor{k/(M\useconstant{constant:go-back-root-length})}} \mu \prod_{s=1}^{k\%(M\useconstant{constant:go-back-root-length})}\cP^{\bpi_s}\prod_{s=1}^{\floor{k/(M\useconstant{constant:go-back-root-length})}} \hat\cP^{k,s}+\delta\sum_{s=1}^{\floor{k/(M\useconstant{constant:go-back-root-length})}} (1-\delta)^{\floor{k/(M\useconstant{constant:go-back-root-length})}-s} \bo\prod_{s'=s+1}^{\floor{k/(M\useconstant{constant:go-back-root-length})}} \hat\cP^{\bpi_{s'}}}
\end{align*}
where $\delta=(\frac{\gamma^N}{|\cA|})^{M\useconstant{constant:go-back-root-length}}$ and $k\%(M\useconstant{constant:go-back-root-length})$ is the remainder of $k$ divided by $M\useconstant{constant:go-back-root-length}$. Notice that \Cref{eq:matrix-factor-U-P} also holds when different matrices are multiplied together. So,
\begin{align*}
    \prod_{s'=k\%(M\useconstant{constant:go-back-root-length})+(s-1)\cdot M\useconstant{constant:go-back-root-length}+1}^{k\%(M\useconstant{constant:go-back-root-length})+s\cdot M\useconstant{constant:go-back-root-length}} \cP^{\bpi_{s'}}=\delta \cU+(1-\delta)\hat\cP^{k,s}.
\end{align*}
We define $\hat\cP^{k,s}= \frac{1}{1-\delta}\rbr{\prod_{s'=k\%(M\useconstant{constant:go-back-root-length})+(s-1)\cdot M\useconstant{constant:go-back-root-length}+1}^{k\%(M\useconstant{constant:go-back-root-length})+s\cdot M\useconstant{constant:go-back-root-length}} \cP^{\bpi_{s'}} - \delta \cU} $ here for ease of notation.

Therefore, for different initial distributions $\mu_1,\mu_2$, we have
\begin{align*}
    \nbr{\frac{1}{K}\sum_{k=1}^K \mu_1\prod_{s=1}^k\cP^{\bpi_s}-\frac{1}{K}\sum_{k=1}^K \mu_2\prod_{s=1}^k\cP^{\bpi_s}}_1=& \nbr{\frac{1}{K}\sum_{k=1}^K (1-\delta)^{\floor{k/(M\useconstant{constant:go-back-root-length})}} (\mu_1-\mu_2) \prod_{s=1}^{k\%(M\useconstant{constant:go-back-root-length})}\cP^{\bpi_s}\prod_{s=1}^{\floor{k/(M\useconstant{constant:go-back-root-length})}} \hat\cP^{k,s}}_1\\
    \leq& \frac{M\useconstant{constant:go-back-root-length}}{K}\sum_{k=0}^{\floor{K/(M\useconstant{constant:go-back-root-length})}} (1-\delta)^k\nbr{\mu_1-\mu_2}_1\leq \frac{2M\useconstant{constant:go-back-root-length}}{K\delta}.
\end{align*}

Since $\bpi_1,\bpi_2,...,\bpi_K$ is an $\epsilon$-approximate CCE, for any $\hat\bpi_1\ui,\hat\bpi_2\ui,...,\hat\bpi_K\ui$,
\begin{align*}
    \frac{1}{K}\sum_{k=1}^K \inner{\mu_0\prod_{s=1}^k\cP^{\bpi_s}}{\cL_i}-\frac{1}{K}\sum_{k=1}^K \inner{\mu_0\prod_{s=1}^k\cP^{(\hat\bpi_s\ui,\bpi_s\uni)}}{\cL_i}\leq \epsilon
\end{align*}
where $\mu_0$ is the initial distribution.

Therefore, when we pick $K$ large enough, for an infinitely repeated game, at timestep $t>K$, we can pick strategy $\bpi_{(t-1)\% K+1}$ as the strategy at this timestep. Then, for any strategy $\hat\bpi\uo_{1},\hat\bpi\uo_{2},...$, we have
\begin{align*}
    &\lim_{T\to\infty} \sup \frac{1}{T}\sum_{B=0}^{T-1} \rbr{\frac{1}{K}\sum_{k=1}^K \inner{\mu_B\prod_{s=1}^k\cP^{\bpi_{s+BK}}}{\cL_i}-\frac{1}{K}\sum_{k=1}^K \inner{\hat\mu_B\prod_{s=1}^k\cP^{(\hat\bpi_{s+BK}\ui,\bpi_{s+BK}\uni)}}{\cL_i}}\\
    =& \lim_{T\to\infty} \sup \frac{1}{T}\sum_{B=0}^{T-1} \Bigg(\frac{1}{K}\sum_{k=1}^K \inner{\mu_0\prod_{s=1}^k\cP^{\bpi_{s+BK}}}{\cL_i}-\frac{1}{K}\sum_{k=1}^K \inner{\mu_0\prod_{s=1}^k\cP^{(\hat\bpi_{s+BK}\ui,\bpi_{s+BK}\uni)}}{\cL_i}\\
    &+\inner{\frac{1}{K}\sum_{k=1}^K (\mu_B-\mu_0)\prod_{s=1}^k\cP^{\bpi_{s+BK}}}{\cL_i}- \inner{\frac{1}{K}\sum_{k=1}^K (\hat\mu_B-\mu_0)\prod_{s=1}^k\cP^{(\hat\bpi_{s+BK}\ui,\bpi_{s+BK}\uni)}}{\cL_i}\Bigg)
\end{align*}
where we use $\mu_B$ ($\hat\mu_B$) to indicate the state distribution at the start of period $B+1$ when playing $\bpi_{1:BK}$ ($(\hat\bpi_{1:BK}\ui,\bpi_{1:BK}\uni)$) starting from initial distribution $\mu_0$. Notice that
\begin{align*}
    \abr{\inner{\frac{1}{K}\sum_{k=1}^K (\mu_B-\mu_0)\prod_{s=1}^k\cP^{\bpi_{s+BK}}}{\cL_i}}\leq& \nbr{\frac{1}{K}\sum_{k=1}^K (\mu_B-\mu_0)\prod_{s=1}^k\cP^{\bpi_{s+BK}}}_1\cdot\nbr{\cL_i}_\infty\\
    \leq& \nbr{\frac{1}{K}\sum_{k=1}^K (\mu_B-\mu_0)\prod_{s=1}^k\cP^{\bpi_{s+BK}}}_1\\
    \leq& \frac{2M\useconstant{constant:go-back-root-length}}{K(\frac{\gamma^N}{|\cA|})^{M\useconstant{constant:go-back-root-length}}}.
\end{align*}
Therefore, we have
\begin{align*}
    &\lim_{T\to\infty} \sup \frac{1}{T}\sum_{B=0}^{T-1} \rbr{\frac{1}{K}\sum_{k=1}^K \inner{\mu_B\prod_{s=1}^k\cP^{\bpi_{s+BK}}}{\cL_i}-\frac{1}{K}\sum_{k=1}^K \inner{\hat\mu_B\prod_{s=1}^k\cP^{(\hat\bpi_{s+BK}\ui,\bpi_{s+BK}\uni)}}{\cL_i}}\\
    \leq& \lim_{T\to\infty} \sup \frac{1}{T}\sum_{B=0}^{T-1} \rbr{\frac{1}{K}\sum_{k=1}^K \inner{\mu_0\prod_{s=1}^k\cP^{\bpi_{s+BK}}}{\cL_i}-\frac{1}{K}\sum_{k=1}^K \inner{\mu_0\prod_{s=1}^k\cP^{(\hat\bpi_{s+BK}\ui,\bpi_{s+BK}\uni)}}{\cL_i}+\frac{4M\useconstant{constant:go-back-root-length}}{K(\frac{\gamma^N}{|\cA|})^{M\useconstant{constant:go-back-root-length}}}}\\
    =&\lim_{T\to\infty} \sup \frac{1}{T}\sum_{B=0}^{T-1} \rbr{\frac{1}{K}\sum_{k=1}^K \inner{\mu_0\prod_{s=1}^k\cP^{\bpi_{s+BK}}}{\cL_i}-\frac{1}{K}\sum_{k=1}^K \inner{\mu_0\prod_{s=1}^k\cP^{(\hat\bpi_{s+BK}\ui,\bpi_{s+BK}\uni)}}{\cL_i}}+\frac{4M\useconstant{constant:go-back-root-length}}{K(\frac{\gamma^N}{|\cA|})^{M\useconstant{constant:go-back-root-length}}}\\
    \leq& \epsilon + \frac{4M\useconstant{constant:go-back-root-length}}{K(\frac{\gamma^N}{|\cA|})^{M\useconstant{constant:go-back-root-length}}}.\qedhere
\end{align*}

\end{proof}

\begin{proof}[Proof of Lemma \ref{lemma:upper-bound-Q}]
Firstly, let $\mu\in\Delta_{|\cH_M|}$ be the initial distribution. By Lemma \ref{lemma:MDP-average-iterate}, the time-average loss is independent of $\mu$. Thus,
\begin{align*}
    \rho^{\bpi}
    &=\lim_{T\to\infty}\inner{\bl}{\frac{1}{T}\sum_{t=0}^{T-1}\mu(\cP^{\bpi})^t}
      \overset{(i)}{=}\delta\sum_{i=0}^\infty (1-\delta)^i\inner{\bl}{\bo (\hat\cP^{\bpi})^i},
\end{align*}
where
\begin{align*}
    l_{\bh} \coloneqq \sum_{\ba\in\cA} \pi(\ba\given \bh)\cL_1(\bh,\ba).
\end{align*}
$(i)$ follows from Lemma \ref{lemma:MDP-last-iterate}.

Therefore, let $K_t \coloneqq \floor{\frac{t}{M\useconstant{constant:go-back-root-length}}}$ for notational simplicity. Then,
\begin{align*}
    \abr{Q^{\bpi}(\bh,\ba)}=&\abr{\cL_1(\bh,\ba)-\rho^{\bpi}+\sum_{t=0}^\infty\Big((1-\delta)^{K_t}\inner{\bl}{\mu (\cP^{\bpi})^{t-M\useconstant{constant:go-back-root-length}K_t}(\hat\cP^{\bpi})^{K_t}}+\delta\sum_{i=0}^{{K_t}-1} (1-\delta)^i\inner{\bl}{\bo (\hat\cP^{\bpi})^i}-\rho^{\bpi}\Big)}\\
    =&\abr{\cL_1(\bh,\ba)-\rho^{\bpi}+\sum_{t=0}^\infty\Big((1-\delta)^{K_t}\inner{\bl}{\mu (\cP^{\bpi})^{t-M\useconstant{constant:go-back-root-length}K_t}(\hat\cP^{\bpi})^{K_t}}-\delta\sum_{i={K_t}}^{\infty} (1-\delta)^i\inner{\bl}{\bo (\hat\cP^{\bpi})^i}\Big)}\\
    \leq&\abr{\cL_1(\bh,\ba)-\rho^{\bpi}}+\sum_{t=0}^\infty\Big((1-\delta)^{K_t}\abr{\inner{\bl}{\mu (\cP^{\bpi})^{t-M\useconstant{constant:go-back-root-length}K_t}(\hat\cP^{\bpi})^{K_t}}}+\delta\sum_{i={K_t}}^{\infty} (1-\delta)^i\abr{\inner{\bl}{\bo (\hat\cP^{\bpi})^i}}\Big)\\
    \leq& 1+\sum_{t=0}^\infty\Big((1-\delta)^{K_t}+\delta\sum_{i={K_t}}^{\infty} (1-\delta)^i\Big)=1+M\useconstant{constant:go-back-root-length}\sum_{K=0}^\infty \Big((1-\delta)^{K}+\delta\sum_{i=K}^{\infty} (1-\delta)^i\Big)\\
    =&1+2M\useconstant{constant:go-back-root-length}\sum_{K=0}^\infty (1-\delta)^{K}=1+2\frac{M\useconstant{constant:go-back-root-length}}{\delta},
\end{align*}
completing the proof.
\end{proof}

\vfill

\end{document}